%% file: arxiv2.tex
\documentclass[twoside,11pt]{article}

\usepackage{lastpage}
\PassOptionsToPackage{bookmarks={false}}{hyperref}
\usepackage{jmlr2e}

\input{temp}

\renewenvironment{proof}{\par\noindent{\bf Proof\ }}{\hfill\BlackBox\\[2mm]}

\usepackage{subfigure}
\usepackage{lipsum}
\usepackage{amsfonts}
\usepackage{amssymb}
\usepackage{amsmath}
\usepackage{amsthm}
\usepackage{latexsym}
\usepackage{verbatim}
\usepackage{epsfig}
\usepackage{amsbsy}
\usepackage{amscd}
\usepackage{bm}
\usepackage{graphicx}
\usepackage{algorithmic}
\usepackage{algorithm}
\usepackage{currfile}
\usepackage{hyperref}
\hypersetup{
colorlinks=true,
    citecolor = {blue}
}

\def\R{\mathbb{R}}
\def\tr{\mathrm{tr}}

\DeclareMathOperator{\srank}{srank}
\DeclareMathOperator{\pre}{\emph{\texttt{P}}}

\def\eps{\boldsymbol{\epsilon}}
\def\epss{\varepsilon}

\jmlrheading{**}{2020}{**-**}{**/**}{**/**}{}{}

\begin{document}

\title{ Two-Stage Approach to Multivariate Linear Regression with\\ Sparsely Mismatched Data}

\author{\name Martin Slawski\email mslawsk3@gmu.edu\\ \addr{George Mason University and Baidu Research} \\
         \name Emanuel Ben-David\email emanuel.ben.david@census.gov\\ \addr{U.S. Census Bureau}  \\
       \name Ping Li\thanks{Corresponding author.}\email liping11@baidu.com\\ \addr{Baidu Research}
       }

\editor{\\\\\\}

\maketitle

\begin{abstract} 

A tacit assumption in linear regression is that (response, predictor)-pairs correspond  to identical observational units. A series of recent works have studied scenarios in which this assumption is violated under terms such as ``Unlabeled Sensing and ``Regression with Unknown Permutation''. In this paper, we study the setup of multiple response variables and a notion of mismatches that generalizes permutations in order to allow for missing matches as well as for one-to-many matches. A two-stage method is proposed under the assumption that most pairs are correctly matched. In the first stage, the regression parameter is estimated by handling mismatches as contaminations, and subsequently the generalized permutation is estimated by a basic variant of matching. The approach is both computationally convenient and equipped with favorable statistical guarantees. Specifically, it is shown that the conditions for permutation recovery become considerably less stringent as the number of responses $m$ per observation increase. Particularly, for $m = \Omega(\log n)$, the required signal-to-noise ratio no longer depends on the sample size $n$. Numerical results on synthetic and real data are presented to support the main findings of our analysis. 
\end{abstract}

\vspace{0.05in}
\section{Introduction}\label{sec:intro}

\vspace{0.1in}

Linear regression and its numerous extensions is an object of timeless interest in statistics and related disciplines. Continuous research efforts are being made to increase the range of situations in which it can be applied with success. A specific challenge that has attracted considerable interest recently is regression in the absence of correspondence between predictors and responses, i.e., both are given as separate samples $\mc{X} = \{\M{x}_i \}_{i = 1}^n$ and $\mc{Y} = \{ \M{y}_i \}_{i = 1}^n$, but it is not (fully) known a priori which elements from $\mc{X}$ and $\mc{Y}$ are matching pairs in the sense of belonging to the same observational unit. Motivated by a number of applications in engineering, regression in this setting has been discussed in a series of recent papers~\citep{Emiya2014, Unnikrishnan2015, Pananjady2016, Pananjady2017, Abid2017, Hsu2017, Haghighatshoar2017, Shi2018, Wang2018, Domankic2018, Tsakiris2018, TsakirisICML19}. On the other hand, the above setup has a long history in statistics under the term ``Broken Sample Problem''  dating back to the early 1970s~\citep{DeGroot1971, Goel1975, DeGroot1976, DeGroot1980, Bai05, Wu1998, Chan2001} and a related line of research involving record linkage and statistical analysis based on merged data files (e.g.,~\cite{Neter65, Lahiri05, Goel2012, Scheuren93, Scheuren97}) partially motivated by government agencies like the U.S.~Census Bureau that routinely combines data from multiple surveys and/or external data to address questions of interest. In this context, the primary interest is in the estimation of parameters (e.g., covariance matrix, regression coefficients, $\ldots$) rather than restoration of the correspondence between elements of $\mc{X}$ and $\mc{Y}$. Instead, the focus is on the adjustment of subsequent analyses for potential mismatches resulting from errors or ambiguities in record linkage based on quasi-identifiers. In fact, unique identifiers such as the social security number often need to be removed because of privacy concerns. Accordingly, in an alternative perspective on the broken sample problem, identification of matching pairs in $\mc{X}$ and $\mc{Y}$ is undesired because $\mc{Y}$ contains sensitive data, but an adversary makes the attempt to use external data along with identifying information stored in $\mc{X}$ to retrieve matching pieces in $\mc{Y}$. Well-known instances of such ``linkage attacks''  are the identification of the medical history of the former governor of Massachusetts~\citep{Sweeney2001} and the partial de-anonymization of Netflix movie rankings with the help of publicly available data in the Internet Movie Database (IMDb)~\citep{Narayanan2008}. Broken sample problems thus bear a relationship to data confidentiality; we refer to~\cite{Domingo2016} for a detailed discussion.\\

\vspace{-0.05in}

\noindent \emph{Related Work.} A starting point of recent research on the subject is the work by~\cite{Unnikrishnan2015} which
studies linear regression in the absence of noise with a scalar response that is observed up to an unknown permutation of
the entries, i.e., $y_i = \M{x}_{\pi^*(i)}^{\T} \beta^*$, $i=1,\ldots,n$, for a permutation $\pi^*$ on $\{1,\ldots,n\}$. The authors show that $\beta^* \in \R^d$ can be recovered with probability one by exhaustive enumeration
over all permutations if $n \geq 2d$ and the entries of $X$ are drawn i.i.d.~from a distribution absolutely continuous w.r.t.~the Lebesgue measure on $\R$. Alternative proofs of this result have been obtained in~\cite{Tsakiris2018b, Domankic2018}.~\cite{Pananjady2016} study computational and statistical limits of recovering $\pi^*$ for Gaussian $\{ \M{x}_i \}_{i = 1}^n$ and Gaussian additive noise with variance $\sigma^2$. The authors show that least squares estimation recovers $\pi^*$ exactly if the signal-to-noise ratio $\textsf{SNR} = \nnorm{\beta^*}_2^2 / \sigma^2 = n^{\Omega(1)}$ which is also shown to be sharp up to a constant factor in the exponent. At the same time, least squares estimation of $\pi^*$ is proved to be \textsf{NP}-hard.~\cite{Abid2017, Hsu2017} shed light on the estimation of $\beta^*$ under similar setups as in~\cite{Pananjady2016}. Specifically,~\cite{Hsu2017} establish that the requirement $\textsf{SNR} = \Omega(d / \log \log n)$ is necessary to ensure low relative squared $\ell_2$-estimation error which
is a dramatic gap compared to the requirement $\textsf{SNR} = \Omega(d / n)$ if $\pi^*$ is known. The paper~\cite{Abid2018} proposes Expectation-Maximization (EM) schemes to tackle the least squares problem for
estimation of $\pi^*$. A clever initialization strategy for those schemes based on algebraic considerations is developed in~\cite{Tsakiris2018}. The paper~\cite{SlawskiBenDavid2017} assumes that $\pi^*$ is
$k$-sparse, i.e., $\pi^*(i) = i$ except for $k \ll n$ indices, and analyzes a convex formulation for
estimating $\beta^*$ in this setting. A similar sparsity assumption is employed in~\cite{Shi2018} for
spherical regression. Order-constrained regression problems with unknown permutation are
discussed in~\cite{Flammarion16, Weed2018, Carpentier2016, Ma2020}. 

\vspace{0.1in}

\noindent \emph{Contributions.} While several papers have elucidated important aspects of linear regression with unknown permutation for a scalar response, only few papers~\citep{Pananjady2017, ZhangSlawskiLi2019, SlawskiRahmaniLi2018}
consider multivariate response, i.e., the $\{\M{y}_i \}_{i = 1}^n$ are $m$-dimensional, $m > 1$. This case is of independent interest for at least two reasons. First, in the context of record linkage it is natural to assume that both data sets $\mc{X}$ and $\mc{Y}$ to be merged are multi-dimensional. Second, the availability of multiple responses affected by the same permutation is expected to facilitate estimation as is confirmed by the results herein. Indeed, the requirements on the $\textsf{SNR}$ to achieve permutation recovery can be considerably weaker, with potential drops from $\textsf{SNR} = n^{\Omega(1)}$ for
$m = O(1)$ to $\textsf{SNR} = \Omega(1)$ for $m = \Omega(\log n)$. Similar benefits are shown in~\cite{Pananjady2017, ZhangSlawskiLi2019, SlawskiRahmaniLi2018}. The results in~\cite{Pananjady2017} concern the prediction or denoising error rather than estimation
of $\pi^*$.~\cite{ZhangSlawskiLi2019} provide information-theoretic lower bounds for permutation recovery; however, the  computational scheme therein is only investigated empirically without theoretical support.  The method in~\cite{SlawskiRahmaniLi2018} requires $m \gtrsim d$ to perform well; another downside of the approach is its cubic runtime in $n$. None of the aforementioned papers on the case $m > 1$ contain rigorous results regarding the estimation of the regression parameter. In order to enable the latter, the tolerable number of mismatches $k$ herein is limited to a sufficiently small fraction of the number of samples, i.e., $k/n < c$ for $c$ small enough. In this regime, estimation of the regression coefficients and restoration of the correct correspondence is shown to be possible based on convex optimization.

Moreover, we consider a more general notion of faulty correspondence between $\mc{X}$ and
$\mc{Y}$ which goes beyond permutations, specifically allowing for missing matches  and one-to-many matches. The effectiveness of the approach is demonstrated
by experiments on synthetic and real data sets \tcb{as well as a case study pertaining to data integration}.

\vspace{0.1in}

\noindent \emph{Outline.} In $\S$\ref{sec:probstatement}, we state the problem and setting under consideration as well as the approach taken. Our main theoretical results are presented in $\S$\ref{sec:mainresults}. Empirical corroboration based on synthetic and real data is provided in $\S$\ref{sec:experiments}. We conclude with a summary and an overview on potential directions of future research in $\S$\ref{sec:conclusion}.      

\vspace{0.1in}

\noindent \emph{Notation.} The symbol $\mathbb{I}$ is used for the indicator function with value one if its argument is true and zero else. For a positive integer $\ell$, $I_{\ell}$ denotes the
$\ell \times \ell$ identity matrix, and $\mathbb{S}^{\ell-1}$ denotes the unit sphere in $\R^{\ell}$. We write $|S|$ for the cardinality of a set $S$. The complement
of $S$ with respect to context-dependent base sets is denoted by $S^c$, and $\text{conv} \, S$ denotes the convex hull of $S$. For a matrix $A$, $\nnorm{A}_2 = \sigma_{\max}(A)$ denotes its spectral norm respectively maximum singular value, $\nnorm{A}_F$ denotes its Frobenius norm, and $\text{range}(A)$ denotes the column space of $A$.
The $i$-th row of $A$ is denoted by $A_{i,:}$, and is treated as column vector. For an index set $I$ and a vector $v$ of real numbers, $v_{I}$ denotes the subvector corresponding to $I$. We write 
$a \vee b = \max\{a,b\}$ and $a \wedge b = \min\{a,b \}$. Positive constants are denoted by $C$, $c$, $c_1$ etc. We make use of the usual Big-O notation in terms of $O$, $o$, $\Omega$ and $\Theta$. We often use $a \lesssim b$, $b \gtrsim a$, and $a \asymp b$ as shortcuts for $a = O(b)$, $b = \Omega(a)$ and
$a = \Theta(b)$, respectively.

\section{Problem statement and proposed approach}\label{sec:probstatement}
We start by fixing the setup under consideration herein before outlining our approach. We then provide a toy data example in order to illustrate some of the main challenges and characteristics of the
given problem and the proposed approach.  
\vskip1ex
\subsection{Setup}
As stated in the introduction, we assume that we are given two samples $\mc{X} = \{ \M{x}_i \}_{i = 1}^n$ and
$\mc{Y} = \{\M{y}_i \}_{i = 1}^n$ taking values in $\R^d$ and $\R^m$, respectively, that are related by the model 
\begin{equation}\label{eq:limo_mis}
\mathfrak{s}_i  \M{y}_i = B^{*\T} \M{x}_{\theta^*(i)} + \tcb{\sigma \mathfrak{s}_i \eps_{i}}, \quad 1 \leq i \leq n, 
\end{equation}
where $\theta^*: \{1,\ldots,n\} \rightarrow \{0, 1,\ldots,n\}$ is a map representing
the (unknown) underlying correspondence between observations in $\mc{X}$ and $\mc{Y}$, with
the convention that $\M{x}_{0} \coloneq 0$, and
$\mathfrak{s}_i = \mathbb{I}(\theta^*(i) \neq 0)$ indicates whether $\M{y}_i$ has a match among $\mc{X}$,
$1 \leq i \leq n$. \tcb{For the set of non-matches $\mc{N} = \{i: \,  \mathfrak{s}_i = 0 \}$, we suppose that $\{ \M{y}_i \}_{i \in \mc{N}}$ is independent of $\mc{X}$}. 

If $\theta^*(i) = i$ for $1 \leq i \leq n$, the above model reduces to an ordinary multivariate regression model
with $m$ responses and $d$ predictor variables, regression coefficients $B^* \in \R^{d \times m}$,
and random error variables $\{ \eps_i \}_{i = 1}^n$. Model~\eqref{eq:limo_mis} can be expressed
equivalently via
\begin{equation}\label{eq:limo_mis_c}
\tcb{\mc{S} Y = \Theta^* X B^*  + \sigma \mc{S} E},
\end{equation}
where $Y$ and $E$ are $n$-by-$m$ matrices whose rows are given by $\{ \M{y}_i^{\T} \}$ and $\{ \eps_i^{\T} \}$, respectively, $\mc{S} = \text{diag}(\mathfrak{s}_1, \ldots, \mathfrak{s}_n)$, $X$ is an $n$-by-$d$ matrix with rows $\{ \M{x}_i^{\T} \}_{i = 1}^n$, and $\Theta^* = (\Theta^*_{ij})_{1 \leq i, j \leq n}$ has entries $\Theta^*_{ij} = 1$ if $\theta^*(i) = j$ for $j \neq 0$, and zero otherwise. Observe that by construction, $\Theta^*$ is contained in the following set of matrices
\begin{align}
  \mc{M}& = \Big\{\Theta \in \R^{n \times n}:\; \Theta_{ij} \in \{0,1\}, \, 1 \leq i,j \leq n, \; \textstyle\sum_j \Theta_{ij} \leq 1 , \, 1 \leq i \leq n \Big\} \label{eq:matchings} \\
  &\supset \mc{P} = \{\Theta \in \R^{n \times n}:\; \Theta^{\T} \Theta  = I_n, \; \Theta_{ij} \in \{0,1\}, \, 1 \leq i,j \leq n  \} \label{eq:permutations},
\end{align}
which contains the set of $n$-by-$n$ permutation matrices $\mc{P}$ in~\eqref{eq:permutations}. Model~\eqref{eq:limo_mis} is hence more general compared to existing work in which $\theta^*$ is restricted to be a permutation. \tcb{In particular, the generalization herein allows for
missing matches via $\Theta^*_{i,:} = 0$ for $i \in \mc{N}$}, as well as for one-to-many matches, i.e., more than one element in $\mc{Y}$
may correspond to the same element in $\mc{X}$; cf.~Figure~\ref{fig:genpermutation} for an illustration. \tcb{We note that the case of one-to-many matches is also considered in~\cite{Pananjady2017}, cf.~Section 2.4 therein.} 

\begin{figure}[h!]
\begin{flushleft}
\hspace*{-1ex}\begin{minipage}{0.08\textwidth}
 $\begin{array}{l}  
  \{y_i \}_{i = 1}^n \\[5ex]
  \{ \M{x}_i \}_{i = 1}^n   
 \end{array}$  
\end{minipage}
\hspace*{.5ex}
\begin{minipage}{0.25\textwidth}
\centering \includegraphics[height = 0.08\textheight]{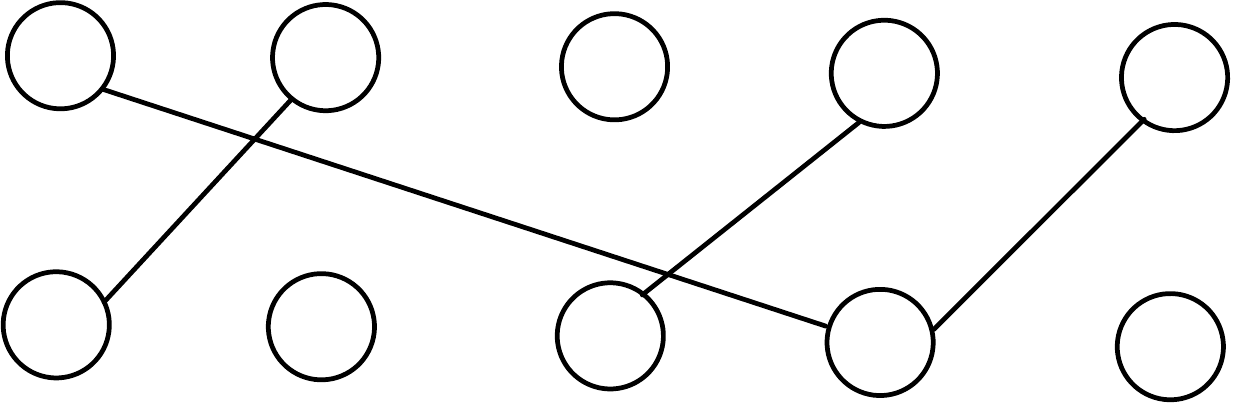}
\end{minipage}
\hspace*{12ex}
\begin{minipage}{0.5\textwidth}
  {\small \begin{equation*}
    \Theta^* = \begin{bmatrix}
      0 & 0 & 0 & 1 & 0 \\
      1 & 0 & 0 & 0 & 0 \\
      0 & 0 & 0 & 0 & 0 \\
      0 & 0 & 1 & 0 & 0 \\
      0 & 0 & 0 & 1 & 0
    \end{bmatrix}, \;\,
    \mc{S} = \begin{bmatrix}
      1 & 0 & 0 & 0 & 0 \\
      0 & 1 & 0 & 0 & 0 \\
      0 & 0 & 0 & 0 & 0 \\
      0 & 0 & 0 & 1 & 0 \\
      0 & 0 & 0 & 0 & 1
    \end{bmatrix}
  \end{equation*}}
\end{minipage}
\end{flushleft}
\vspace{-0.2in}
\caption{Illustration of the generalized permutation model herein for $n = 5$ including a missing match ($\M{y}_3$) and a one-to-many match between $(\M{y}_1, \M{y}_5)$ and $\M{x}_4$.}\label{fig:genpermutation}\vspace{-0.2in}
\end{figure}

Depending on the application, the goals in the setup~\eqref{eq:limo_mis} concern estimation of $B^*$ and/or $\Theta^*$. If $\Theta^*$ is recovered exactly by an estimator $\wh{\Theta}$, i.e., the event $\{ \wh{\Theta} = \Theta^* \}$ occurs, estimation of $B^*$ becomes an ordinary regression problem. In post-linkage data analysis, $\Theta^*$ can be used to
model error in the file linkage process, caused, e.g., by ambiguities resulting from the use of quasi-identifiers (say, the combination of age, gender, and race), but is typically treated as a nuisance parameter while primary interest concerns $B^*$. By contrast, in the setting of linkage attacks, the adversary aims at leveraging the linear
relationship between elements of $\mc{X}$ and $\mc{Y}$, and hence $B^*$ is only regarded as a means to retrieve $\Theta^*$. In the sequel, we adopt neither viewpoint and consider estimation of both $B^*$ and $\Theta^*$.  \\

\noindent \emph{Assumptions.} Below, we summarize and discuss the main assumptions of our analysis.
\begin{itemize}
\item The map $\theta^*$ is said to be $k$-sparse if $\theta^*(i) = i$ except for indices
  $S_* \subset \{1,\ldots,n\}$ with $|S_*| \leq k$ for $k \ll n$. Equivalently, $S_* = \{i: \, \Theta_{ii}^* \neq 1 \}$. Model~\eqref{eq:limo_mis_c} implies that
  \begin{equation}\label{eq:contamination}
Y =  X B^*  + \Phi^* + \sigma \mc{S} E, 
\end{equation}
where $\Phi^*_{i,:} = \M{y}_i - B^{*\T} \M{x}_{i}$ if $\theta^*(i) = 0$ and $\Phi^*_{i,:} = B^{*\T} \M{x}_{\theta^*(i)} - B^{*\T} \M{x}_{i}$ otherwise, $1 \leq i \leq n$.\blanco{with some slight abuse of notation, we replace $\Theta^* E$ by $\mc{S} E$, which follows the same distribution.} Observe that $k$-sparsity of $\theta^*$ implies that $\Phi^*$ has at most $k$ non-zero rows. Throughout this paper, we shall impose
constraints on the size of $k$. As of now, if $\sigma > 0$ and $k$ is not restricted, no practical estimation scheme with provable guarantees
is known even if $\theta^*$ is a permutation. \tcb{Apart from that, the sparse regime is relevant to applications in record linkage as elaborated in detail in the case study in $\S$\ref{sec:experiments}.}    
\item The matrix $X$ has i.i.d.~Gaussian rows $\M{x}_i \sim N(0, \Sigma)$, $1 \leq i \leq n$. Without loss of generality, we   
  assume that $\Sigma = I_d$ as can be ensured by re-defining $B^*$ accordingly.
\item Likewise, the matrix $E$ has i.i.d.~Gaussian rows $\eps_i \sim N(0, I_m)$, $1 \leq i \leq n$, and is
      independent of $X$. 
\end{itemize}
The second assumption and the first part of the third assumption do not appear critical to our approach, but they considerably
simplify results and proofs and thus aid presentation. \tcb{The main results in this
paper continue to hold for $X$ and $E$ with i.i.d.~sub-Gaussian rows up to slight modifications, cf.~Appendix~\ref{app:tosubgaussian}}. Moreover, it is common to assume 
that the $m$ entries of the noise terms $\{ \eps_i \}_{i = 1}^n$ are correlated; such extension can be accommodated, too.

Finally, we note that representation~\eqref{eq:contamination} is general enough to cover various
other scenarios involving mismatched data in regression. For example, it also applies if a subset of the predictors  
is collected jointly with the response, i.e., we observe samples $\mc{D}_1 = \{ (\M{x}_i^{(1)}, \M{y}_i) \}_{i = 1}^n$ and $\mc{D}_2 = \{ \M{x}_i^{(2)} \}_{i = 1}^n$ with $\{ \M{x}_i^{(1)} \}_{i = 1}^n$ and  $\{ \M{x}_i^{(2)} \}_{i = 1}^n$ having dimension $d_1$ and $d_2$, respectively,
$d_1 + d_2 = d$, and associated regression model
\begin{equation}\label{eq:covariates_with_response}
\M{y}_i = B_{(1)}^{*\T} \M{x}_{i}^{(1)} + B_{(2)}^{*\T} \M{x}_{\theta^*(i)}^{(2)} + \sigma \eps_i, \; i=1,\ldots,n, 
\end{equation}
where $\theta^*$ is a permutation of $\{1,\ldots,n\}$. Here, model~\eqref{eq:covariates_with_response}
is subsumed by~\eqref{eq:contamination} by setting $B^* = \left[ \begin{array}{c} B_{(1)}^* \\[.2ex] B_{(2)}^* \end{array} \right]$, $\Phi^*_{i,:} =  B_{(2)}^{*\T} \M{x}_{\theta^*(i)}^{(2)} - B_{(2)}^{*\T} \M{x}_{i}^{(2)}$, $1 \leq i \leq n$, and $\mc{S} = I_n$. 
The approach and its analysis
below applies to this and presumably also to other modifications with slight changes.  

\subsection{Approach}
We suggest to tackle estimation of $B^*$ and $\Theta^*$ in a two-stage approach that we motivate as follows. \tcb{Suppose first that there are no missing matches so that $\sum_{j} \Theta_{ij}^* = 1$, $1 \leq i \leq n$, and denote by $\overline{\mc{M}}$ the corresponding subset of $\mc{M}$ that excludes matrices with all-zero rows.} Joint
least squares estimation, i.e., $\min_{\Theta \in \overline{\mc{M}}, \, B \in \R^{d \times m}} \nnorm{Y - \Theta X B}_F^2$, is \textsf{NP}-hard
\citep{Pananjady2016}. However, if $B^*$ is known, least squares estimation of $\Theta^*$ reduces to a tractable
optimization problem that decouples along the rows of $Y$:
\begin{equation}\label{eq:Theta_oracle_p}
  \min_{\Theta \in \overline{\mc{M}}} \nnorm{Y - \Theta X B^*}_F^2 = \su \left\{ \min_{1 \leq j \leq n} \nnorm{\M{y}_i - B^{*\T} \M{x}_{j}}_2^2 \right \}.   
\end{equation}
\blanco{where we recall that $\M{x}_{0} = 0$.}Assuming for simplicity that the minimizing indices $\wh{j}(i)$ for the optimization problems
inside the curly brackets are unique, we have $\wh{\Theta}_{i \wh{j}(i)} = 1$\blanco{if $\wh{j}(i) \neq 0$}, $1 \leq i \leq n$; all other entries
of $\wh{\Theta}$ equal zero. \tcb{If in addition} $\theta^*$ is known to be one-to-one (i.e., a permutation), minimization over $\mc{M}$ can be replaced by minimization over $\mc{P}$~\eqref{eq:permutations}. The latter optimization problem reduces to a linear assignment problem~\citep{Burkard2009}, a specific linear program that can be solved efficiently by specialized techniques such as the Hungarian Algorithm~\citep{Kuhn1955} or the Auction Algorithm~\citep{Bertsekas1992}. 

\tcb{In the case of missing matches, taking the minimum in~\eqref{eq:Theta_oracle_p} over $\mc{M}$ instead of over $\overline{\mc{M}}$ cannot be expected to ensure the successful identification of missing matches. In fact, a row 
of zeroes in $\Theta$ means that the corresponding row of $Y$ is paired with the zero vector rather than with any of the
$\{ B^{*\T} \M{x}_j \}_{j = 1}^n$, but the use of the zero vector as a reference for missing matches is not meaningful. This observation prompts the following modification of~\eqref{eq:Theta_oracle}:
\begin{equation}\label{eq:Theta_oracle}
\text{Compute} \, \min_{1 \leq j \leq n} \nnorm{\M{y}_i - B^{*\T} \M{x}_{j}}_2^2 \, \text{;$\;\,$set}\, \, \wh{\Theta}_{ij} =\begin{cases} 1 \;\; &\text{if} \, j = \wh{j}(i) \, \text{and} \, \nnorm{\M{y}_i - B^{*\T} \M{x}_{\wh{j}(i)}}_2 \leq \tau,\\
0 \; \; &\text{otherwise},\;\;\, \qquad \qquad \quad 1 \leq i,j \leq n,
\end{cases}
\end{equation}
where  $\{ \wh{j}(i) \}_{i = 1}^n$ are the minimizing indices as above, and $\tau > 0$ is a suitably chosen threshold whose choice is discussed in Theorem~\ref{theo:correspondence} below.}\\

So far, $B^*$ was supposed to be known. If $B^*$ is unknown, it has to be replaced by an estimator $\wh{B}$. At this point, our approach makes use of the sparsity assumption for $\theta^*$. In view of relation~\eqref{eq:contamination}, we consider
\begin{equation}\label{eq:grouplasso}
\min_{B \in \R^{d \times m}, \, \Xi \in \R^{n \times m}} \frac{1}{2n \cdot m}\nnorm{Y - XB - \sqrt{n} \Xi}_F^2 + \lambda \su \nnorm{\Xi_{i,:}}_2,
\end{equation}
for a tuning parameter $\lambda > 0$, where $\Xi$ targets $\Xi^* \coloneq \Phi^* / \sqrt{n}$ with $\Phi^*$ as in~\eqref{eq:contamination}, and $\nnorm{\Xi_{i,:}}_2$ being used as a convex surrogate for $\mathbb{I}(\nnorm{\Xi_{i,:}}_2 > 0)$, $1 \leq i \leq n$, in order
to promote row-wise sparsity of $\Xi$~\citep{YuanLin2006, Eldar2009, Lounici2011}. The use of the re-scaled quantity $\Xi^*$ in place
of $\Phi^*$ is done merely for technical reasons. We note that a variant of~\eqref{eq:grouplasso} for a single response variable has
been employed in the context of linear regression with outliers~\citep{She2012, Laska2009, Nguyen2013}.

\begin{algorithm}[!h]
  \caption{Block coordinate descent for minimizing~\eqref{eq:grouplasso}}\label{alg:bcd}
Compute the QR factorization $X = Q R$ of $X$, and initialize $X B^{(0)} = Q Q^{\T} Y$, $\Xi^{(0)} \equiv 0$.\\
\textbf{1. Update for  $\Xi$}
\begin{align*}
  \Xi^{(t+1)} &\leftarrow (1 - \alpha^{(t)}) \Xi^{(t)} + \alpha^{(t)} \textsc{GroupThreshold}(Y - X B^{(t)}, \tau)/\sqrt{n}, \qquad \tau \coloneq m \cdot \sqrt{n} \cdot \lambda,  
\end{align*}
where for a matrix $A$ with rows $\{ a_i \}_{i = 1}^n$ and $\eta \geq 0$, $\text{\textsc{GroupThreshold}}(A, \eta)$ is defined by 
\begin{equation*}
a_{i} \leftarrow a_{i} \cdot \left(1 - \eta/\nnorm{a_{i}}_2 \right)_+, \; i=1,\ldots,n, \qquad (\cdot)_+ \coloneq \max\{\cdot,0\}.  
\end{equation*}
\\[-3ex]
\textbf{2. Update for $X B$}:
\begin{equation*}
X B^{(t+1)} \leftarrow (1 - \gamma^{(t)}) X B^{(t)} + \gamma^{(t)} Q Q^{\T} (Y - \sqrt{n} \Xi^{(t+1)}). 
\end{equation*}
The step sizes $\alpha^{(t)}, \gamma^{(t)} \subset (0,1)$ are chosen by back-tracking line search~\citep{Bertsekas1999}.
\end{algorithm}

Optimization problem
\eqref{eq:grouplasso} can be solved efficiently by block coordinate descent as outlined in Algorithm~\ref{alg:bcd} that has performed
extremely well throughout our experiments, typically converging after a small number of iterations. Formal convergence results follow
immediately from the general framework in~\cite{Tseng2010}.

The estimator $\wh{B}$ resulting from~\eqref{eq:grouplasso} can potentially be refined by a least squares re-fitting step after removing data corresponding to $\wh{S}(t) = \{1 \leq i \leq n:\, \nnorm{\wh{\Xi}_{i,:}}_2 \geq t \}$, where $\wh{\Xi}$ denotes the
minimizing $\Xi$ in~\eqref{eq:grouplasso} and $t$ is a suitably chosen threshold. The rationale is to remove mismatches as they
hamper parameter estimation. This yields
\begin{equation}\label{eq:refitting}
\min_{B \in \R^{d \times m}} \sum_{i \notin \wh{S}(t)} \nnorm{\M{y}_i - B^{\T} \M{x}_i}_2^2.
\end{equation}
In summary, this yields the following two-stage (or optionally three-stage) approach for estimating
$B^*$ and subsequently $\Theta^*$.
\vspace*{-0.1ex}
\begin{enumerate}
\item Estimate $B^*$ from~\eqref{eq:grouplasso}, and optionally refine via~\eqref{eq:refitting}. 
\item Estimate $\Theta^*$ from~\eqref{eq:Theta_oracle} with $B^*$ replaced by the estimator
      obtained in 1. 
    \end{enumerate}
\tcb{It is worth pointing out that sparsity of $\Theta^*$ is incorporated at step 1.~only. The procedure~\eqref{eq:Theta_oracle} can be modified accordingly by applying it only for the indices corresponding to the $k$ largest values among $\{ \nnorm{\M{y}_i - B^{*\T} \M{x}_i}_2^2 \}_{1 \leq i \leq n}$, and setting $\wh{\Theta}_{ii} = 1$ for all remaining $i$. 
We do not study this modification in the sequel since it does not fundamentally change the statistical limits in recovering $\Theta^*$ as stated in Theorem~\ref{theo:correspondence} below.} \\

\begin{figure}[h!!]
  {\scriptsize  \begin{tabular}{|lllllll|lllllll|}
                  \hline
                  Jan &  Mar &   May &  Jul &  Sep &  Nov & $\mc{X}$ & $\mc{Y}$ & Feb &   Apr &  Jun &   Aug &  Oct &  Dec  \\
                  \hline
    16  &  33  &  59  &  74  &  62  &  34 & Minneapolis &  Memphis &  46 &   63 &   80 &   82 &   64  &  44 \\
    -8  &  12  &  50  &  63  &  45  &   3 & Fairbanks & San Antonio & 56  &  70  &  83  &  85  &  71   & 53 \\
     1  &  54  &  72  &  83  &  75  &  53 & Memphis & Fairbanks & -1  &  33  &  61  &  57  &  24   & -4 \\
    34  &  44  &  64  &  78  &  68  &  47 & Baltimore & Dallas & 50  &  66  &  81  &  86  &  68   & 47 \\
    46  &  58  &  74  &  86  &  78  &  57 & Dallas & Tampa & 63  &  72  &  82  &  83  &  76   & 63 \\
    23  &  35  &  56  &  72  &  63  &  39 & Milwaukee$^*$ & Pittsburgh &  31 &   51 &   69 &   72 &   53  &  33 \\
    61  &  67  &  78  &  83  &  82  &  69 & Tampa & Minneapolis & 21  &  48  &  69  &  71  &  49   & 20 \\
    29  &  40  &  60  &  73  &  64  &  43 & Pittsburgh$^*$ & Portland & 44  &  52  &  64  &  70  &  55   & 40 \\
    52  &  62  &  77  &  85  &  80  &  61 & San Antonio & Baltimore & 36  &  54  &  73  &  76  &  57   & 37 \\
                  41  &  48  &  58  &  69  &  65  &  47 & Portland & Milwaukee & 26  &  46  &  67  &  71  &  52   &  27  \\
                 \hline 
                \end{tabular}}
              \vskip1ex
              \begin{center}
                \begin{minipage}{.92\textwidth}
                  \begin{center}
              {\scriptsize \begin{tabular}{|l|llllll|}
                             \hline & & & & & & \\[-1.7ex]
               $\wh{S}$   & Baltimore \hspace*{0.0001ex} &  Dallas  &  Fairbanks  & \hspace*{0.0001ex} Las Vegas$^\dagger$  \hspace*{0.0001ex} &  Memphis  \hspace*{0.0001ex} &  \hspace*{0.0001ex} Minneapolis \hspace*{0.0001ex} \\[.5ex]
                             $\wh{\theta}(\wh{S})$ & \hspace*{0.0001ex} Milwaukee \hspace*{0.0001ex} &  Seattle  &   Fairbanks & \hspace*{0.0001ex} Dallas \hspace*{0.0001ex}  &  Baltimore \hspace*{0.0001ex} & \hspace*{0.0001ex} Minneapolis \hspace*{0.0001ex} \\
                             \hline
                           \end{tabular}\\[.7ex]
                           continued:\\[1ex]
                           \begin{tabular}{|l|llllll|}
                             \hline &&&&&&\\[-1.7ex]
                             $\wh{S}$    &  Phoenix & Portland & San Antonio & San Francisco$^{\dagger}$ & Seattle$^{\dagger}$ & Tampa \\[.5ex]                                    $\wh{\theta}(\wh{S})$  &Las Vegas  & Memphis & Phoenix &  San Francisco & San Antonio & Tampa \\
                             \hline    \end{tabular}}
                       \end{center}
                     \end{minipage}
\end{center}
\vspace*{-2.35ex}
\caption{Top: mismatched subset of the U.S.~cities temperatures data set. Bottom: estimated subset of mismatched cities $\wh{S}$ and estimated
  correspondence $\wh{\theta}(\wh{S})$. Asterisked cities Milwaukee and Pittsburgh did not end up included in $\wh{S}$ since the misfit 
  resulting from shuffling happened not to be substantial enough. The superscript $\dagger$ refers to cities not affected by shuffling yet included
in $\wh{S}$.}\label{fig:cities}
            \end{figure}

\noindent\emph{Illustration.} An illustration of the above approach is provided in Figure~\ref{fig:cities}. The data set
consists of monthly average temperatures of $n = 46$ U.S.~cities as reported on~\cite{Temperatures2019}. The
data set is broken into two samples $\mc{X}$ and $\mc{Y}$ with the former containing the temperatures
of the odd numbered months (January, March, $\ldots$, November) and the latter containing the temperatures
of the even numbered months. For a random subset of $k = 10$ cities, we randomly permute matching records in
$\mc{X}$ and $\mc{Y}$. Linear regression is used to predict the $m = 6$ temperatures in $\mc{Y}$ from $\mc{X}$.
Due to high correlations among predictors, we work with the top $d = 3$ principal components as regressors. In the absence of
partial data shuffling, this yields a reasonable goodness of fit overall in terms of a coefficient of determination $R^2 \approx0.73$, apart from poor
model fit for several west coast cities (Los Angeles, San Diego, Seattle and San Francisco) with mild winters
and small seasonal differences, as well as for cities in desert
regions (Las Vegas and Phoenix) with extreme temperatures during summer. After data shuffling, model fit drops
to $R^2 \approx 0.4$. The approach outlined above shows some potential in this setting. With the choice
of $\lambda = \frac{1}{3} \cdot \wh{\sigma}_0/ \sqrt{n \cdot m}$, where $\wh{\sigma}_0$ is the estimated error variance from
the regression model in the absence of partial data shuffling, we ensure $R^2 \approx0.62$. Subsequent restoration
of the correct correspondence between $\mc{X}$ and $\mc{Y}$ is restricted to observations in
$\wh{S} = \{i:\;\nnorm{\wh{\Xi}_{i,:}}_2 \geq \sqrt{2m} \wh{\sigma}_0 \}$; for all other observations, no mismatches
are assumed, i.e., $\wh{\Theta}_{ii} = 1$, $i \notin \wh{S}$. The results highlight the challenges that
are encountered in the estimation of $\Theta^*$. Most crucially, the more an observation is distinct from the rest,
the easier it is identified as mismatch and the easier to retrieve its matching counterpart, with Fairbanks here being the
most distinct instance. On the other hand, the temperature differences between Milwaukee and Pittsburgh are only marginal, and
accordingly this mismatch remains undetected. Moreover, it is hard to disentangle cities affected by shuffling and poor
fit of the linear model, respectively. Nevertheless, re-matching succeeds for three cities (Fairbanks, Minneapolis, Tampa) and
gets close in case of Phoenix $\rightarrow$ Las Vegas and San Antonio $\rightarrow$ Phoenix.\\

\noindent\tcb{\emph{Alternatives to~\eqref{eq:grouplasso}}. Formulation~\eqref{eq:grouplasso} treats mismatches in the same way as generic data contamination (outliers). A promising alternative approach if an upper bound on $k$ is known and $m = 1$ can be found in~\cite{Bhatia2017}. A direct extension of this approach to the multiple response case with row-sparse contaminations is given by
{\small \begin{equation}\label{eq:Bhatia}
\wt{B} \in \argmin_{B \in \R^{d \times m}} \nnorm{Y - \wt{\Phi} - X B}_F^2, \;\; \text{where} \; \, \wt{\Phi} \in \argmin_{\Phi \in \R^{n \times m}}  \nnorm{\texttt{P}_{X}^{\perp} (Y - \Phi)}_F^2  \; \; \text{subject to} \, \su I(\Phi_{i,:} \neq \mathbf{0}) \leq k,
\end{equation}}where $\texttt{P}_{X}^{\perp}$ denotes the projection on the orthogonal complement of $\text{range}(X)$. Following~\cite{Bhatia2017}, the rightmost optimization problem in~\eqref{eq:Bhatia} is tackled via iterative hard thresholding~\citep{Blumensath2009}, and 
the result is substituted into the leftmost optimization problem to obtain an estimator for $B^*$. In our experiments, the performance of~\eqref{eq:Bhatia} is rather similar to that of the three-stage approach~\eqref{eq:refitting}.}    

\tcb{Given that both~\eqref{eq:grouplasso} and~\eqref{eq:Bhatia} treat mismatches as generic contaminations, it is worth exploring whether the additional structure under consideration here can be leveraged for improved performance. In the following, we present two approaches that are based on optimization over the polyhedron 
\begin{equation}
 \mc{C} =  \Big\{\Theta \in \R^{n \times n}:\; \Theta_{ij} \in [0,1], \, 1 \leq i,j \leq n, \; \textstyle\sum_j \Theta_{ij} \leq 1 , \, 1 \leq i \leq n \Big\}. \label{eq:matchings_relaxed}
\end{equation}
The first proposal can be seen as an immediate refinement of~\eqref{eq:grouplasso}:
\begin{equation}\label{eq:grouplasso_refined}
\min_{\Theta \in \mc{C}} \frac{1}{2n \cdot m}\nnorm{\texttt{P}_{X}^{\perp} \Theta Y}_F^2 + \lambda \textstyle\su \nnorm{Y^{\T}(I - \ \Theta)^{\T} e_i}_2,
\end{equation}
with $\texttt{P}_{X}^{\perp}$ as defined below~\eqref{eq:Bhatia} and $\{ e_i \}_{i = 1}^n$ denoting the canonical basis of $\R^n$. Similar to~\eqref{eq:grouplasso}, the penalty in~\eqref{eq:grouplasso_refined} is motivated by the fact that $(I - \Theta^*) Y$ has only few non-zero rows.}  

\tcb{Given an upper bound on $k$, an alternative to~\eqref{eq:grouplasso_refined} is given by the optimization problem
\begin{equation}\label{eq:grouplasso_refined_cons}
\min_{\Theta \in \mc{C}} \frac{1}{2n \cdot m}\nnorm{\texttt{P}_{X}^{\perp} \Theta Y}_F^2  \; \; \text{subject to} \; \textstyle\su \Theta_{ii} \geq n-k. 
\end{equation}
Given a minimizer $\wt{\Theta}$ of~\eqref{eq:grouplasso_refined} or~\eqref{eq:grouplasso_refined_cons}, an estimate
of $B^*$ is obtained via least squares regression of $\wt{\Theta} Y$ on $X$. Both~\eqref{eq:grouplasso_refined} and~\eqref{eq:grouplasso_refined_cons} are convex problems;~\eqref{eq:grouplasso_refined_cons} is a quadratic program. In spite of this,~\eqref{eq:grouplasso_refined} and~\eqref{eq:grouplasso_refined_cons} have significant computational drawbacks compared to the approaches~\eqref{eq:grouplasso} and ~\eqref{eq:Bhatia} since the former involve $n^2$ variables and thus scale poorly with problem size. According to own experiments, state-of-the art solvers for quadratic programs such as \textsf{cplexqp} in \textsf{CPLEX}\footnote{http://www.ibm.com/us-en/marketplace/ibm-ilog-cplex} take prohibitively long to solve instances of~\eqref{eq:grouplasso_refined_cons} even for $n = 200$. In Appendix~\ref{app:congradient}, we present reasonably practical algorithms for obtaining approximate solutions of~\eqref{eq:grouplasso_refined} and~\eqref{eq:grouplasso_refined_cons} based on the conditional gradient (aka Frank-Wolfe) method~\citep{Jaggi2013}, which are also used in an empirical comparison with our primary proposal~\eqref{eq:grouplasso} in $\S$\ref{sec:experiments}. In that comparison, neither~\eqref{eq:grouplasso_refined} nor~\eqref{eq:grouplasso_refined_cons} achieve substantial improvements over~\eqref{eq:grouplasso}.}

\section{Main results}\label{sec:mainresults}
This section provides theoretical results on the approach introduced in the previous section. Theorem
\ref{theo:parameter_estimation} quantifies the error in estimating $B^*$, while recovery of the
correct correspondence in terms of $\Theta^*$ is discussed in a separate subsection. \\
\begin{theorm}\label{theo:parameter_estimation}
  Consider model~\eqref{eq:contamination} and the minimizer $(\wh{B}, \wh{\Xi})$ of~\eqref{eq:grouplasso} with $\lambda \geq 2 \lambda_0$, where
  \begin{equation}\label{eq:lambdascaling}
\lambda_0 =  \frac{\mu_{n,d} \, \sigma}{\sqrt{n \cdot m}} \left( 1 + \sqrt{\frac{4 \log n}{m}}\right), \; \; \mu_{n,d} \coloneq \Big(\textstyle\frac{n-d}{n} + \textstyle\sqrt{24 \frac{\log n}{n}} \Big) \wedge 1, 
  \end{equation}
 and suppose $d/n < 1/4$. Then for any $\epss \in (0,1/3)$, there exists constants $c_{\epss}, c_{\epss}' > 0$ so that if $k \leq c_{\epss} n / \log(n/k)$, it holds that   
  \begin{align}\label{eq:xibound}
\frac{\nnorm{\wh{\Xi}  - \Xi^*}_F}{\sqrt{m}} &\leq 2\epss^{-2} \cdot \lambda \sqrt{m} \cdot \frac{\lambda + \lambda_0}{\lambda - \lambda_0} \sqrt{k}.
\end{align}
with probability at least $1 - 2/n  - 3.5 \cdot \exp(-c_{\epss}' n)$. Furthermore,    
\begin{align*}
\frac{\nnorm{\wh{B} - B^*}_F}{\sqrt{m}} \leq \frac{1}{1 - \sqrt{\frac{4 d \vee \log n}{n}}} \, \left( \sigma \sqrt{\frac{5 (d  \vee \log(n))}{n}} + \frac{\nnorm{\wh{\Xi}  - \Xi^*}_F}{\sqrt{m}} \right) 
\end{align*}
with probability at least $1 - 2\exp(-\frac{1}{2} (d \vee \log n)) - \exp(-(d \cdot m) \vee \log(n \cdot m))$. 
\end{theorm}
\noindent In order to better understand the consequences of Theorem~\ref{theo:parameter_estimation}, we spell out essential
scalings in $(n,k,d,m)$ below. According to~\eqref{eq:lambdascaling}, the parameter $\lambda$ should be chosen proportional
to 
\begin{equation}\label{eq:lambdscaling_spelledout}
\lambda_0 \asymp \textstyle\frac{1}{\sqrt{n \cdot m}} (1 + \sqrt{\log(n)/m})
\end{equation}
in which case $\frac{\nnorm{\wh{\Xi}  - \Xi^*}_F}{\sqrt{m}} \lesssim \sqrt{\frac{k}{n}} (1 + \sqrt{\log(n)/m})$ which are familiar rates for multivariate regression with block sparsity regularization~\citep{Lounici2011}. At the same time, the estimation error for the regression coefficients scales as $\frac{\nnorm{\wh{B} - B^*}_F}{\sqrt{m}} \lesssim \sqrt{d/n} + \frac{\nnorm{\wh{\Xi}  - \Xi^*}_F}{\sqrt{m}}$, where the first term on the right hand side equals the estimation rate of least squares regression in the absence of mismatches while the second term reflects the slack arising from the presence of the latter. The bottom line is that
the estimation error is in check as long as the fraction of mismatches $k / n$ is small. In fact, the condition preceding
\eqref{eq:xibound} imposes a bound on that fraction as well. In experiments, performance degrades more noticeably once $k/n >0.3$. Theorem~\ref{theo:parameter_estimation} also indicates a positive influence of the number of response variables $m$ in that one can choose $\lambda \asymp \frac{1}{\sqrt{n \cdot m}}$ once $m \gtrsim \log n$ which in turn eliminates the factor $\sqrt{\log n}$ in~\eqref{eq:lambdscaling_spelledout} and thus also in ~\eqref{eq:xibound}. This is a known benefit of block sparsity regularization in comparison to element-wise sparsity regularization~\citep{Lounici2011}.

\subsection*{Restoring Correspondence}

In this subsection, we study recovery of $\Theta^*$. To begin with, we suppose that the regression
parameter $B^*$ is known, and establish one sufficient and one necessary condition for exact
recovery of $\Theta^*$ based on the oracle estimator~\eqref{eq:Theta_oracle}. A crucial
quantity in the analysis is
\begin{equation}\label{eq:gamma}
\gamma^2 = \min_{i < j} \frac{\nnorm{B^{*\T} (\M{x}_i - \M{x}_j)}_2^2}{\nnorm{B^*}_F^2}, 
\end{equation}
the minimum squared distance among all pairs of linear predictors scaled by $\norm{B^*}_F^2$. A lower bound on $\gamma^2$ is
clearly needed in order to reliably match noisy responses $\{ \M{y}_i \}_{i = 1}^n$ to
the corresponding elements in $\{ B^{*\T} \M{x}_i \}_{i = 1}^n$: if there exists a pair
$(i,j)$ such that $\nnorm{B^{*\T} (\M{x}_i - \M{x}_j)}_2$ is smaller than the noise level, then
there is a good chance that the corresponding responses get swapped. The following two lemmas provide
upper and lower bounds on~\eqref{eq:gamma}.
\begin{lemma}\label{lem:gamma_min} Let $\srank(B^*) \coloneq \frac{\nnorm{B^*}_F^2}{\nnorm{B^*}_2^2}$ denote the stable
  rank of $B^*$, and consider $\gamma^2$ as defined in~\eqref{eq:gamma}. There exists
  universal constants $\alpha_0 \in (0,1)$ and $\kappa$ such that for any $\epss > 0$, with probability at least $1 - n^{-2\epss}$, it holds that
\begin{equation}\label{eq:gamma_latala}
\gamma^2 > \min\left\{2 n^{\frac{-2(1 + \epss)}{\kappa \cdot \srank(B^*)}},  \alpha_0 \right \}^2 
\end{equation}
\end{lemma}
\noindent The stable rank of $B^*$ as defined in the lemma crucially governs the scaling of $\gamma$. It
is instructive to consider the extreme case $\srank(B^*) = 1$: we then obtain $\gamma^2 \gtrsim n^{-C}$ for $C > 0$. Results in~\cite{SlawskiBenDavid2017}
on the case $m = 1$ show that $\gamma^2 \lesssim n^{-2}$ with constant probability, which indicates
sharpness of the above result in this case up to a constant in the exponent of $n$. On the other hand, if $\srank(B^*) = m \gtrsim \log n$, we have
\begin{equation*}
2 n^{\frac{-2(1 + \epss)}{\kappa \cdot \srank(B^*)}} = \exp \left(-\frac{2(1 + \epss)}{\kappa \cdot \srank(B^*)} \log(2n) \right) = \Omega(1),
\end{equation*}
i.e., the lower bound on $\gamma^2$ does no longer decay with $n$. Additional insights
can be obtained by considering the special case in which all non-zero singular values
of $B^*$ are equal to $b_* > 0$ and thus also $\srank(B^*) = \text{rank}(B^*) = r$. For
$r = 2(q + 1), \; q \geq 0$, the quantity~\eqref{eq:gamma} then becomes analytically tractable based
on a closed form expression for $\chi^2$-random variables with an even degrees of freedom. 

\begin{lemma}\label{lem:gamma_min_const} Consider $\gamma^2$ as defined in~\eqref{eq:gamma} and suppose that $B^*$ has
  exactly $r = 2(q + 1), \, q \in \{0,1,\ldots\}$ non-zero singular values equal to $b_* > 0$. Then for
  all $\delta > 0$
\begin{equation*}
\text{(Lower Bound):} \;\;\, \p\left(\gamma^2 \geq \frac{2}{e} (n^{-2} \, \delta)^{\frac{2}{r}} \right)\geq 1 - \delta/2. 
\end{equation*}
Moreover, if $n > 8 (r/2)^{r/2}$,
\begin{equation*}
\hspace*{-4ex}\text{(Upper Bound):} \;\;\, \p\left(\gamma^2  \leq  2 \cdot 8^{2/r}  n^{-2/r} \right) \geq0.75.
\end{equation*}
\end{lemma}
\noindent Lemma~\ref{lem:gamma_min_const} sheds some light on the range of the exponent 
$\kappa$ in the previous Lemma~\ref{lem:gamma_min}, and provides essentially matching upper and lower bounds on $\gamma^2$, where ``essentially'' refers to $n^{-4/r} \lesssim \gamma \lesssim n^{-2/r}$, i.e., the match is up to constant
factors and a factor $2$ in the exponent.

\tcb{In order to address the case of missing matches, we shall also consider}
\begin{equation}\label{eq:gamma0}
\gamma_0^2 = \min_{\substack{i \in \mc{N} \\  1\leq j \leq n}} \nnorm{\M{y}_i - B^{*\T} \M{x}_j}_2^2 / \nnorm{B^*}_F^2, 
\end{equation}
\tcb{where we recall that $\mc{N} = \{i: \,  \theta^*(i) = 0 \}$ denotes the set of missing matches.
The quantity~\eqref{eq:gamma0} exhibits scalings very similar to $\gamma^2$~\eqref{eq:gamma} as discussed in the remark following Lemma~\ref{theo:latala} in Appendix~\ref{app:separation}.}\\

Equipped with Lemma~\ref{lem:gamma_min} \&~\ref{lem:gamma_min_const}, we are in better position to interpret the following theorem.  \begin{theorm}\label{theo:correspondence} Let $\wh{B} = \wh{B}(X,Y)$ be an estimator of $B^*$, and let $\wh{\Theta}(\wh{B}) = \big(\wh{\Theta}_{ij}(\wh{B}) \big)$ denote the estimator~\eqref{eq:Theta_oracle} with $\tau > \tau_0 \coloneq \sigma (\sqrt{m} + 2 \sqrt{\log n})  + \max_{1 \leq j \leq n} \nnorm{\M{x}_j}_2 \nnorm{B^* - \wh{B}}_2$ and $B^*$ replaced by $\wh{B}$, i.e.,
\begin{equation*}
\wh{\Theta}_{ij}(\wh{B}) = \begin{cases} 1, &\quad \text{\emph{if}} \; j = \wh{j}(i) \; \text{\emph{and}} \; \, \nnorm{\M{y}_i - \wh{B}^{\T} \M{x}_{\wh{j}(i)}}_2 \leq \tau, \\
0 & \quad\text{\emph{otherwise}}, \qquad \qquad \qquad \qquad \qquad \qquad \quad 1 \leq i,j \leq n,
\end{cases}
\end{equation*}
where the index $\wh{j}(i)$ is defined by $\nnorm{\M{y}_i - \wh{B}^{\T} \M{x}_{\wh{j}(i)}}_2 = \min_{1 \leq j \leq n}\nnorm{\M{y}_i - \wh{B}^{\T} \M{x}_{j}}_2, \; 1 \leq i \leq n$.  
Let $\gamma^2$ and $\gamma_0^2$ be as in~\eqref{eq:gamma} and~\eqref{eq:gamma0}, respectively, and define the signal-to-noise ratio by $\textsf{\emph{SNR}} = \frac{\nnorm{B^*}_F^2}{\sigma^2 m}$. Consider the event
{\small\begin{equation*}
\mc{B} = \left\{ \min\{ \gamma_0^2, \gamma^2 \} \, \textsf{\emph{SNR}} >  36 \max \left\{ \frac{\nnorm{\wh{B} - B^*}_2^2}{\sigma^2 m} \max_{1 \leq i \leq n} \nnorm{\M{x}_i}_2^2,  2 \left( 1 + \sqrt{\frac{4\log n}{m}} \right)^2, \,  \frac{\tau^2}{\sigma^2 m}   \right\} \right\}.
\end{equation*}}Conditional on $\mc{B}$, with probability at least $1 - \p(\mc{B}^c) - 1/n$, $\{ \wh{\Theta}(\wh{B}) = \Theta^* \}$. Conversely, in the case that $\theta^*(i) \neq 0$ for $1 \leq i \leq n$, the following holds:
\begin{itemize}
\item There exists $c > 0$  so that if
  $\textsf{\emph{SNR}} <  c \frac{\log n}{m}$, $\p(\wh{\Theta}(B^*) \neq \Theta^*) \geq 1/3$.\\[-3ex] 
\item If additionally $m = O(1)$, there exists $c' > 0$ so that if $\min\{\gamma_0^2, \gamma^2\} \textsf{\emph{SNR}} < c'$,
  $\p(\wh{\Theta}(B^*) \neq \Theta^*) \geq 1/3$.\\
\end{itemize}
\end{theorm}

The above theorem contains both an achievability result in the form of a sufficient
condition for successful recovery of $\Theta^*$ given any estimator of $\wh{B}$, as well
as inachievability results concerning failure of recovery in the situation where $B^*$ is known. As explained in more detail below, the above sufficient and necessary conditions agree up to multiplicative constants in certain regimes. To shed more light on the implications of the theorem, it is instructive to consider
certain special cases of interest and to discuss them in connection with the error bounds stated in Theorem~\ref{theo:parameter_estimation}.
\begin{itemize}
\item[i)] The conditions of Theorem~\ref{theo:correspondence} involve \textsf{SNR} as the ratio of
the signal energy $\nnorm{B^*}_F^2/m$ per response variable and noise variance $\sigma^2$. If
$\wh{B} = B^*$ and every element of $\mc{Y}$ has match in $\mc{X}$, the condition of
the event $\mc{B}$ becomes
\begin{equation}\label{eq:suffSNR_Bknown}
\min \{\gamma_0^2, \gamma^2 \} \textsf{SNR} \geq 2 (1 + \sqrt{\log(n)/m})^2.  
\end{equation}
If $m = O(1)$, the scaling of $\gamma^2$ according Lemmas~\ref{lem:gamma_min} and~\ref{lem:gamma_min_const} imply that the condition $\textsf{SNR}  = \Omega(n^c)$ for a constant $c$ depending on $\text{srank}(B^*)$ suffices for recovery of $\Theta^*$. 
\item[ii)] The second bullet in Theorem~\ref{theo:correspondence} implies that for $m = O(1)$, the condition $\textsf{SNR}  = \Omega(n^c)$ is also necessary (up to a constant factor in the exponent of $n$). In particular,
Theorem~\ref{theo:correspondence} qualitatively recovers earlier results in~\cite{Pananjady2016} and~\cite{SlawskiBenDavid2017} on $m = 1$.
\item[iii)] Regarding the scaling of $m$, the threshold case appears to be $m \asymp \log n \asymp \srank(B^*)$. In this regime,~\eqref{eq:suffSNR_Bknown} requires only $\textsf{SNR} = \Omega(1)$ which is a far less stringent
condition compared to the regime of uniformly bounded $m$. Again, the sufficient condition is matched up
to a constant multiplicative factor by the necessary condition stated in the first bullet of Theorem~\ref{theo:correspondence}. 
\item[iv)] Once $m$ respectively $\srank(B^*)$ grow at a faster rate than $\log n$, the necessary condition of the first bullet is no longer aligned with~\eqref{eq:suffSNR_Bknown}. It remains an open
question whether Theorem~\ref{theo:correspondence} can be sharpened in this regard.
\end{itemize}
We now discuss the situation in which $B^*$ is replaced by an estimator $\wh{B}$. In the absence of mismatches, random matrix theory~\citep{Vershyninproduct} shows that ordinary least squares estimation obeys
$\E [\nnorm{\wh{B} - B^*}_2^2 / (\sigma^2 m)  ] \lesssim (d + m)/(n \cdot m)$
while $\max_{1 \leq i \leq n}\nnorm{\M{x}_i}_2^2 \lesssim d$ with high probability assuming that $d \gtrsim \log n$, which implies
that the first term in the outer ``$\max$'' of the event $\mc{B}$ is at best of the order $d^2/(n \cdot m)$. A slightly less
favorable condition is obtained when substituting the error bound of the proposed estimator in Theorem~\ref{theo:parameter_estimation}. In this case,
\begin{equation*}
\nnorm{\wh{B} - B^*}_2^2 / (\sigma^2 m) \leq \nnorm{\wh{B} - B^*}_F^2 / (\sigma^2 m) \lesssim (k + d)/n 
\end{equation*}
with the stated probability, and thus Theorem~\ref{theo:correspondence} yields the condition $n \gtrsim d \cdot (k \vee d)$. In summary, the
effect of replacing $B^*$ by the proposed estimator can either be compensated by imposing a more stringent
condition on \textsf{SNR} or the ratio $d / n$.

Lastly, let us comment on the case of missing matches, i.e., $\mc{N} \neq \emptyset$, and the choice
of $\tau$. As long as $\tau$ is chosen proportional to the threshold $\tau_0$, the requirements on the \textsf{SNR}
remain qualitatively unchanged. The dependence of $\tau_0$ on the noise level is intrinsic, hence approximate knowledge of $\sigma$ is inevitable to guide the choice of $\tau$. While $\tau_0$ also depends on $\nnorm{\wh{B} - B^*}_2$, the latter can be estimated 
given bounds on the estimation error as discussed in the preceding paragraph. Clearly, $\tau$ can be set to zero whenever it is known that $\mc{N} = \emptyset$.




\subsection*{Identification of Mismatched Data}
In the following, we discuss a simpler task than recovery of $\Theta^*$, namely recovery
of $S_* = \{1 \leq i \leq n:\;\theta^*(i) \neq i \}$, or equivalently,
$S_* = \{1 \leq i \leq n:\;\Xi^*_{i,:} \neq 0 \}$ with $\Xi^* = \Phi^* / \sqrt{n}$ as
defined in~\eqref{eq:contamination}. The following statement provides a condition
that ensures that we can separate mismatched data $S_*$ and correctly matched data
$S_*^c$ in terms of $\{ \nnorm{\wh{\Xi}_{i,:}}_2 \}_{i = 1}^n$, where $\wh{\Xi}$
is obtained from optimization problem~\eqref{eq:grouplasso} and analyzed in Theorem
\ref{theo:parameter_estimation}.
\begin{prop}\label{prop:discovery_mismatches}
 Let $\wh{\Xi}$ be as in Theorem~\ref{theo:parameter_estimation}, and let $\gamma_0^2$, $\gamma^2$, and \textsf{\emph{SNR}} be as in Theorem~\ref{theo:correspondence}. We then have $\min_{i \in S_*} \nnorm{\wh{\Xi}_{i,:}}_2 > \max_{i \in S_*^c} \nnorm{\wh{\Xi}_{i,:}}_2$ if 
  \begin{equation}\label{eq:mismatch_recovery_rhs}
  \min\{\gamma_0^2, \gamma^2 \}\textsf{\emph{SNR}} \geq  \frac{4 \max_{1 \leq i \leq n} \nnorm{\sqrt{n} ( \wh{\Xi}_{i,:} - \Xi_{i,:}^*)}_2^2}{\sigma^2 m}.
  \end{equation}
\end{prop}
\noindent The practical consequences are as follows: if it holds that $\min_{i \in S_*} \nnorm{\wh{\Xi}_{i,:}}_2 > \max_{i \in S_*^c} \nnorm{\wh{\Xi}_{i,:}}_2$, we can sort the $\{ \nnorm{\wh{\Xi}_{i,:}}_2 \}_{i = 1}^n$ and retain the observations
corresponding to the $\lfloor \nu n \rfloor$ smallest elements for $\nu \in (0, (1-k/n)]$. Any choice of $\nu = \Omega(1)$ in
that range identifies $Q \subseteq S_*^c$ with $|Q| = \Omega(n)$. The least squares estimator $\wt{B}$ of $B^*$ using observations
in $Q$ only, i.e.,
\begin{equation*}
\wt{B} = \argmin_{B \in \R^{d \times m}} \sum_{i \in Q} \nnorm{\M{y}_i - B^{\T} \M{x}_i}_2^2
\end{equation*}
can substantially improve over the estimator $\wh{B}$ in Theorem~\ref{theo:parameter_estimation}. The condition of Proposition~\ref{prop:discovery_mismatches} tends to be easier to satisfy than that for recovery
of $\Theta^*$ in Theorem~\ref{theo:correspondence}. The right hand side of~\eqref{eq:mismatch_recovery_rhs} 
is of the order $O(1 + \log(n)/m)$ and $O(k\{1 + \log(n/m) \})$ in the best and worst case, respectively, in view of Theorem~\ref{theo:parameter_estimation}; the best case is obtained if $\max_i \nnorm{\wh{\Xi}_{i,:} - \Xi_{i,:}^*}_2^2 \lesssim \nnorm{\wh{\Xi} - \Xi^*}_F^2/k$, i.e., the error in Frobenius norm is spread out roughly evenly over $\Omega(k)$ rows.  



\section{Experiments}\label{sec:experiments}
In the sequel, we present empirical evidence supporting central aspects of our analysis, \tcb{and provide numerical comparisons to the alternative methods outlined at the end of $\S$2 as well as to an extension of the EM scheme in~\cite{Wu1998, Abid2018} for multiple response variables}. For simplicity, we confine ourselves to the case in which $\Theta^*$ is a permutation matrix, i.e., an element of~\eqref{eq:permutations}. Accordingly, the minimization in~\eqref{eq:Theta_oracle_p} is performed over the set of permutation matrices by means of the Auction Algorithm~\citep{Bertsekas1992}. We note that this modification does not affect our theoretical results. Specifically, the achievability result in Theorem~\ref{theo:correspondence} continues to hold because it asserts recovery over a superset of~\eqref{eq:permutations}. Similarly, the inachievability results continue to hold if $\Theta^*$ is required to be a permutation.\\

\noindent {\bfseries Synthetic data.}\\
\emph{Setup}. Data is generated according to the model 
\begin{equation*}
\M{y}_i = B^{*\T} \M{x}_{\theta^*(i)} + \sigma \eps_i, \; i=1,\ldots,n,
\end{equation*}
where the $\{ \M{x}_i \}_{i = 1}^n$ and $\{ \eps_{i} \}_{i = 1}^n$, are i.i.d.~from $N(0, I_d)$
and $N(0, I_m)$, respectively, $\theta^*$ is a random permutation that shuffles $\{1,\ldots,k\}$
uniformly at random, and is the identity map when restricted to the remaining indices, i.e,
$\theta^*(i) = i$ for $i > k$. The matrix $B^*$ is obtained by first generating a $d$-by-$d$ matrix
(i.e., $d = m$) with i.i.d.~$N(0,1)$-entries, then computing its singular value decomposition
$B^* = U S V^{\T}$, and replacing the diagonal entries $\{s_1, \ldots, s_d \}$ of $S$ according to 
$s_j \leftarrow j^{-q}$, $1 \leq j \leq d$ for $q \in \{0, 0.05, 0.1, 0.2, 0.5, 1, 2, 5\}$; finally,
$B^*$ is re-scaled such that $\nnorm{B^*}_F^2 = m$. This construction ensures that the stable
rank $\srank(B^*)$, which has a critical influence on the recovery of $\Theta^*$, varies between
$m = d$ (achieved for $q = 0$) and $1$ (achieved for $q \rightarrow \infty$). In addition, the
signal-to-noise ratio then results as $\textsf{SNR} = \sigma^{-2}$ with $\sigma \in \{0.01, 0.02, 0.05, 0.1, 0.2, 0.5, 1, 2\}$. Lastly,
the fraction of mismatches $k/n$ varies between $0.05$ and $0.4$ in steps of $0.05$ with $n \in \{200, 500, 1000\}$ and
$d/n \in \{0.03,0.06,0.12 \}$. For each configuration of $(n,d,k, q, \sigma)$, 100 independent replications
are performed. The following approaches are compared. 
\vskip1ex
\noindent {\bfseries \textsf{naive}, \textsf{oracle}}. Plain least squares and estimation of $B^*$ with knowledge of $\Theta^*$, respectively.  
\vskip1ex
\noindent {\bfseries \textsf{proposed}}. $B^*$ is estimated according to~\eqref{eq:grouplasso}   with the choice $\lambda = \lambda^{\star} = 4 \sigma \frac{1}{\sqrt{n \cdot m}}$ which is the lower bound on $\lambda$ suggested
  by Theorem~\ref{theo:parameter_estimation} when treating $\sqrt{4 \log(n)/m}$ simply as $1$.
\vskip1ex
\noindent {\bfseries \textsf{proposed+}}. The re-fitting approach~\eqref{eq:refitting} building on {\bfseries \textsf{proposed}}, cf.~also Proposition~\ref{prop:discovery_mismatches}. Assuming that $k$ is known, the set of mismatches $S_*$ is estimated by $\wh{S} = \{1 \leq i \leq n: \; \nnorm{\wh{\Xi}_{i,:}}_2 > t_{(n-k)}\}$,
  where $t_{(i)}, \, 1 \leq i \leq n$, denotes the $i$-th order statistic of the $\{ \nnorm{\wh{\Xi}_{i,:}}_2 \}_{i = 1}^n$.
\vskip1ex
\noindent {\bfseries \textsf{CRR}}. ``\underline{C}onsistent \underline{R}obust \underline{R}egression", following the title for  the approach~\eqref{eq:Bhatia} used in~\cite{Bhatia2017}. The number of mismatches $k$ is assumed to be known. 
\vskip1ex
\noindent {\bfseries \textsf{EM}}. The EM-scheme in~\cite{Wu1998, Abid2018} in which $\Theta^*$ is treated as missing data in conjunction with the use of the EM algorithm. Since the E-step involves intractable integration over the set of permutation matrices, MCMC is employed to approximate this step. In our implementation, the permutation is initialized as the identity, and the number of MCMC iterations per EM iteration is set to 10,000 given a "burn-in period" of 1,000.
\vskip1ex
\noindent {\bfseries \textsf{DS-reg}}. The approach~\eqref{eq:grouplasso_refined} that arises as a refinement of {\bfseries \textsf{proposed}}, and here involves optimization over the set of doubly stochastic matrices of size $n$. 
We consider $\lambda \in 2^{-p} \lambda^{\star}$, $p \in \{-1,0,\ldots,3 \}$, with $\lambda^{\star}$ as in the description of {\bfseries \textsf{proposed}} above, and choose $p$ replication by replication to minimize
the estimation error w.r.t.~$\nnorm{\cdot}_F$ of the resulting estimator of $B^*$.  
\vskip1ex
\noindent {\bfseries \textsf{DS-cons}}. The approach~\eqref{eq:grouplasso_refined_cons} with $k$ assumed to be known. 
\vskip1ex
\noindent {\bfseries \textsf{DS-reg+}, \textsf{DS-cons+}}. Re-fitting approaches associated with {\bfseries \textsf{DS-reg}} and {\bfseries \textsf{DS-cons}}. The set $S_*$ is estimated by $\wt{S} = \{1 \leq i \leq n: \; \wt{\Theta}_{ii} < \wt{t}_{(n-k)}\}$, where $\wt{\Theta}$ is the estimator of $\Theta^*$ from~\eqref{eq:grouplasso_refined} and~\eqref{eq:grouplasso_refined_cons}, respectively, and $\wt{t}_{(i)}, \, 1 \leq i \leq n$, denotes the $i$-th order statistic of $\{ \wt{\Theta}_{ii} \}_{i = 1}^n$.
\vskip2ex
\noindent Since solving the optimization problems associated with {\bfseries \textsf{DS-reg}} and {\bfseries \textsf{DS-cons}} entails
substantial additional efforts even with customized solvers (Appendix~\ref{app:congradient}) given $O(n^2)$ variables, we only consider a reduced set of configurations for $(n,d,k,q,\sigma)$ with 
$n \in \{ 200, 500 \}$, $d/n = 0.03$, and $q = 0$, while the ranges for $k/n$ and $\sigma$ remain unchanged. In addition, the number of replications per configuration is lowered to $20$.  
\vskip2ex
\noindent \emph{Results (I): Estimation of $B^*$.}  For better comparison across experimental configurations, we visualize the following ``standardized''  estimation error
\begin{equation}\label{eq:normalized_estimation}
  \sigma^{-1} m^{-1/2} \nnorm{B^{\text{est}} - B^*}_F  - \sqrt{d/n}, 
\end{equation}
where $B^{\text{est}}$ is a placeholder for the various estimators mentioned in the previous paragraph. Note that 
\eqref{eq:normalized_estimation} approximately equals zero in expectation for the oracle estimator equipped with $\Theta^*$, thus~\eqref{eq:normalized_estimation} can be interpreted as the excess error relative to that oracle. For the estimator $\wh{B}$ analyzed in Theorem~\ref{theo:parameter_estimation}, the quantity~\eqref{eq:normalized_estimation} is expected to be proportional to $\sqrt{k/n}$. Selected results are shown in Figure~\ref{fig:estimation_errors}, which displays averages of~\eqref{eq:normalized_estimation} for $n \in \{500, 1000 \}$ and 
$\sigma \in \{ 0.05, 0.1, 0.2 \}$; the number of different values for $\sigma$ considered in a single plot had
to be limited to ensure readability since for {\bfseries \textsf{naive}} and {\bfseries \textsf{EM}},~\eqref{eq:normalized_estimation} still depends substantially on $\sigma$. To account for that, shaded areas are 
used to represent the ranges of~\eqref{eq:normalized_estimation} for those two approaches; the upper and lower 
margins of the shaded areas represent the normalized estimation error for $\sigma = 0.05$ and $\sigma = 0.2$, respectively, while the dashed lines inside the shaded areas correspond to $\sigma = 0.1$. Accordingly, the performance of {\bfseries \textsf{naive}} and {\bfseries \textsf{EM}} (initialized by {\bfseries \textsf{naive}}) relative to~\eqref{eq:normalized_estimation} improves, which is 
unsurprising given that as $\sigma 	\sqrt{m} \nearrow  \nnorm{B^*}_F$ (recall that $\nnorm{B^*}_F = \sqrt{m}$), the error induced by mismatches is of the same
order as the noise in which case the gap between {\bfseries \textsf{naive}} and {\bfseries \textsf{oracle}} narrows. 
With the same reasoning, remedies for mismatches compared here are most effective if $\sigma \sqrt{m}/ \nnorm{B^*}_F$ is small: for example, {\bfseries \textsf{proposed}} achieves a roughly tenfold reduction in standardized estimation error over {\bfseries \textsf{naive}} for $\sigma = 0.05$; that margin reduces gradually with increasing $\sigma$. 
\begin{figure}[h]
\begin{center}
\begin{tabular}{ccc}
  \hspace*{-1.7ex}\includegraphics[width = 0.32\textwidth]{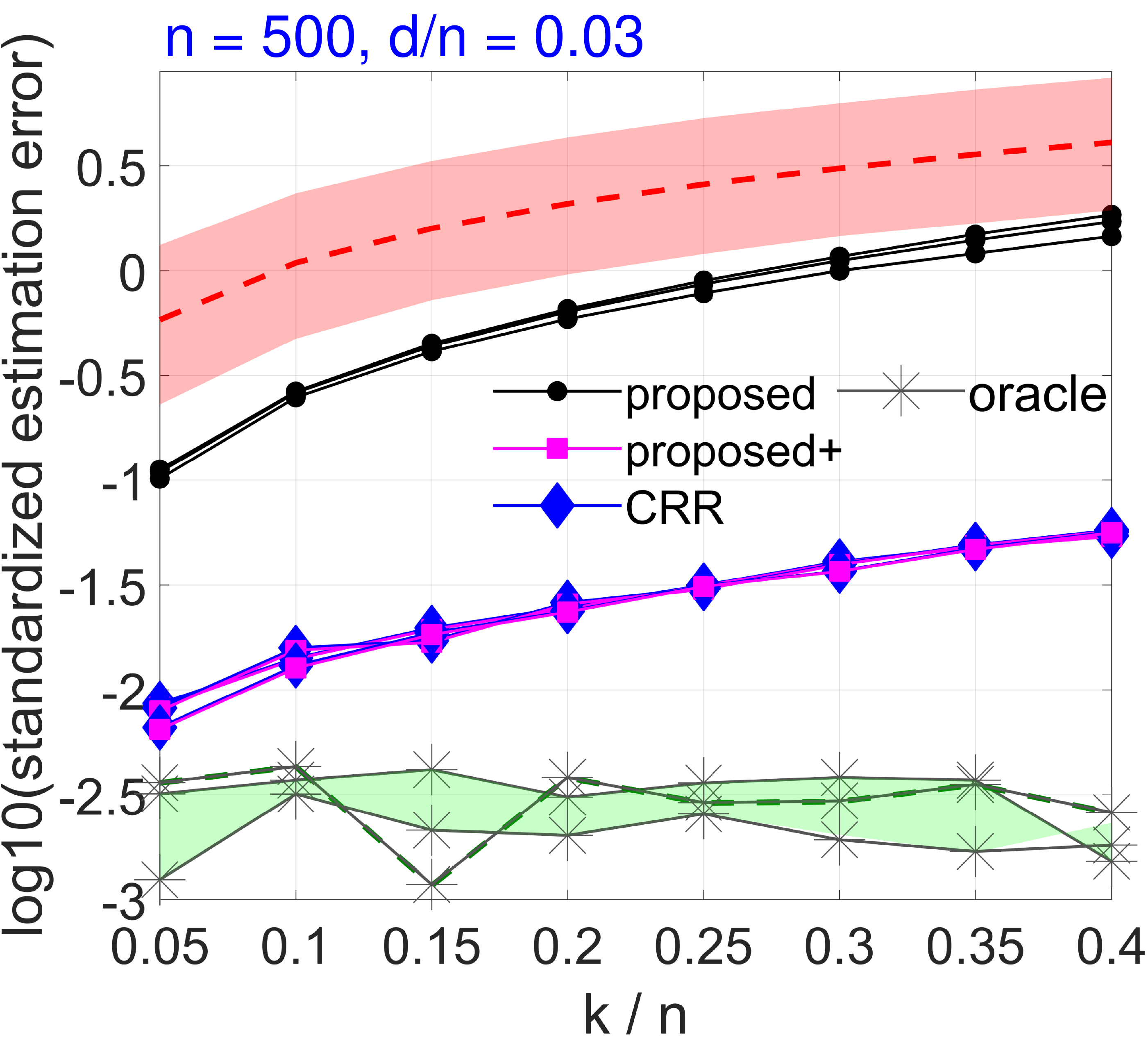}
  & \includegraphics[width = 0.32\textwidth]{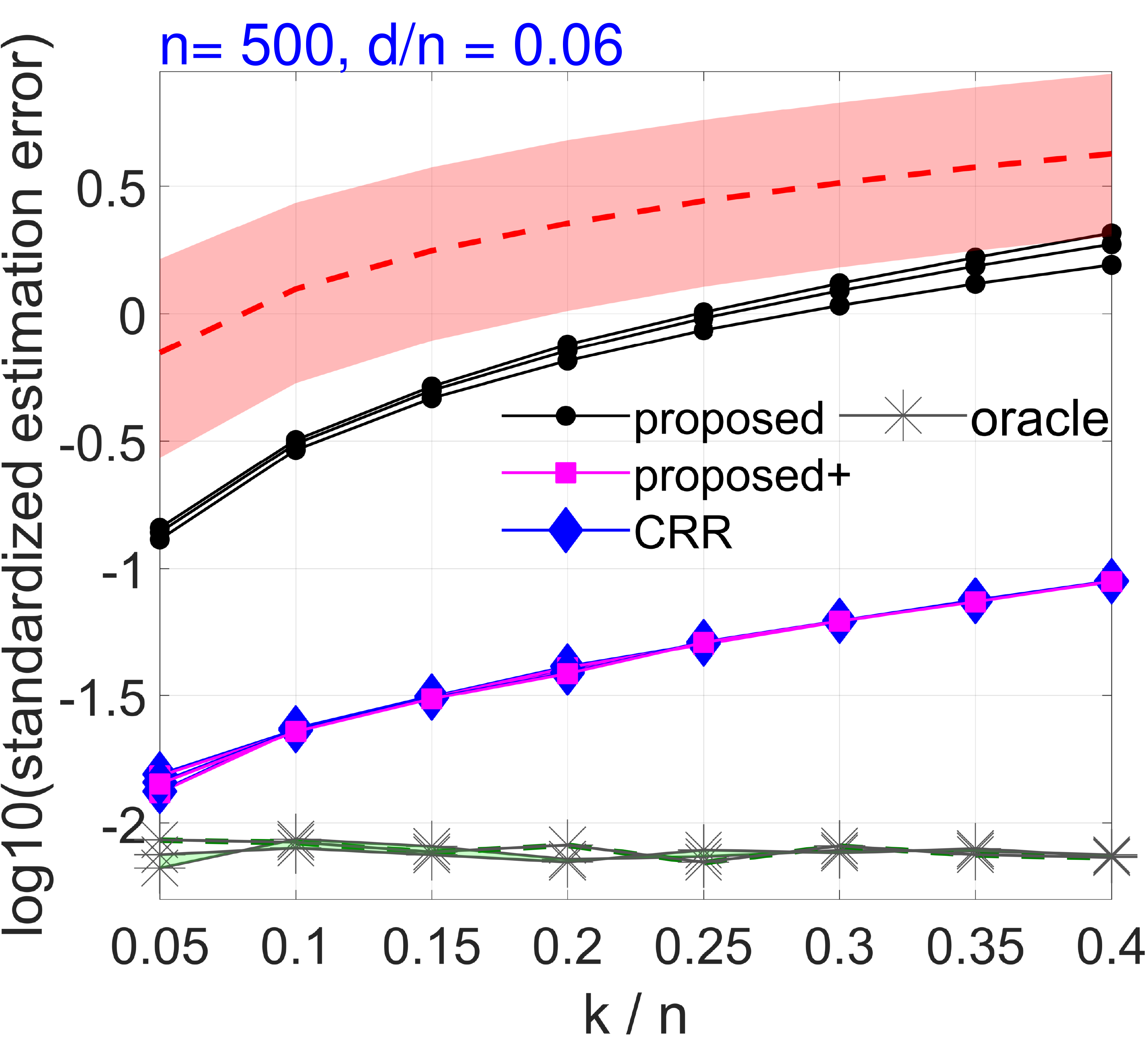} 
  & \includegraphics[width = 0.32\textwidth]{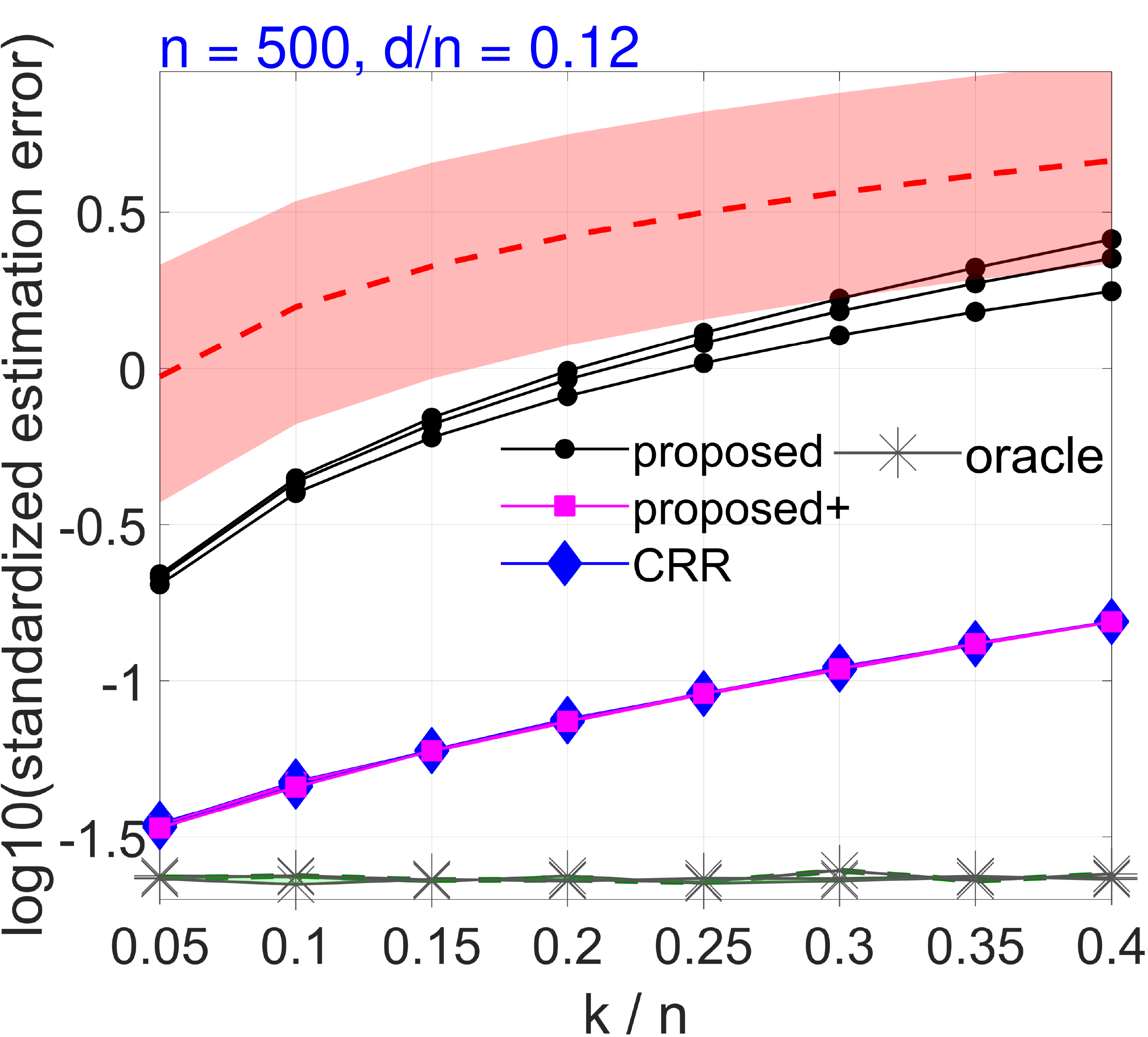}  \\[-.7ex]
   \hspace*{-1.7ex} \includegraphics[width = 0.32\textwidth]{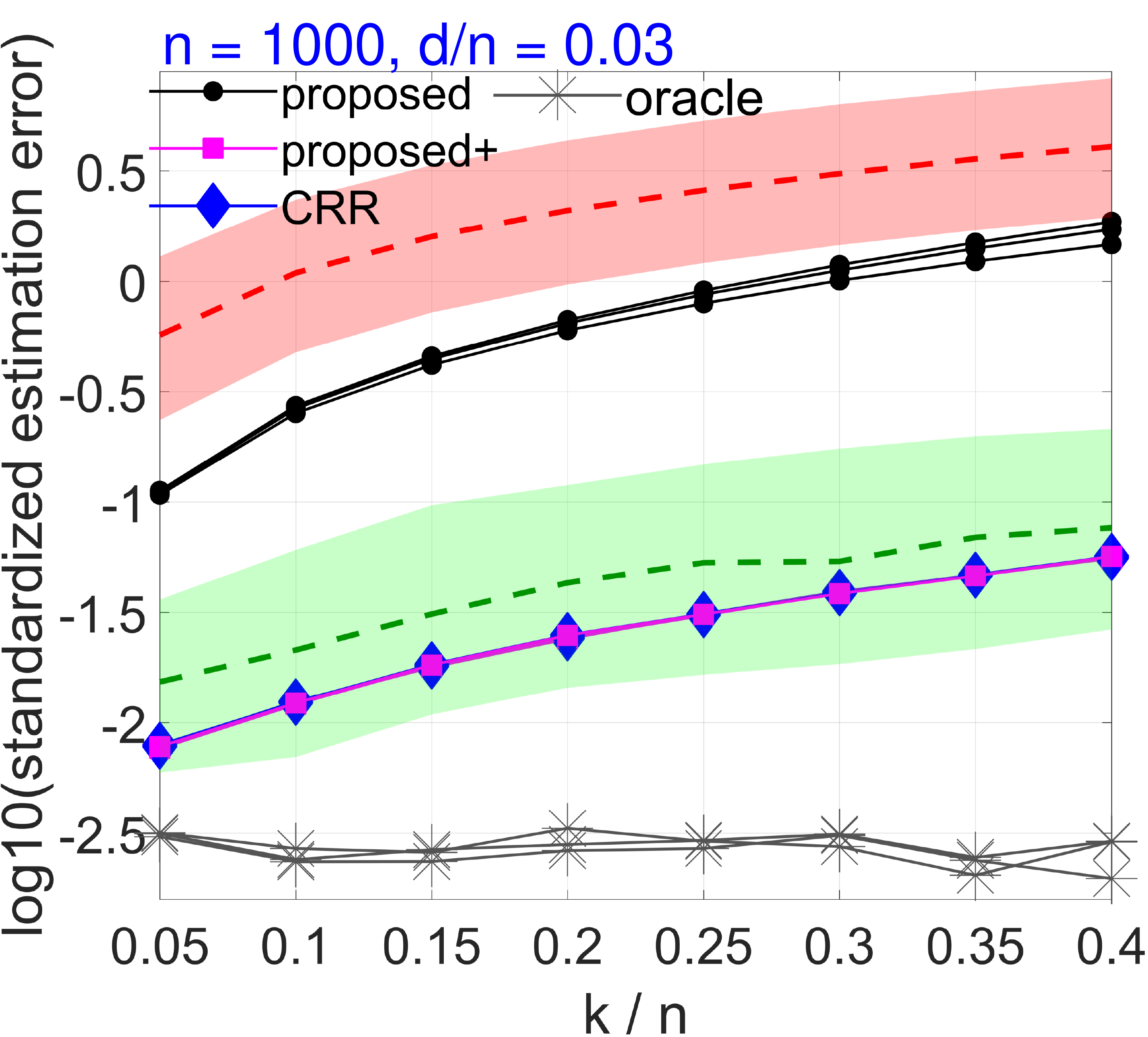}
  & \includegraphics[width = 0.32\textwidth]{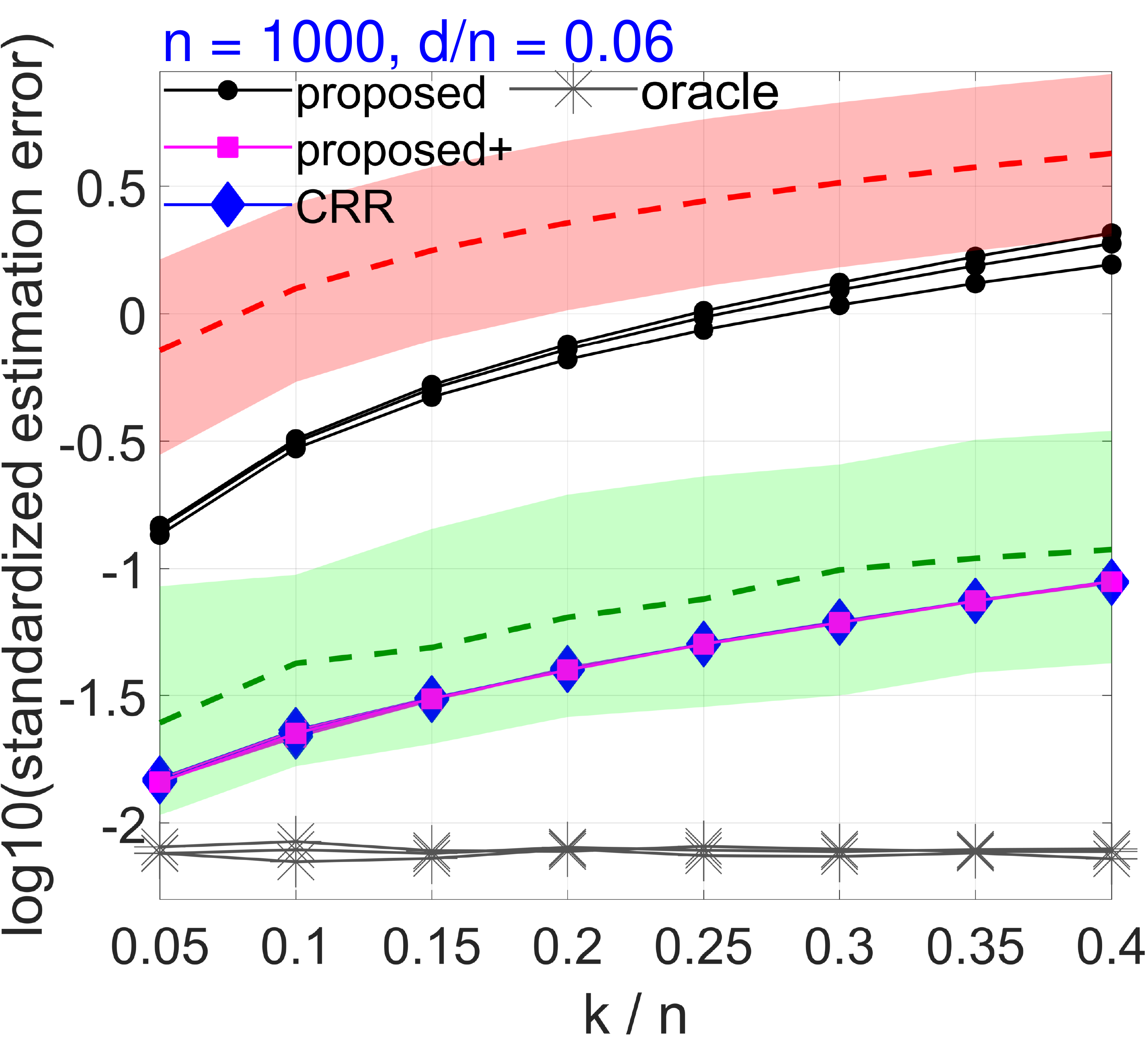} 
  & \includegraphics[width = 0.32\textwidth]{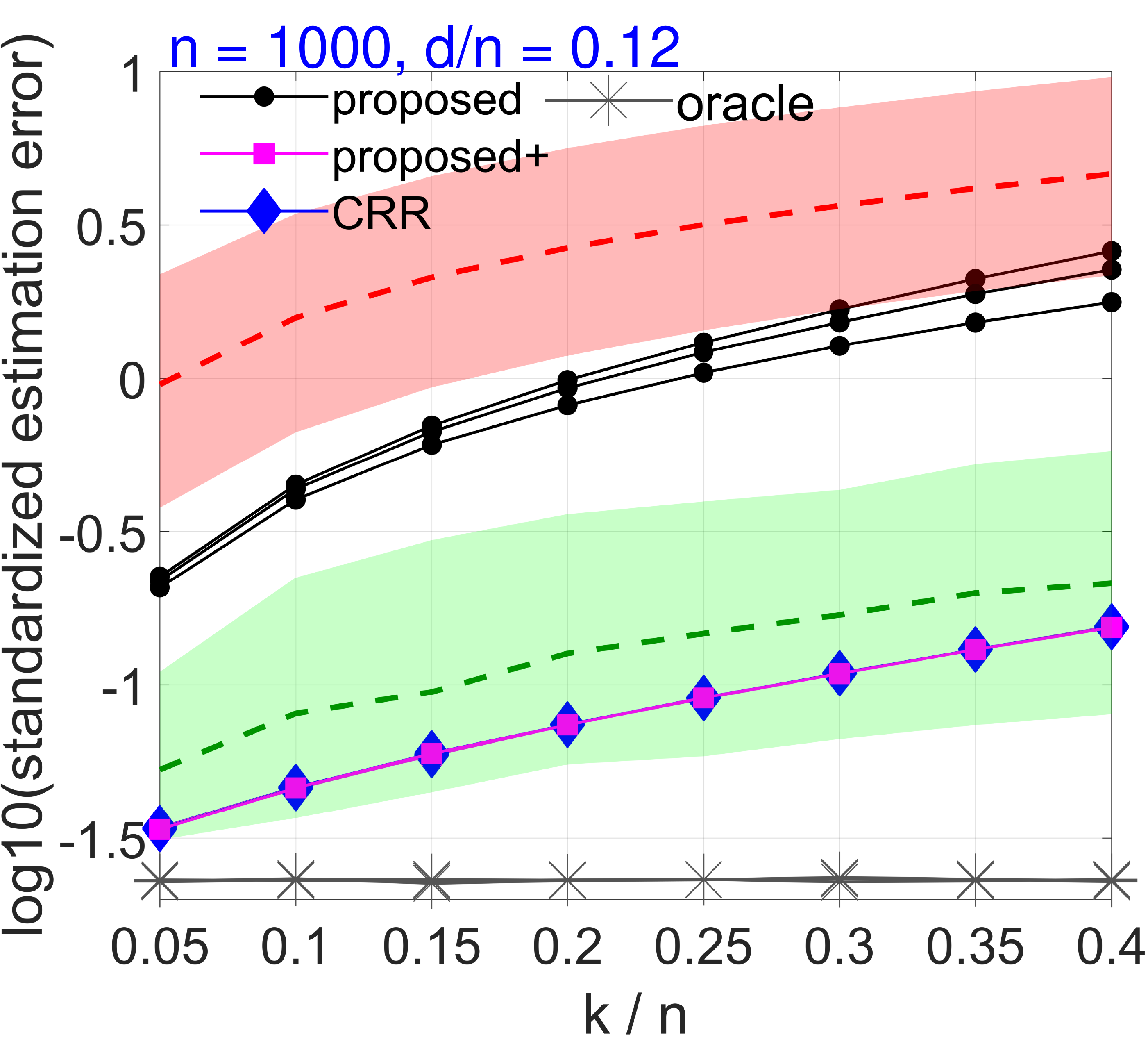} 
\end{tabular}
\end{center}
\vspace*{-5ex}
\caption{Average standardized estimation errors~\eqref{eq:normalized_estimation} on a $\log_{10}$-scale, with one curve for each $\sigma \in \{0.05, 0.1, 0.2 \}$. For {\bfseries \textsf{naive}} (in red) and {\bfseries \textsf{EM}} (in green), the resulting curves do not cluster together, and are hence captured by the upper ($\sigma = 0.05$) and lower ($\sigma = 0.2$) boundaries of the shaded areas plus a dashed line ($\sigma = 0.1$).}\label{fig:estimation_errors}
\end{figure}

Figure~\ref{fig:estimation_errors} also shows that refitting after applying {\bfseries \textsf{proposed}} and estimating $S_*$ considerably boosts performance. The performance of the resulting approach {\bfseries \textsf{proposed+}} is indistinguishable from {\bfseries \textsf{CRR}}. While {\bfseries \textsf{EM}} performs
on par with the oracle for $n = 500$ (and $n = 200$, not shown), the approach degrades with $n$. One likely explanation is that the challenges associated with the E-step become more severe with $n$: specifically, the MCMC approximation tends to be less reliable for larger values $n$. For the same reason, {\bfseries \textsf{EM}} is at least an order of magnitude slower than {\bfseries \textsf{proposed+}} and {\bfseries \textsf{CRR}}.

\begin{figure}[h!]
\begin{center}
\begin{tabular}{ccc}
\hspace*{-1.7ex} \includegraphics[width = 0.32\textwidth]{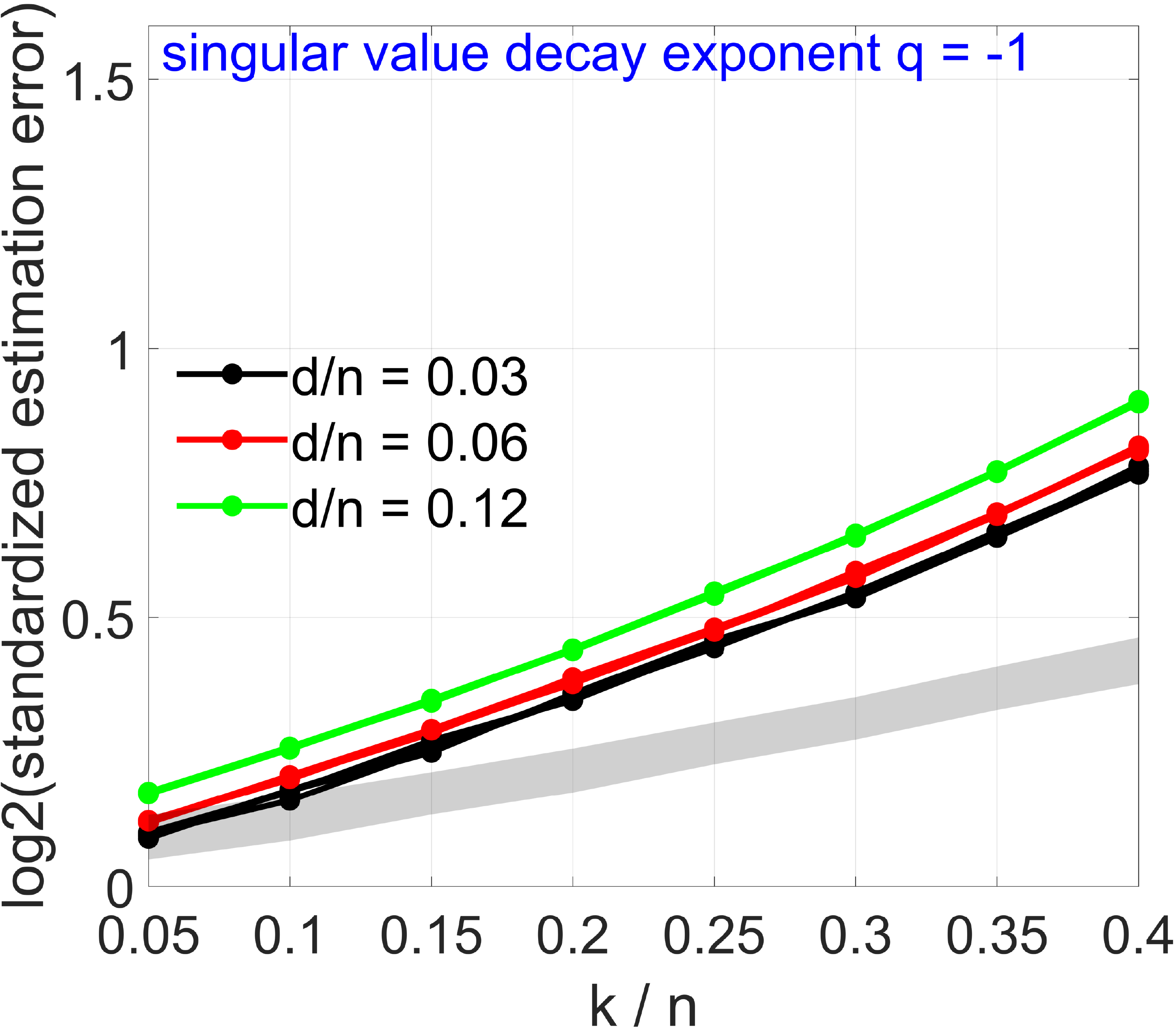}
  & \hspace*{-1.2ex} \includegraphics[width = 0.32\textwidth]{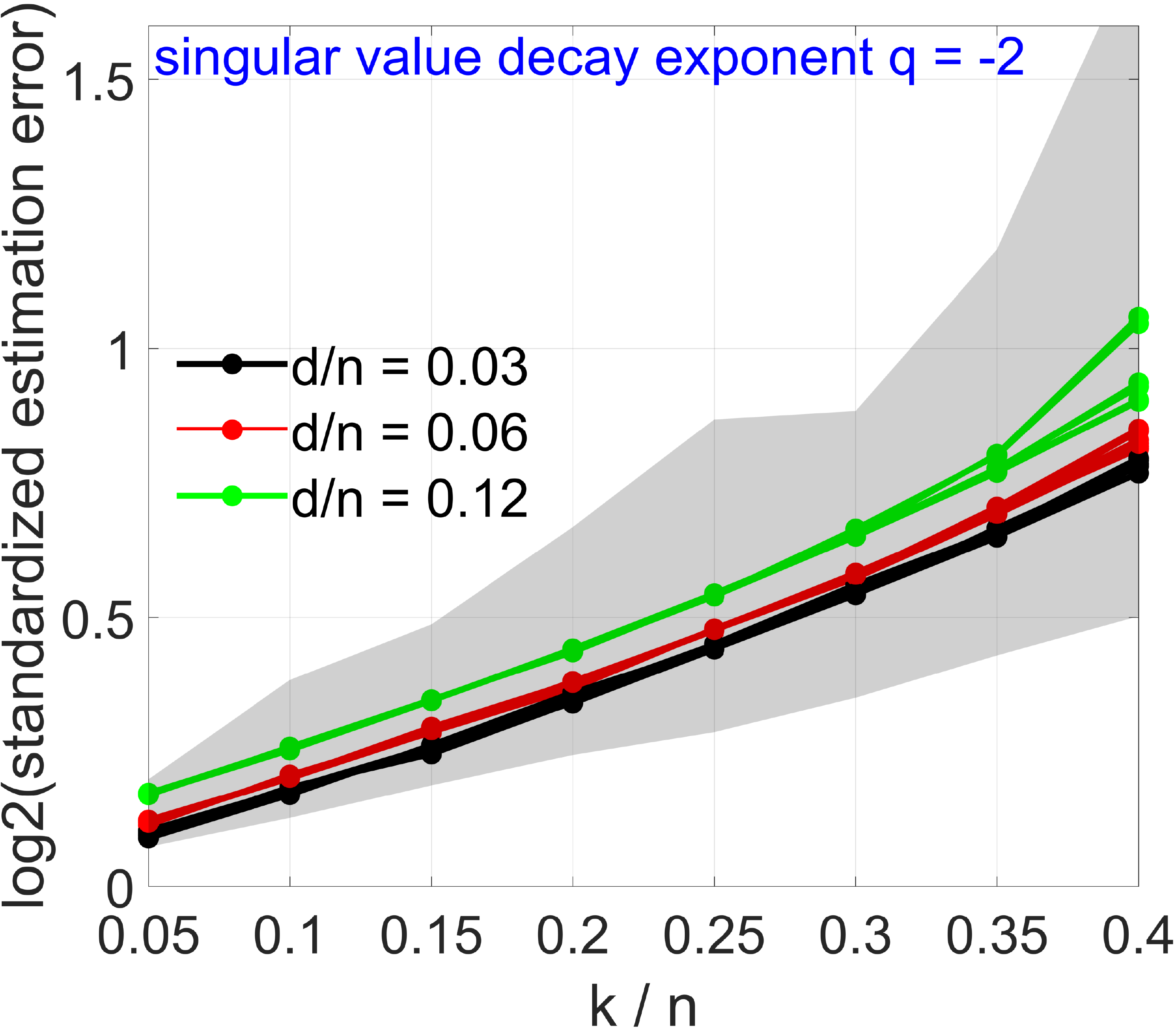} 
  & \hspace*{-1.2ex}\includegraphics[width = 0.32\textwidth]{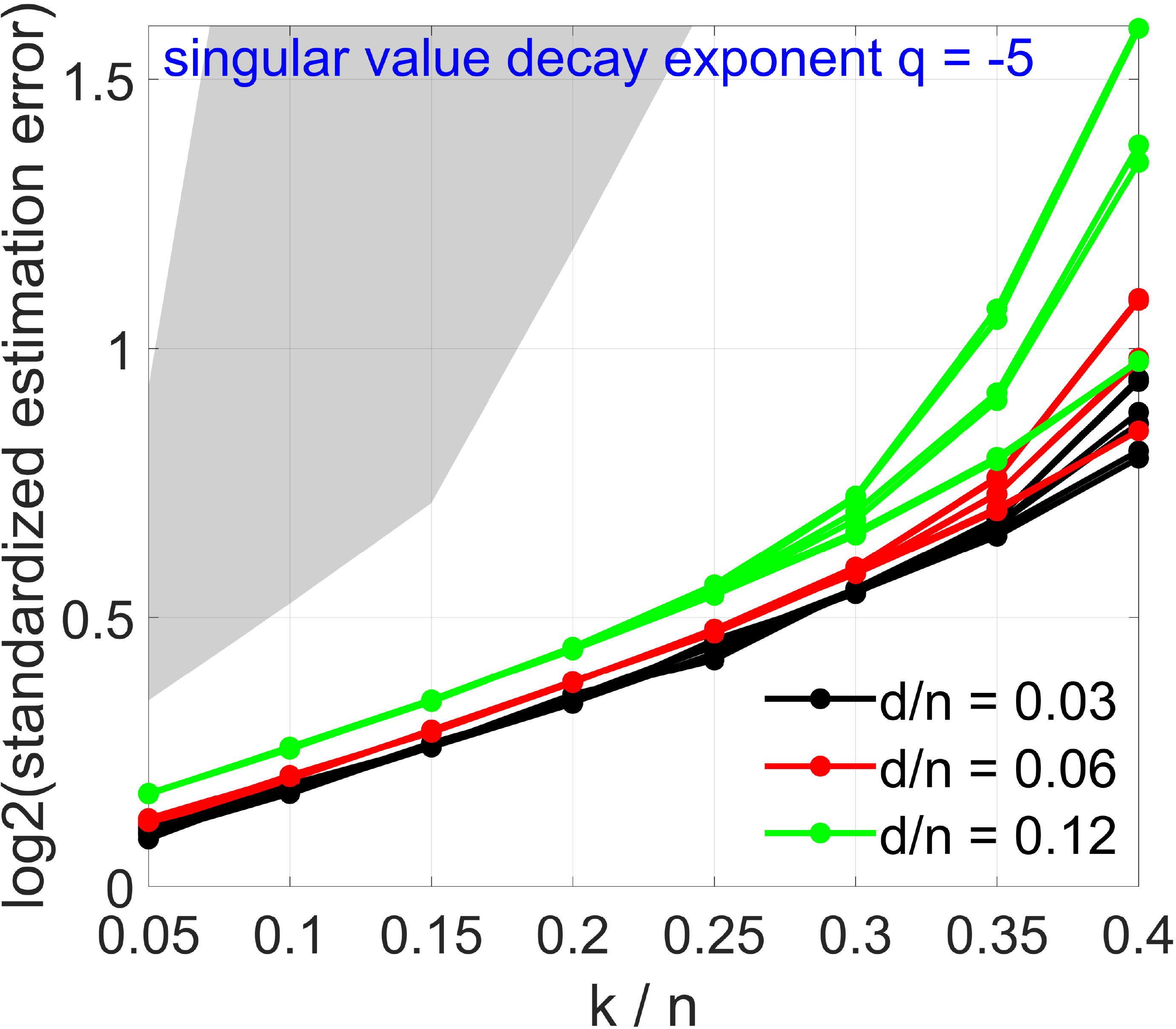}
\end{tabular}
\end{center}
\vspace*{-5ex}
\caption{Average standardized estimation errors $\frac{\sigma^{-1} m^{-1/2}  \nnorm{B^{\text{est}} - B^*}_F}{(d/(n-k))^{1/2}}$ ($\log_2$-scale) of the re-fitting approach {\bfseries \textsf{proposed+}} (lines) and {\bfseries \textsf{EM}} (shaded areas) for different rates of decay of the singular values of $B^*$ corresponding to decreasing $\text{srank}(B^*)$ from left to right. Curves for different combinations of $n$ and $\sigma$ appear in the same plots; due to poor clustering of those curves for {\bfseries \textsf{EM}} in conjunction with the chosen error normalization, their range is indicated by shaded areas for better readability.}\label{fig:estimation_errors_refit}  
\end{figure}

In Figure~\ref{fig:estimation_errors_refit}, the performance of {\bfseries \textsf{proposed+}} relative to {\bfseries \textsf{EM}} is investigated in more detail. In addition to poor scalability with $n$, the competitiveness of {\bfseries \textsf{EM}} also hinges on the stable rank of $B^*$ not to be too small. The sequence of three plots in 
Figure~\ref{fig:estimation_errors_refit} indicates a transition from superior to comparable and eventually not competitive performance of {\bfseries \textsf{EM}} as the singular values in $B^*$ decay more rapidly.   

Finally, Figure~\ref{fig:DS} provides a comparison to the approaches {\bfseries \textsf{DS-reg}} and {\bfseries \textsf{DS-cons}}. Despite the additional sophistication involved, the results only indicate minor improvements, 
which largely disappear when considering refitting. In particular, the observed gains in performance do not appear to justify the massive computational effort associated with the solution of the optimization problems underlying {\bfseries \textsf{DS-reg}} and {\bfseries \textsf{DS-cons}}.   

\begin{figure}[h!]
\begin{center}
\begin{tabular}{cc}
 \hspace*{.3ex} \includegraphics[width = 0.36\textwidth]{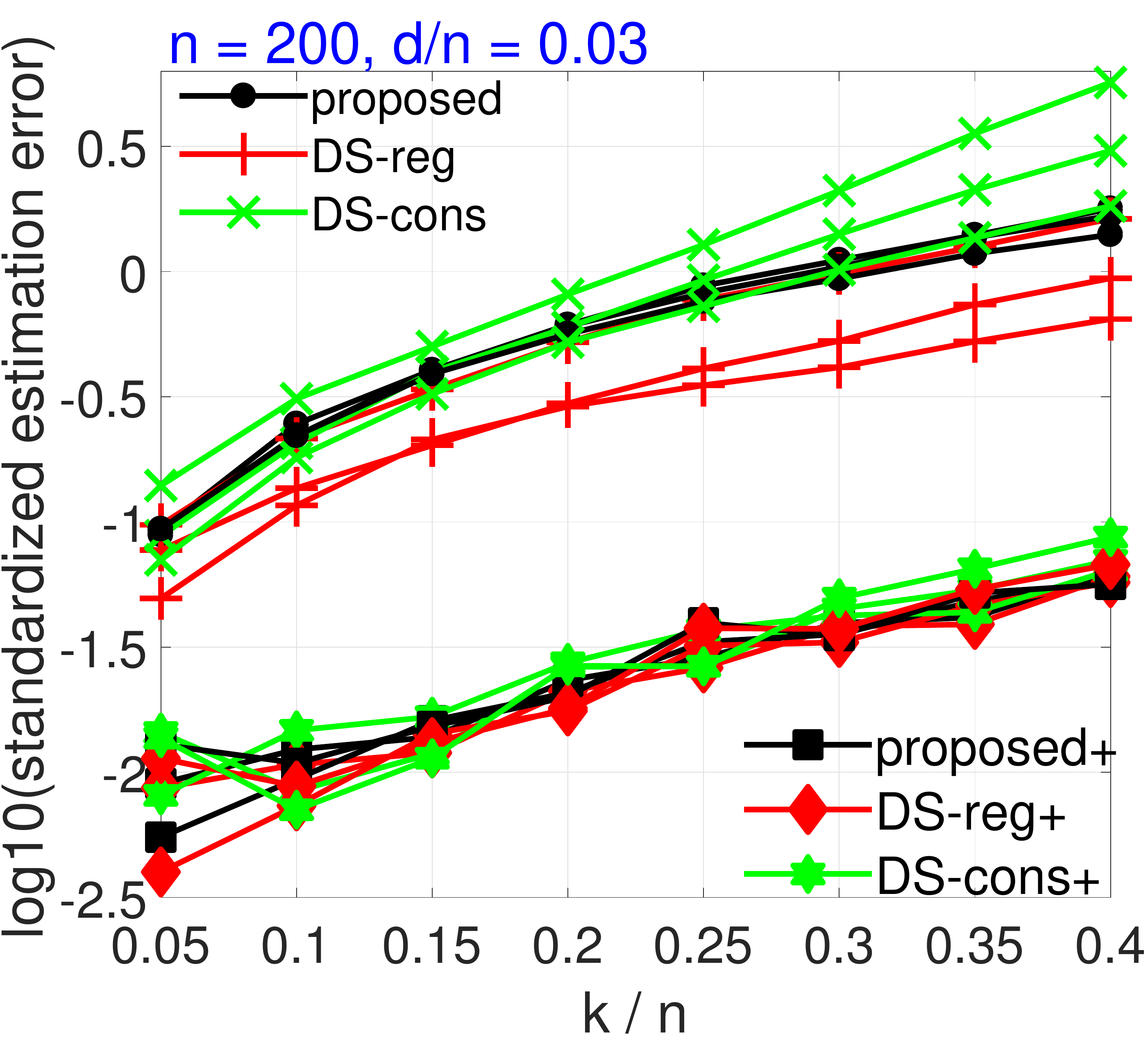}
  & \hspace*{8.4ex} \includegraphics[width = 0.36\textwidth]{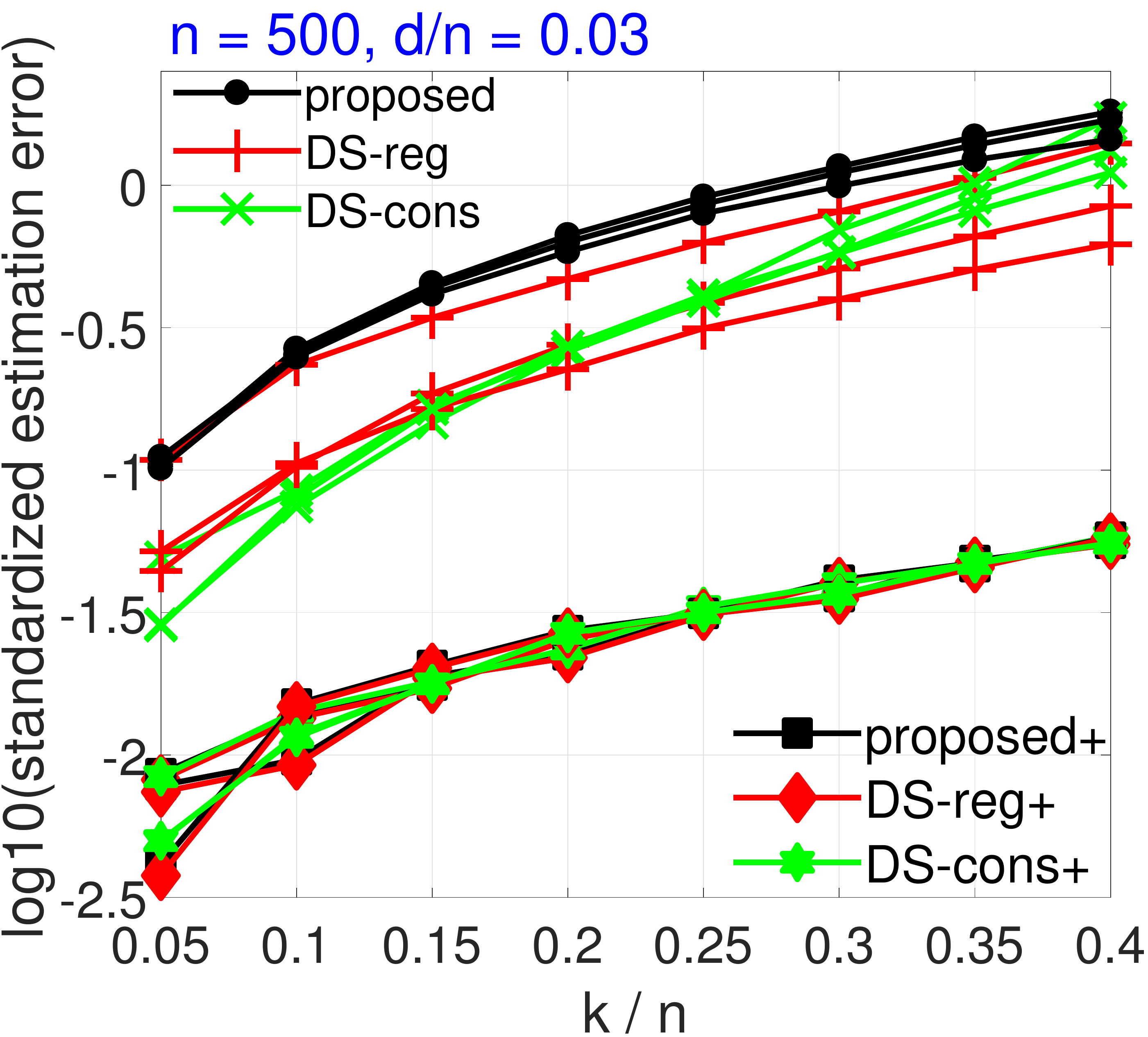}
\end{tabular}
\end{center}
\vspace*{-5ex}
\caption{Average standardized estimation errors~\eqref{eq:normalized_estimation} of {\bfseries \textsf{DS-reg}} and {\bfseries \textsf{DS-cons}} in comparison to {\bfseries \textsf{proposed}} along with their counterparts for refitting.}
\label{fig:DS}
\end{figure}

\vskip2ex
\noindent \emph{Results (II): Recovery of $\Theta^*$.}  We evaluate the normalized Hamming distance $\frac{1}{n} \su \mathbb{I}(\theta^*(i) \neq \wh{\theta}(i))$,
  where the matrix counterpart of $\wh{\theta}$ is given by $\wh{\Theta}$, i.e., the plug-in estimator~\eqref{eq:Theta_oracle_p} (modified to incorporate the constraint
  that $\Theta^*$ is a permutation) with $B^*$ replaced by $\wh{B}$ from~\eqref{eq:grouplasso}. In light of Theorem~\ref{theo:correspondence} and Lemmas~\ref{lem:gamma_min} \&~\ref{lem:gamma_min_const}, recovery of $\Theta^*$ is successful if $\gamma^2 \cdot \textsf{SNR} \asymp
  n^{-c/\text{srank}(B^*)} \cdot \textsf{SNR} $ is large enough. We therefore plot the normalized Hamming distance in dependency of
  the (log)``normalized'' SNR $-c/\text{srank}(B^*) \log(n) -2\log(\sigma)$, where the choice $c =0.7$ was found to ensure a reasonable alignment of the results
  across different experimental configurations. Figure~\ref{fig:permutationrecovery} indicates that recovery of $\Theta^*$ follows a phase transition: if the normalized SNR drops below a certain threshold, the normalized Hamming distance rises sharply. This observation is in alignment with the inachievability results in Theorem~\ref{theo:correspondence}. Interestingly, plug-in estimation (lower panel) does not lead to a significant degradation in performance compared to the situation in which $B^*$ is known (upper panel) even if the fraction of mismatches is noticeable ($k/n = 0.4$).

\begin{figure}[h!]
\begin{center}
\begin{tabular}{ccc}
 \hspace*{-1.88ex}  \includegraphics[width = 0.32\textwidth]{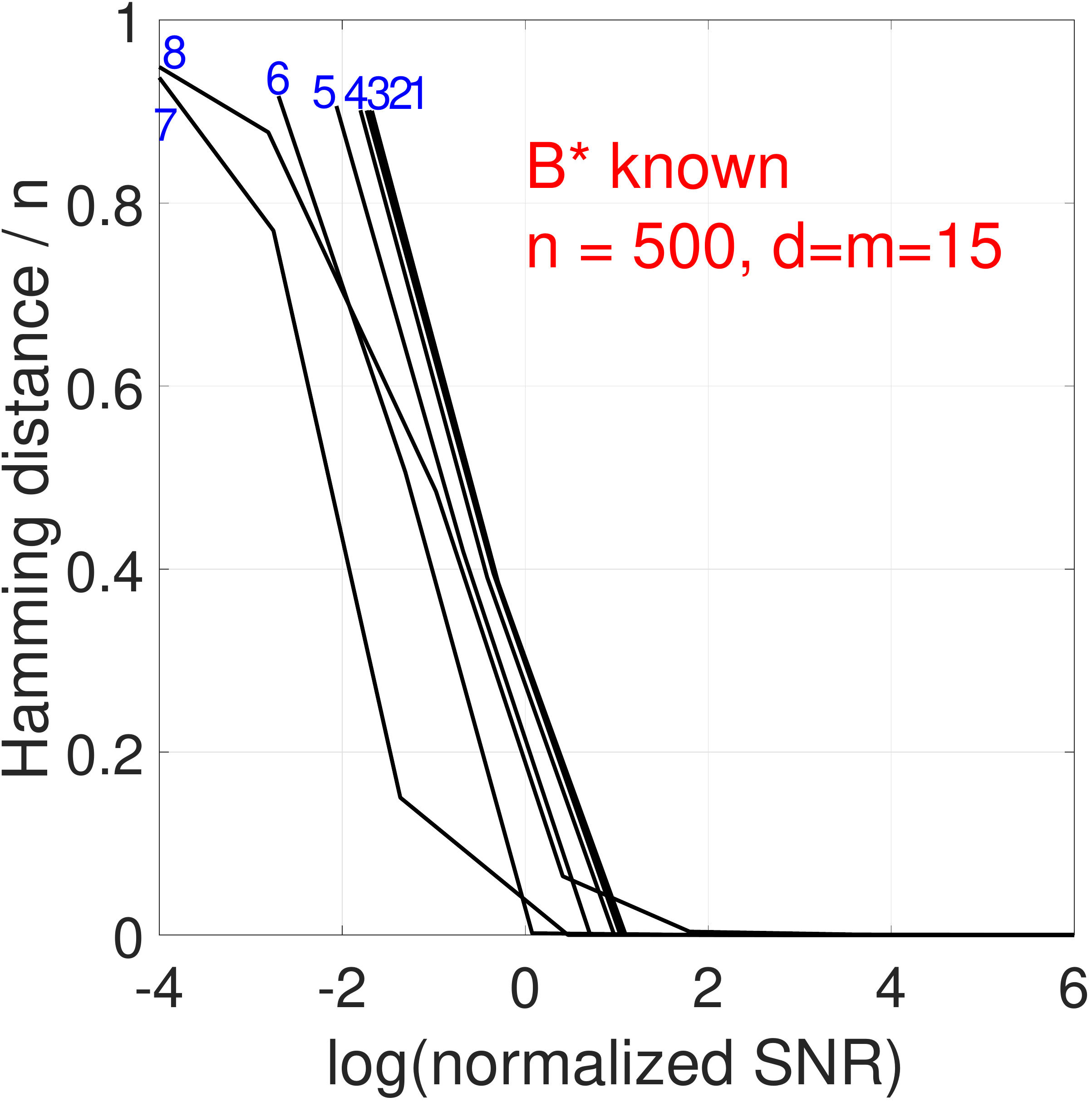}
  & \hspace*{-1.33ex} \includegraphics[width = 0.32\textwidth]{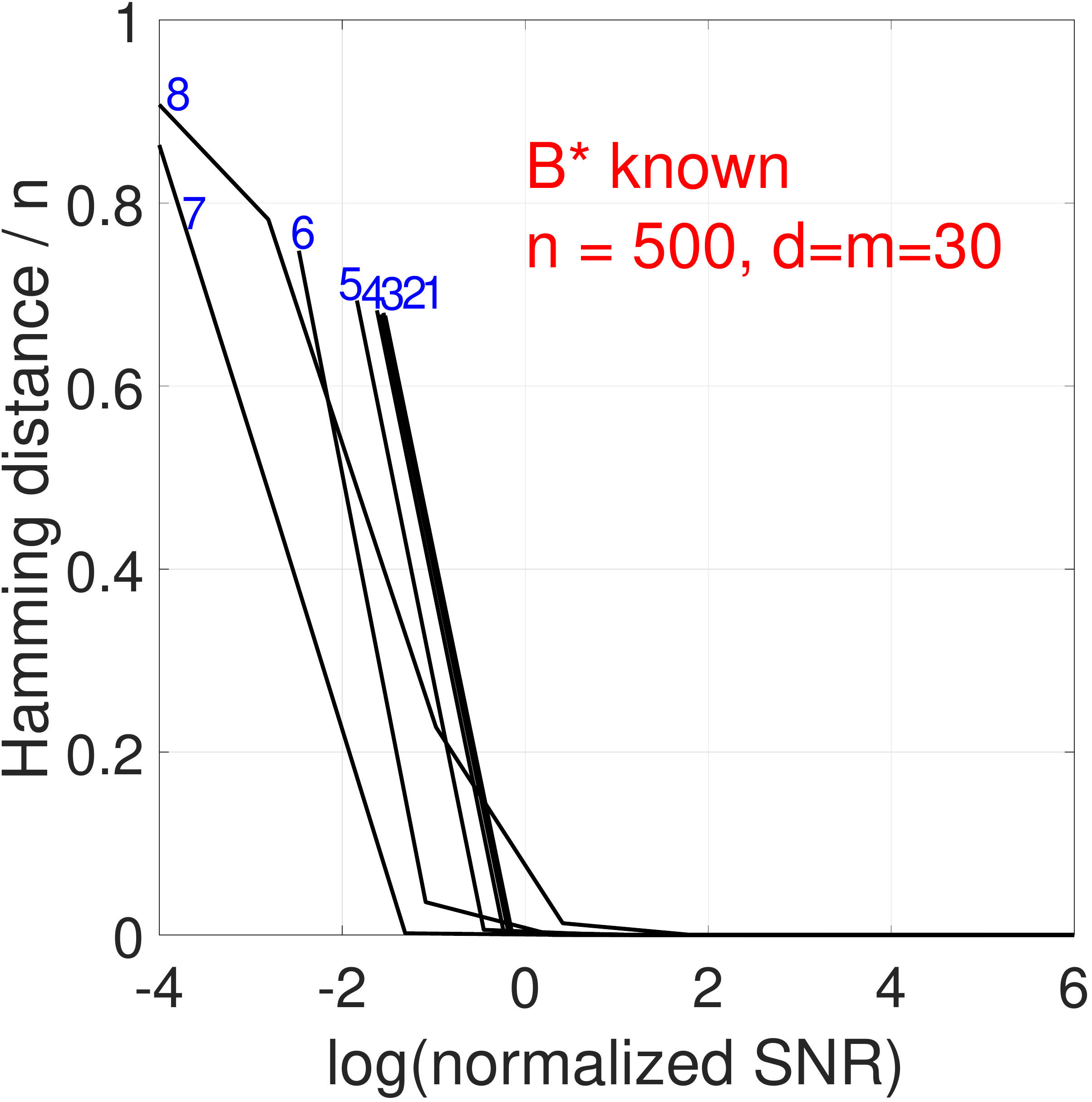} 
  & \hspace*{-1.33ex} \includegraphics[width = 0.32\textwidth]{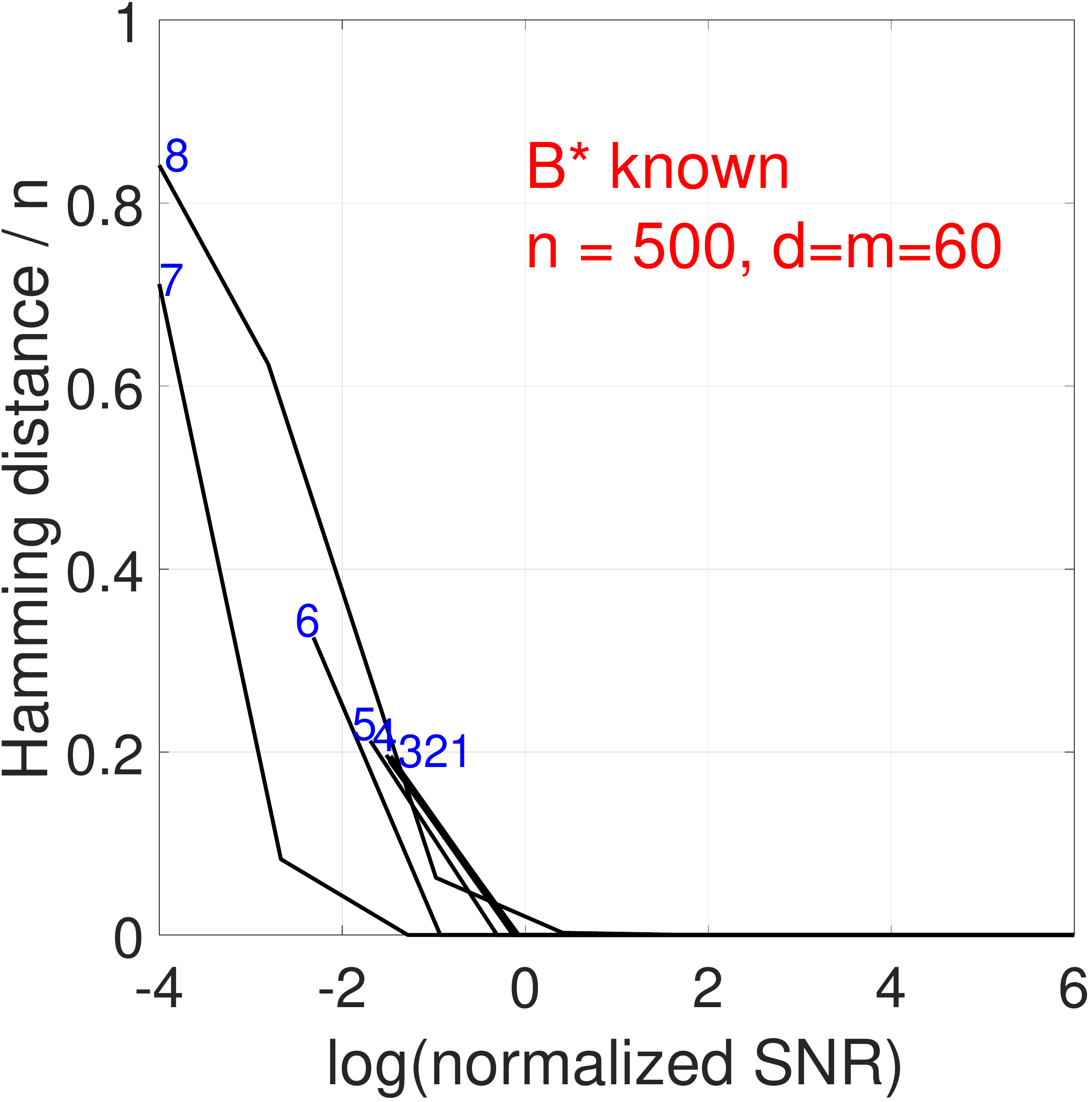}  \\[-.7ex]
  \hspace*{-1.85ex}  \includegraphics[width = 0.32\textwidth]{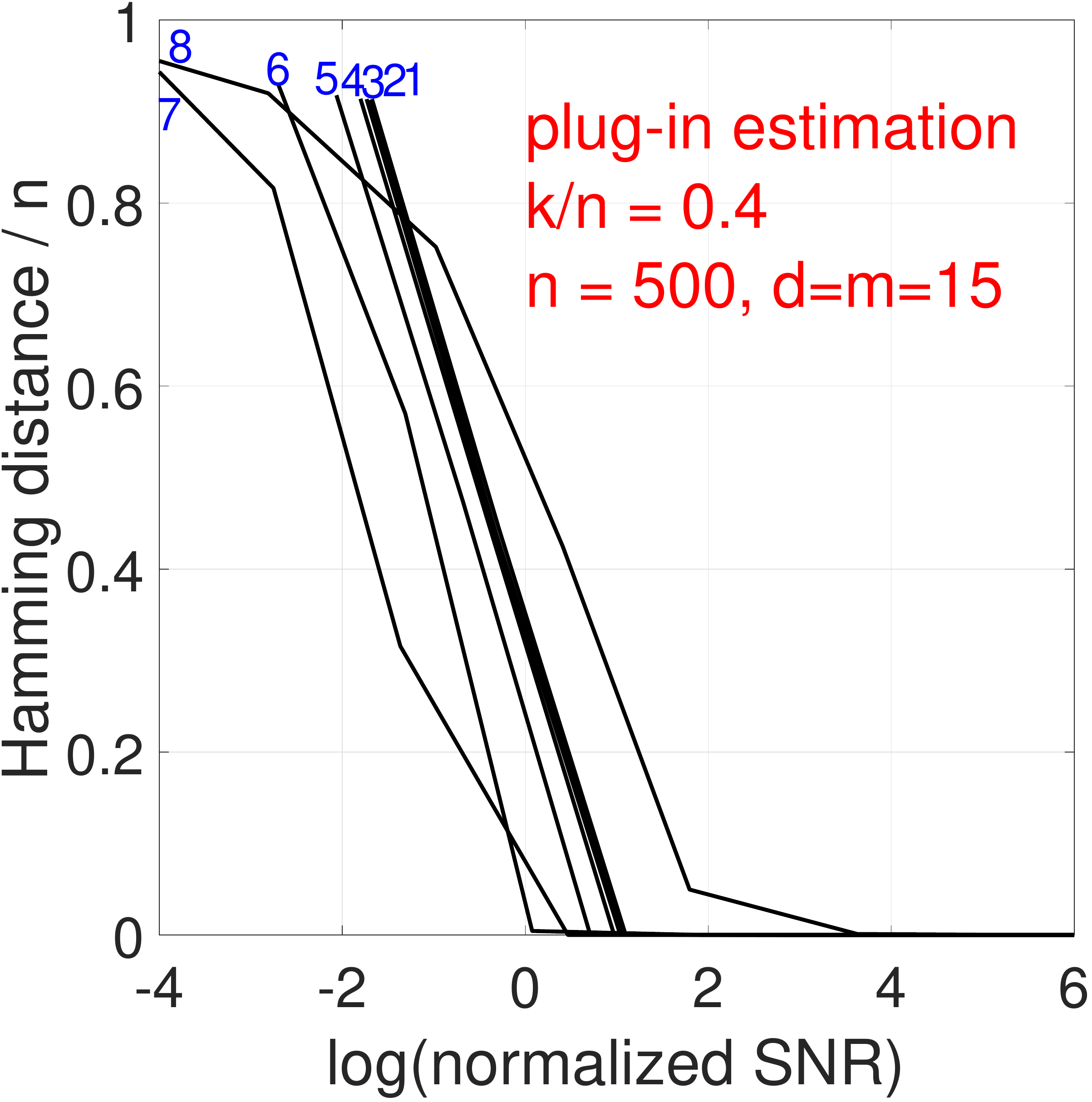}
  & \hspace*{-1.33ex} \includegraphics[width = 0.32\textwidth]{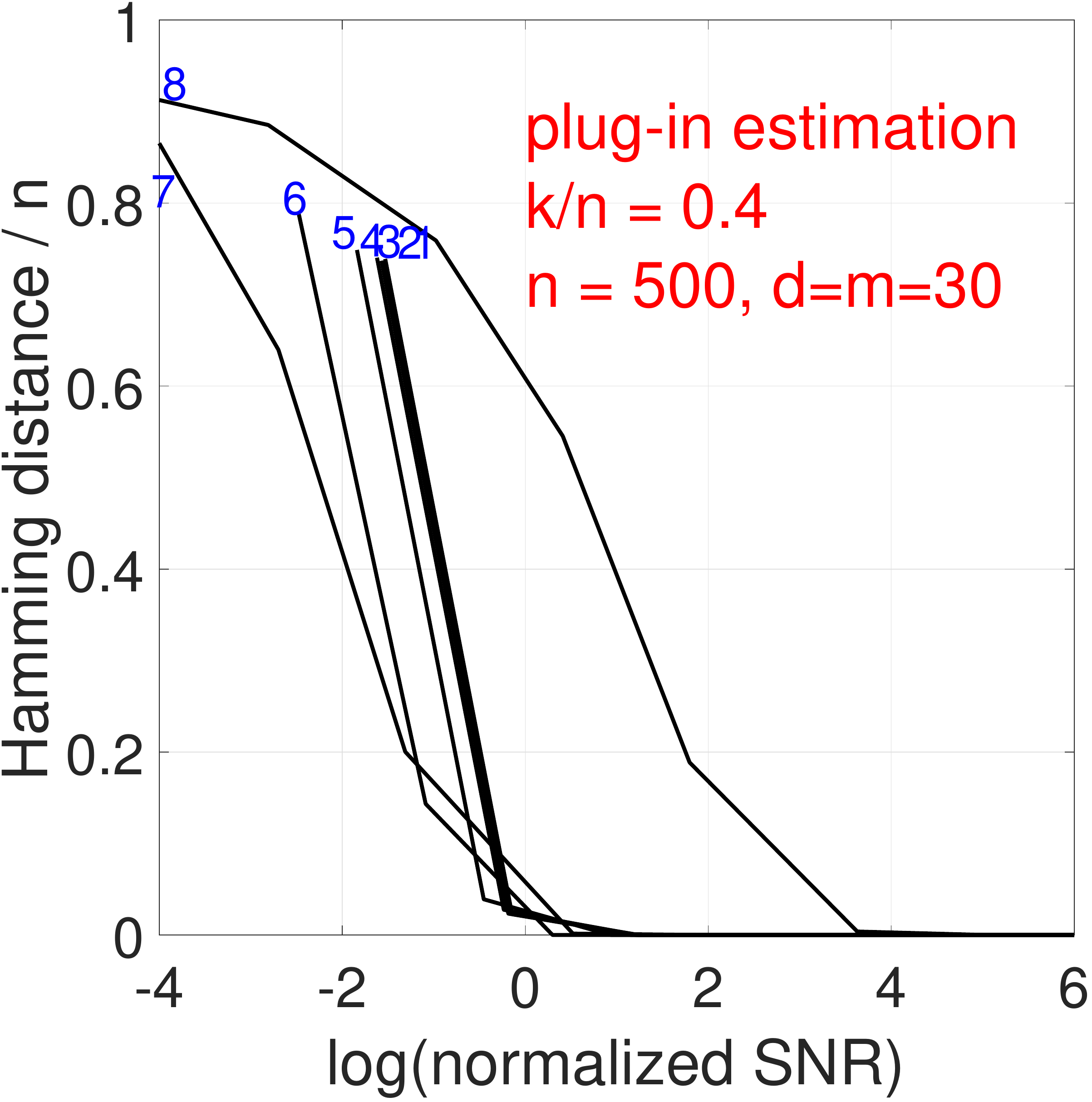} 
  & \hspace*{-1.33ex} \includegraphics[width = 0.32\textwidth]{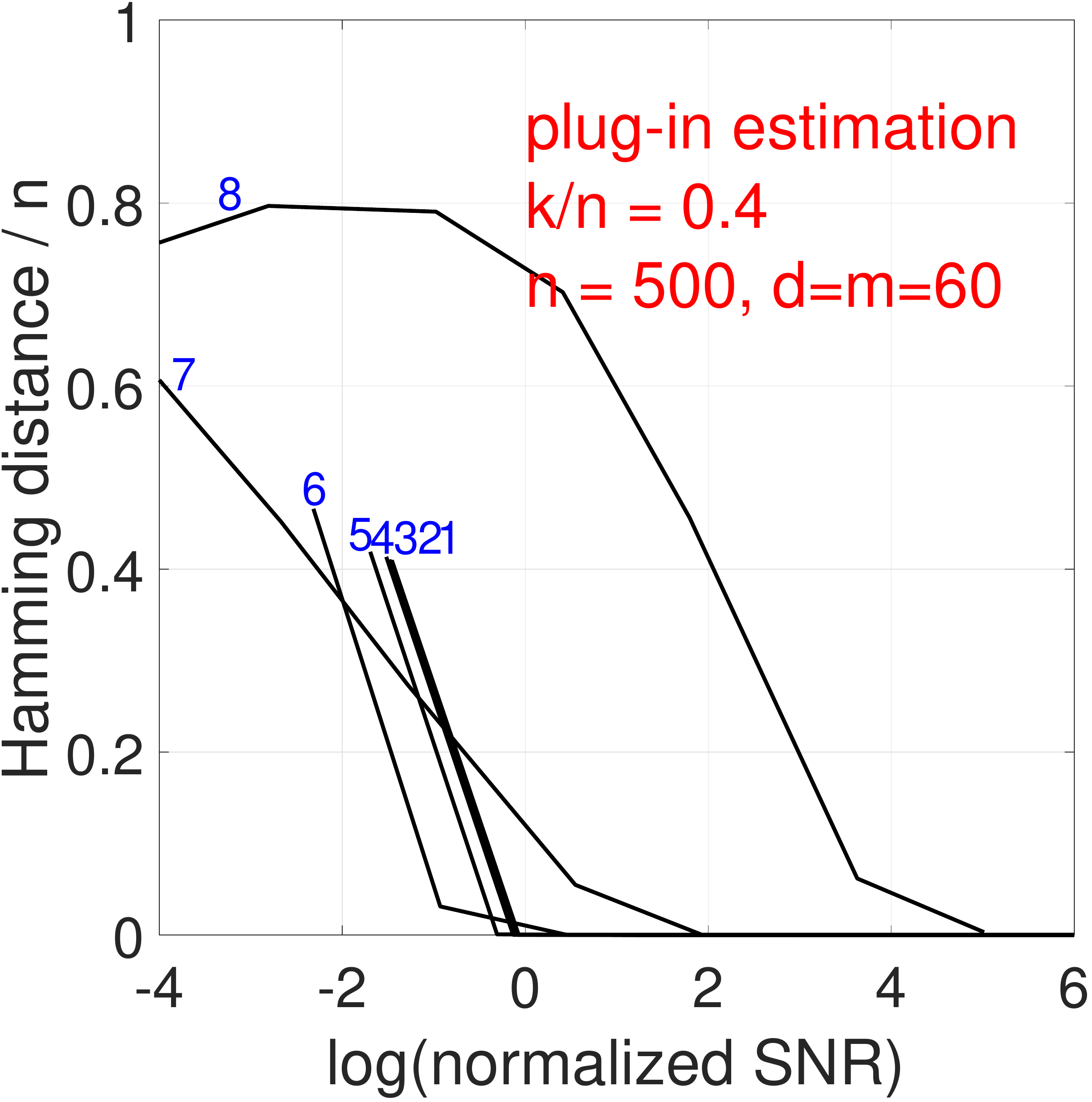} 
\end{tabular}
\end{center}
\caption{Average Hamming distance (scaled by $1/n$) between $\wh{\Theta}(B^*)$ and $\Theta^*$ (top row) and
  between $\wh{\Theta}(\wh{B})$ and $\Theta^*$ (bottom row) vs.~the (log) normalized SNR $=-c/\text{srank}(B^*) \log(n) -2\log(\sigma)$. The numbering indicates different values of the parameter $q$ controlling $\text{srank}(B^*)$, with higher numbers for larger $q$ (smaller $\text{srank}(B^*)$). The better the curves align, the more accurate the predicted dependence on the normalized SNR.}\label{fig:permutationrecovery}
\end{figure}
\renewcommand{\thefootnote}{\fnsymbol{footnote}}

\begin{table}[h!]
\caption{Overview on the data sets considered in this paragraph. $R^2$ here refers to the coefficient of determination in the absence of shuffling.\vspace{-0.2in}}\label{tab:datasets}
\begin{center}
{\small   \begin{tabular}{llllll}
\hline
   \textbf{Full Name}            & \textbf{Short Name}     & $n$ & $d$ & $m$ & $R^2$\\  \hline
  SARCOS robot arm~\citep{GPML}     &\textsf{sarcos} & 44,484    & 10    & 6 & 0.76     \\
  Flight Ticket Prices~\citep{mulan}  &\textsf{ftp}    & 335    & 30    &  6      & 0.89 \\
  Supply Chain Management~\citep{mulan}&\textsf{scm}    & 8,966   &  35   & 16     & 0.58\\
 \hline
\end{tabular}}
\end{center}
\end{table}

\noindent {\bfseries Real data.} We consider three benchmark data sets for multivariate regression as tabulated in Table~\ref{tab:datasets}. The data sets  are preprocessed versions of their original counterparts. The columns of the matrices
$X$ and $Y$ were centered, and $X$ was subsequently reduced to an adequate number of principal components
since due to (almost) linearly independent predictors the oracle least squares estimator (here assigned the role of $B^*$) would (essentially) not be defined. For \textsf{sarcos}, one of the response variables was removed to improve goodness
of fit, and hence to observe a better contrast in performance with an increasing fraction of mismatches. Likewise, two outliers with Cook's distance $>0.7$ were removed from \textsf{ftp}. We randomly permute varying fractions
(between $0.05$ and $0.4$) of the rows of $Y$, and investigate to what extent the proposed approach is
able to restore the goodness-of-fit (in terms of the coefficient of determination $R^2$\footnote{Here and in the sequel, the reported $R^2$ refers to the $R^2$ on the original data (i.e., before shuffling) given an estimator $B^{\text{est}}$ obtained from the shuffled data (cf.~caption of Figure~\ref{fig:realdata}).}) and the
regression coefficients of the least squares estimator in the complete absence of mismatches that here
takes the role of $B^*$. The performance of the proposed approach is compared to naive least squares based
on the permuted data. For each data set, we consider 20 independent random permutations for each value of
$k/n$. Performance with regard to permutation recovery is assessed via $\nnorm{(\wh{\Theta}(B^{\text{est}}) - \Theta^*) Y}_F/\nnorm{(I_n - \Theta^*) Y}_F$, i.e., via the relative reduction in error induced by random shuffling. This is a somewhat less stringent metric than the Hamming distance reported for synthetic data. The change in metric is motivated by the fact that exact permutation recovery cannot
be expected for the data sets under consideration given that separability in terms of~\eqref{eq:gamma} relative to the noise level is poor. Approach~\eqref{eq:grouplasso} is run with the choice $\lambda = M \cdot \frac{\wh{\sigma}_0}{\sqrt{n \cdot m}}$ for
$M \in \{0.25,0.5, 1, 2\}$ and $\wh{\sigma}_0$ denoting the root mean square error of the least squares estimator
in the absence of shuffling. We consider the same list of competitors and associated settings as for the synthetic data experiments, apart from the omission of {\bfseries \textsf{DS-reg}} and {\bfseries \textsf{DS-cons}} given the aforementioned scalability issues. 

\begin{figure}[h!]
\begin{center}
\begin{tabular}{ccc}
\hspace*{-1.88ex} \includegraphics[width = 0.32\textwidth]{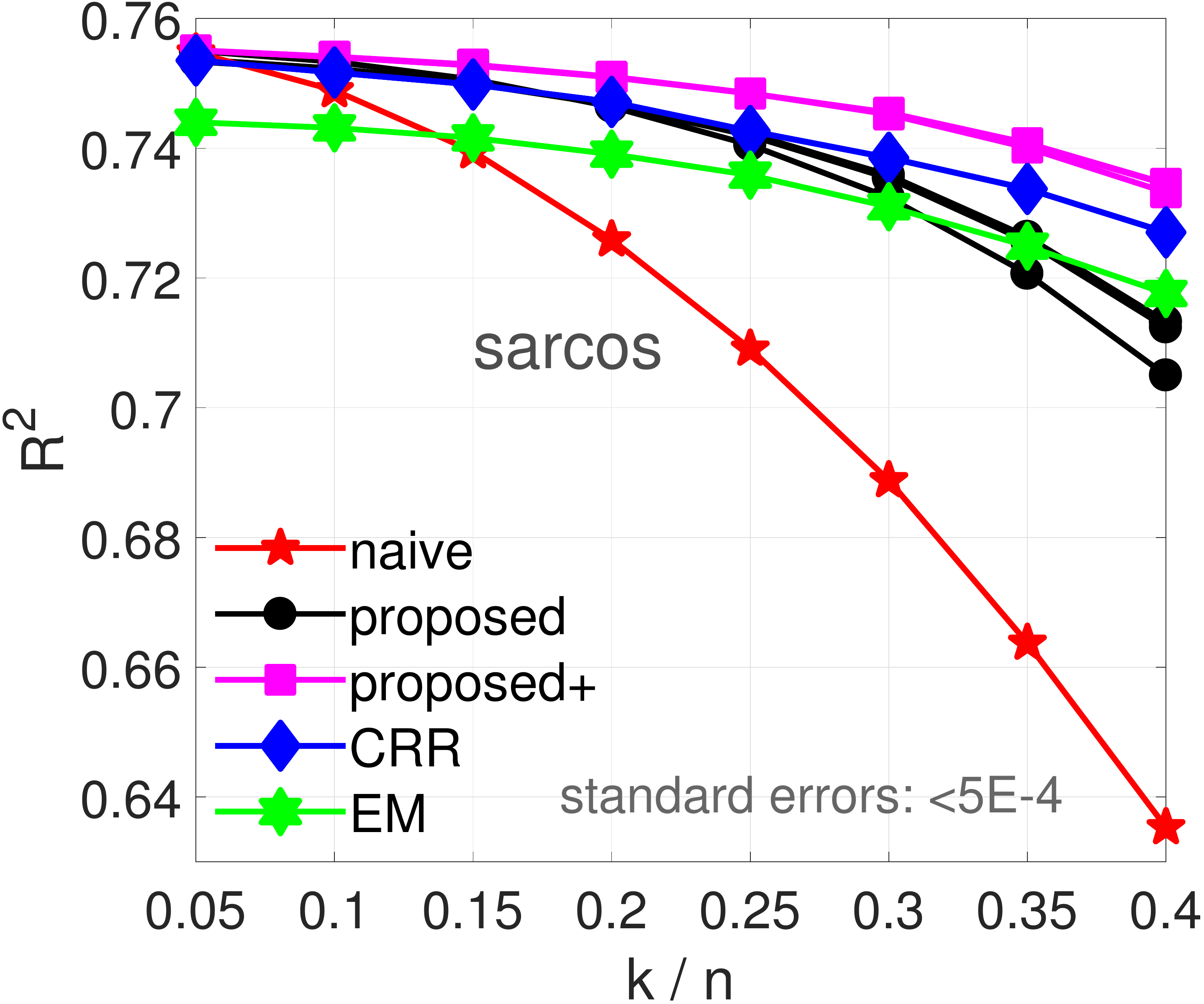}
  \hspace*{-1.33ex}   & \includegraphics[width = 0.32\textwidth]{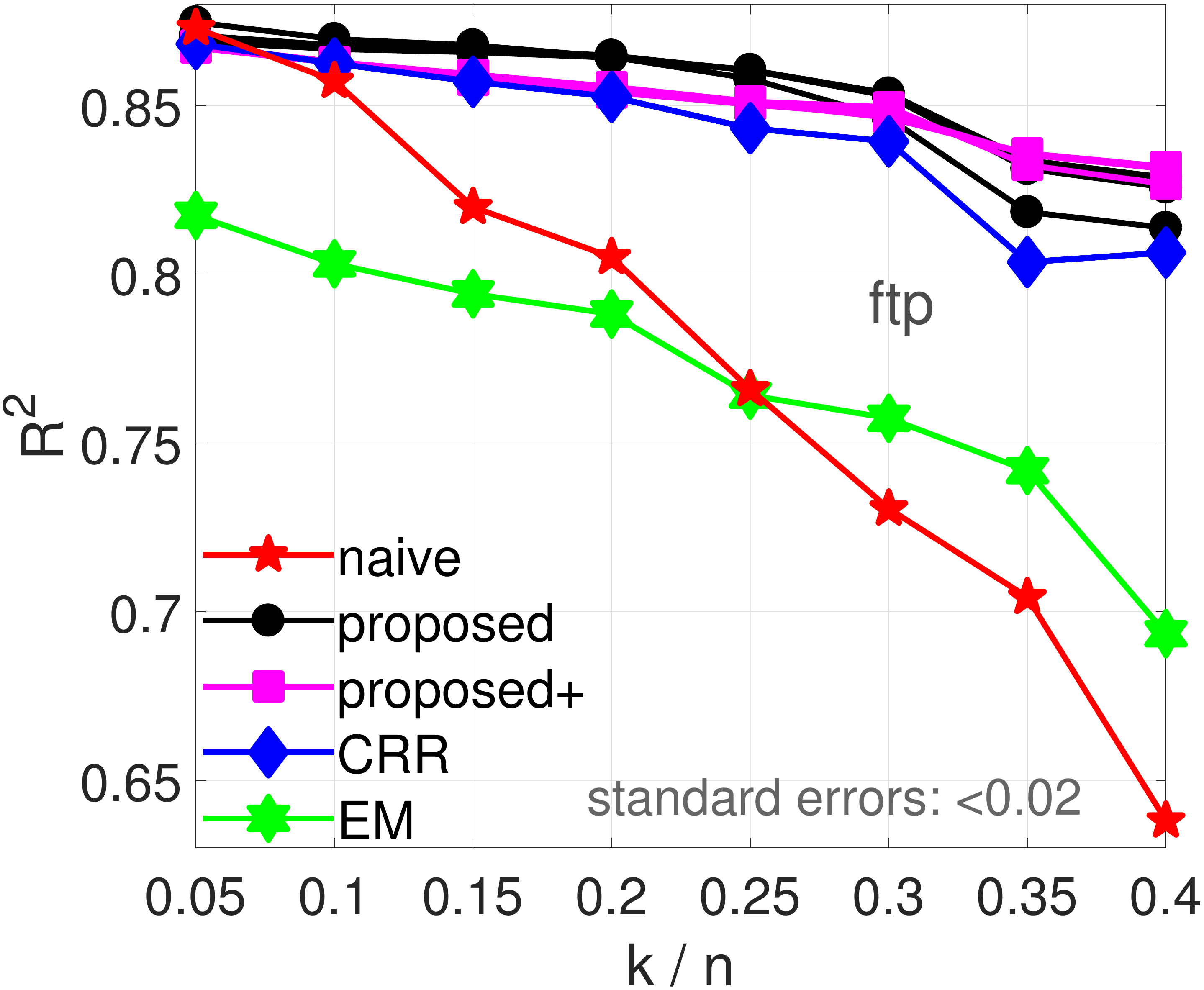} 
\hspace*{-1.33ex}   & \includegraphics[width = 0.32\textwidth]{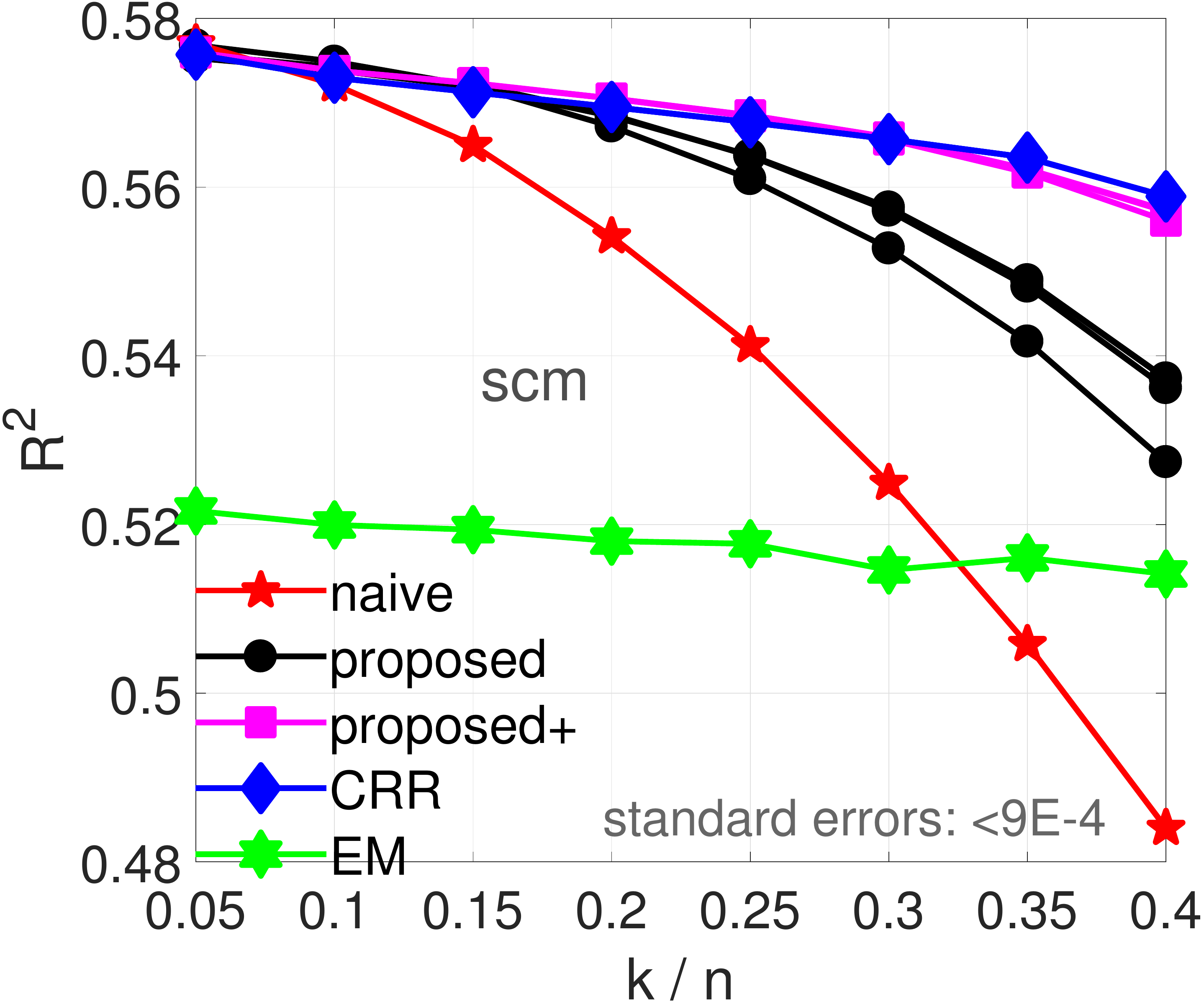}  \\[-.7ex]
\hspace*{-1.88ex}   \includegraphics[width = 0.32\textwidth]{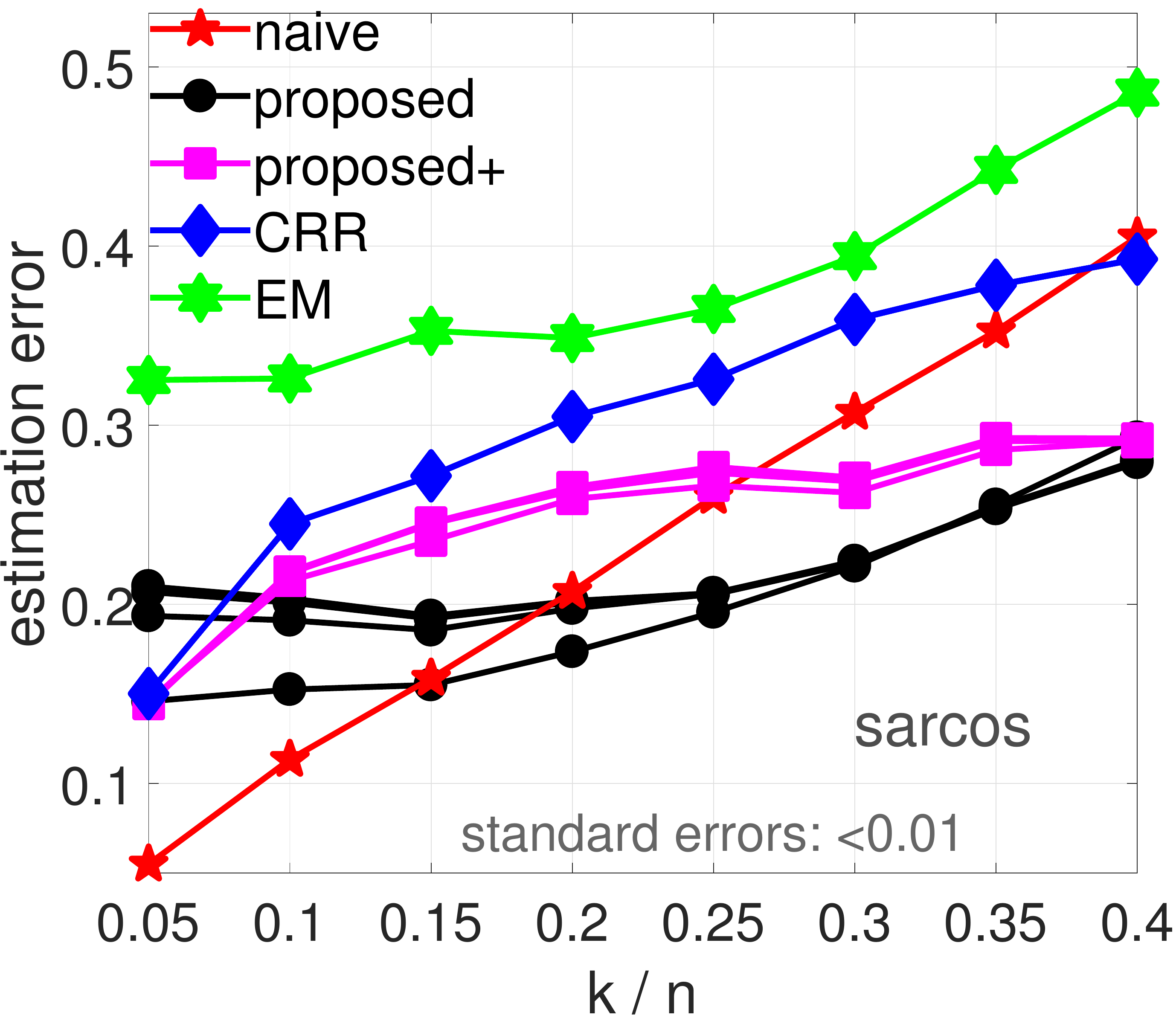}
\hspace*{-1.33ex}  & \includegraphics[width = 0.32\textwidth]{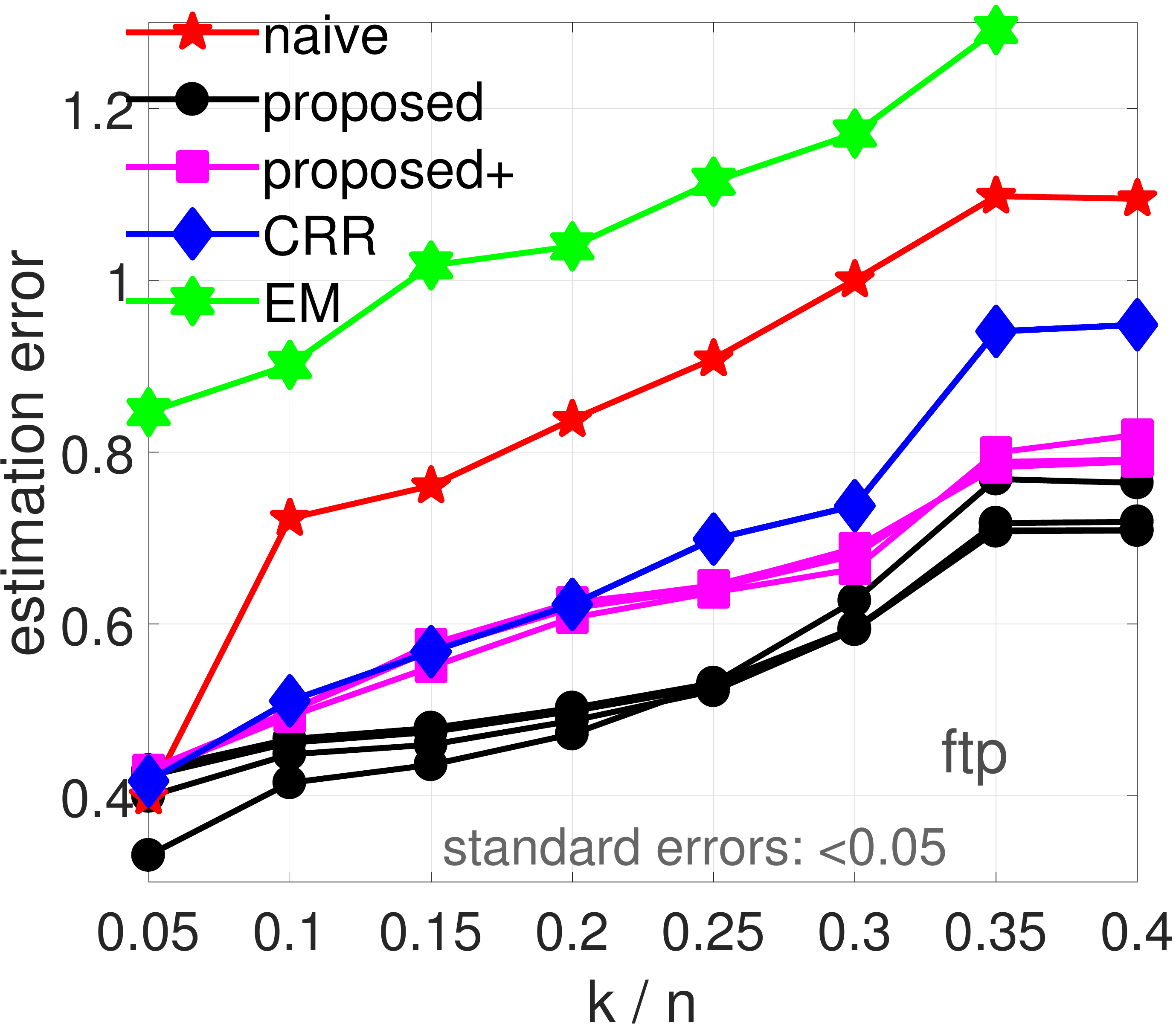} 
\hspace*{-1.33ex}  & \includegraphics[width = 0.32\textwidth]{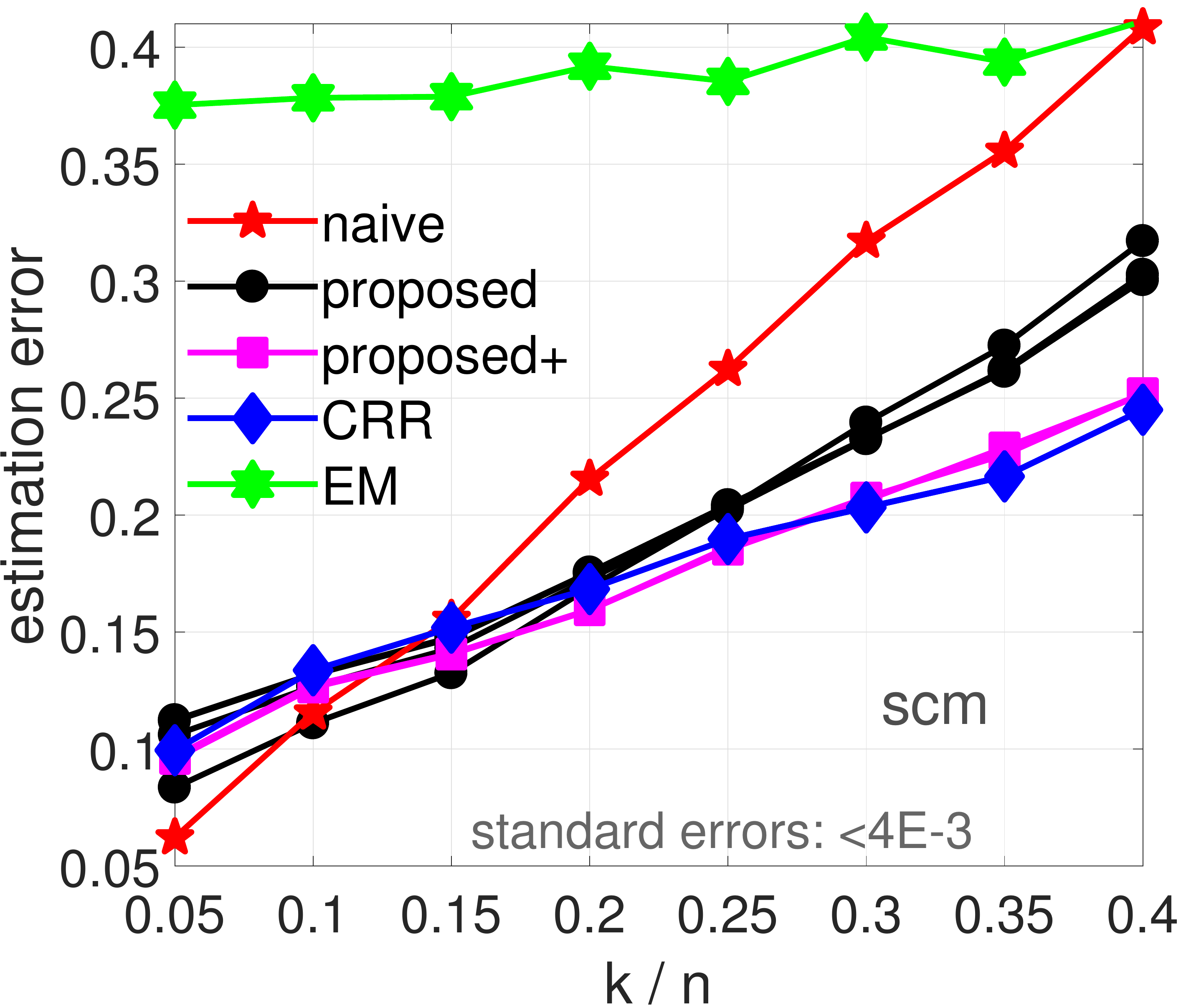} \\[-.7ex]
\hspace*{-1.88ex}    \includegraphics[width = 0.32\textwidth]{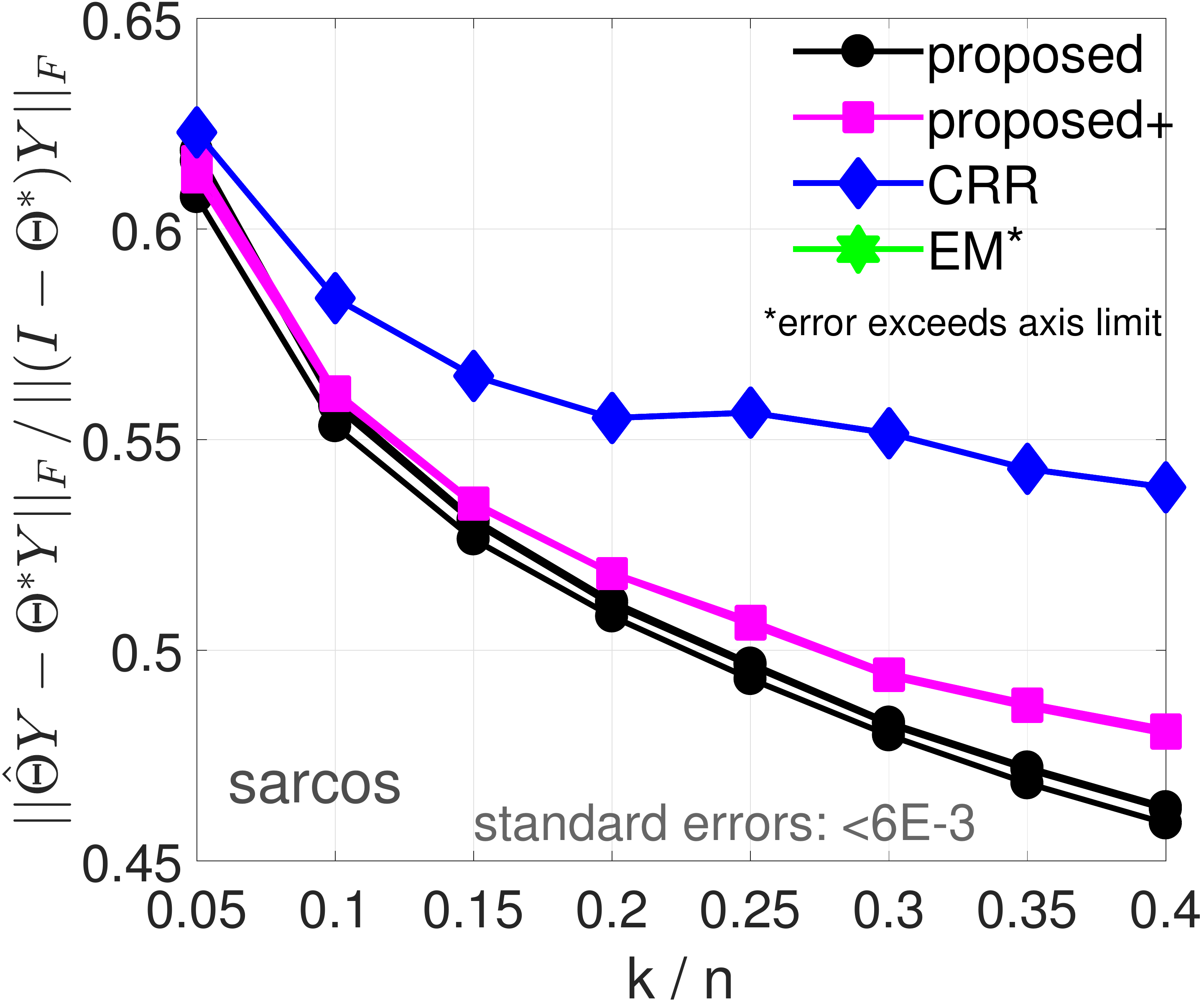}
\hspace*{-1.33ex}  & \includegraphics[width = 0.32\textwidth]{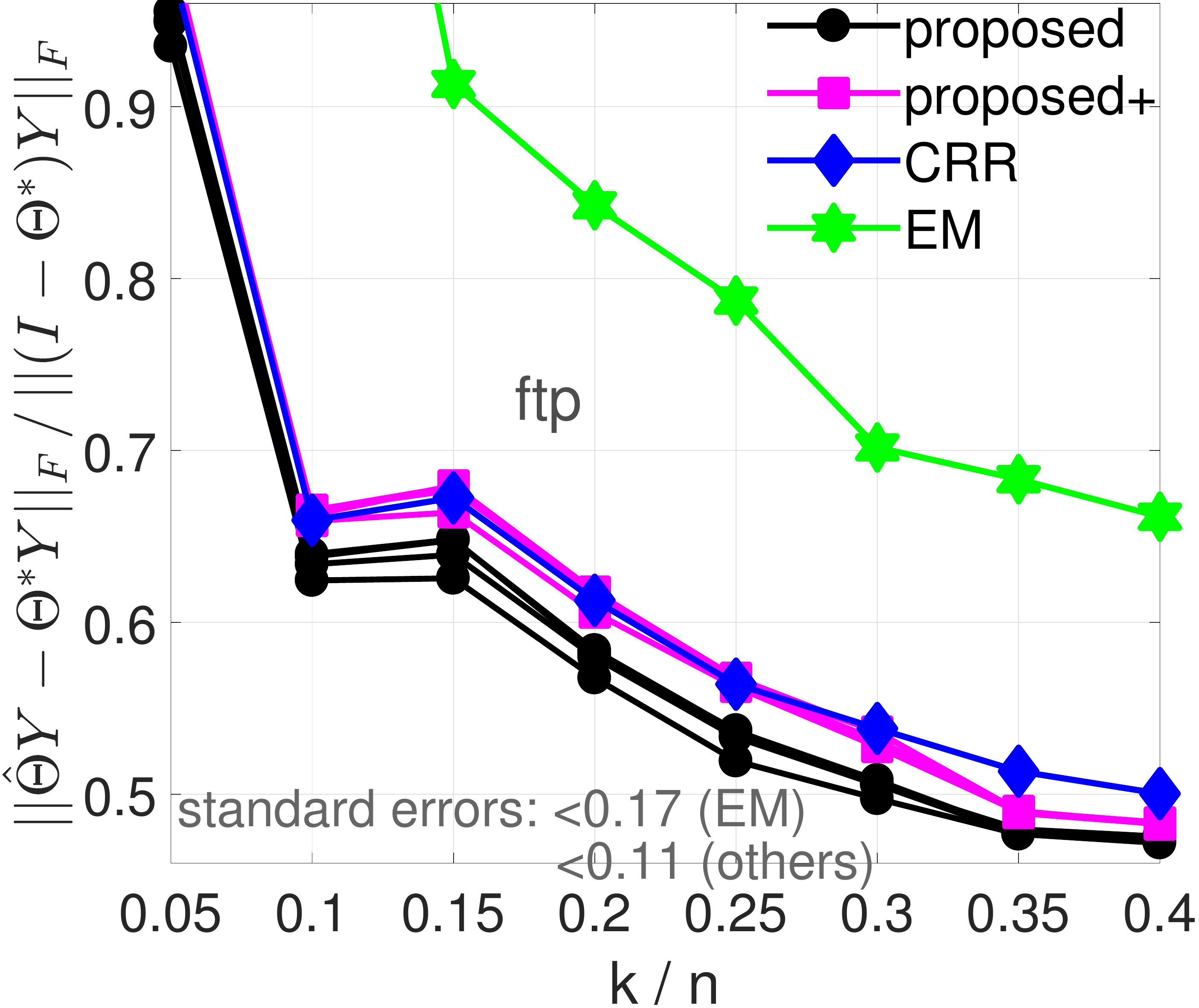} 
\hspace*{-1.33ex}  & \includegraphics[width = 0.32\textwidth]{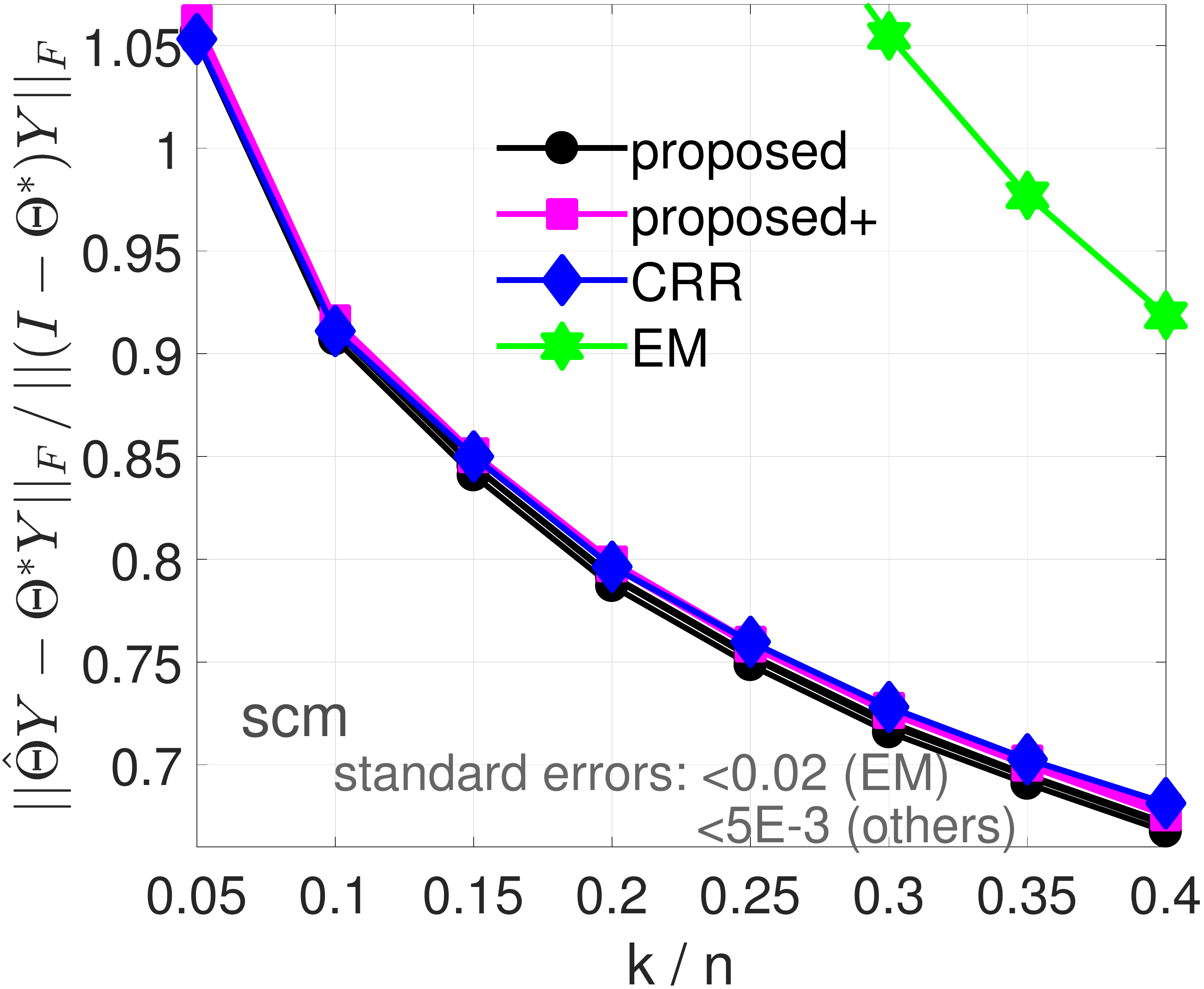} 
\end{tabular}
\end{center}
\vspace*{-4ex}
\caption{Top: Goodness of fit in terms of the coefficient of determination $R^2 = \nnorm{Y - XB^{\text{est}}}_F^2 / \nnorm{Y}_F^2$. Middle: Relative estimation errors $\nnorm{B^{\text{est}} - B^*}_F/\nnorm{B^*}_F$, where $B^*$ here refers to the oracle least squares estimator equipped with knowledge of $\Theta^*$. Bottom: Performance in approximate recovery of $\Theta^*$ evaluated in terms of $\nnorm{(\wh{\Theta}(B^{\text{est}}) - \Theta^*) Y}_F/\nnorm{(I_n - \Theta^*) Y}_F$. Each of the black lines corresponds to one specific
value of the multiplier $M$ in $\lambda = M \wh{\sigma}_0/\sqrt{n \cdot m}$.}\label{fig:realdata}
\end{figure}
As can be seen from Figure~\ref{fig:realdata}, the results are not sensitive to
the choice of the multiplier $M$. The proposed approach consistently improves over naive least squares once the
fraction of mismatches exceeds $0.2$, and yields more pronounced improvements as that fraction increases. Two-stage estimation of
$\Theta^*$ yields noticeable reductions of the error $\nnorm{(I_n - \Theta^*) Y}_F$ induced by shuffling. Approaches {\bfseries \textsf{proposed+}} and {\bfseries \textsf{CRR}} (equipped with knowledge of $k$), yield only occasional and rather minor improvements over {\bfseries \textsf{proposed}}. Interestingly, {\bfseries \textsf{EM}} exhibits poor performance even for moderate $n$ (data set \textsf{ftp}), often falling short of {\bfseries \textsf{naive}} in sharp contrast to the results observed for the synthetic data. This raises the question whether competitive performance of {\bfseries \textsf{EM}} is tied to specific properties of Gaussian design.      
\vskip2ex
\noindent {\bfseries Case study.} We here illustrate the use of the proposed approach and its competitors in data integration scenarios based on a setting designed to mimic the analysis of data obtained from multiple sensors in an asynchronous fashion. The specific example presented in the sequel is based on the Multi-Site Beijing Air Quality data set~\citep{BeijingAir} which contains measurements of various air pollutants and climate parameters recorded at an hourly rate from March 1st, 2013 to February 28th, 2017. For demonstration purposes, we confine ourselves to complete records from the site Nongzhanguan for the years 2016 and 2017 ($n = 9,726$). A linear regression model is fitted in which the response variables are given by the square roots of the air concentrations of the pollutants PM2.5, PM10, SO2, NO2, O3 ($m = 5$) and the predictor variables are given by temperature, dew point temperature, air pressure, precipitation, wind speed, CO concentration, and all associated quadratic terms plus intercept ($d = 28$). This model achieves an $R^2 \approx0.725$.

At the next stage, we suppose that the response and predictor variables are collected by two different sensors, with temperature and air pressure collected by both sensors. In order to recreate the situation of mismatch error in record linkage that commonly results from the use of inexact or erroneous identifiers~\citep{Christen2012}, the two sets of measurements are merged based on incomplete time stamps (day and hour are missing) and inaccurate temperature and air pressure measurements (rounded to integers). Requiring that linked records must agree on this
combination of four matching variables implies that the merged file is of the form $[\Theta^* X \;\; Y]$, where
$\Theta^*$ is a permutation matrix that can be arranged in block diagonal structure with the blocks corresponding to groups of measurements having the same combination of matching variables. It is assumed that the data analyst has no knowledge about the linkage process, in particular about the use of matching variables and the resulting block structure of $\Theta^*$; this setting is typically referred to as ``secondary analysis'' in the record linkage literature~\citep{Chambers2019improved}.

Only $1,379$ out of $n = 9,726$ observations yield singleton blocks, i.e., they are uniquely identifiable based on the matching variables, while all other observations belong to blocks of size two up to $20$. To fix $\Theta^* = \text{bdiag}(\Theta_{(1)}^*, \ldots, \Theta_{(K)}^*)$, we set $\Theta_{(l)}^* = \argmax_{\Theta} \nnorm{Y_{(l)} - \Theta Y_{(l)}}_F^2$\footnote{This optimization problem reduces to a linear assignment problem.} where $Y_{(l)}$ denotes the rows of $Y$ corresponding to the $l$-th block, $1 \leq l \leq K = 3,625$, and the $\argmax$ is over all permutations associated with the respective block. While the resulting nominal fraction of mismatches $|\{i: \, \Theta_{ii}^* \neq 1 \} |/n \approx0.63$ does not appear to fit the sparse regime, the majority of mismatches do not introduce substantial contamination in the sense that $\nnorm{Y_{i,:} - Y_{\theta^*(i), :}}_F$ is within the noise level; to a good extent, this can be attributed to the fact that the responses tend to be more similar within blocks than across blocks.

The same regression model as above is fitted based on the merged records $[\Theta^*X \;\; Y]$. Naive least squares regression leads to a noticeable drop of the $R^2 \approx0.66$ and a root mean squared error (RMSE) of $431.4$ relative to the original (i.e., based on $[X \; Y]$) regression parameter estimate $B^*$. Application of the approach~\eqref{eq:grouplasso} with the choice $\lambda = \frac{\wh{\sigma}}{\sqrt{n \cdot m}}$, where $\wh{\sigma}$ can be taken as the root mean squared prediction error of either the original or the naive least squares fit, lifts the $R^2$ to $0.70$ and reduces the RMSE for the regression parameter to $318.1$. Following the proposed two-stage method, we use the resulting estimator $\wh{B}$ to
correct mismatches by solving the following optimization problem:

\begin{align}
  \min_{\Pi \in \mc{P}} \nnorm{Y - \Pi (\Theta^* X) \wh{B}}_F^2 \quad\;\; \text{subject to \;} &\Pi_{ii} = 1 \; \text{if} \; \, \nnorm{Y_{i,:} - \Theta^*_{i,:} X \wh{B}}_F \leq \sqrt{2m}\wh{\sigma} \label{eq:rematch_casestudy}\\
                                                                                               &\Pi_{ij} = 0 \; \text{if} \; \, \nnorm{Y_{i,:} - \Theta^*_{i,:} X \wh{B}}_F \leq \nnorm{Y_{i,:} - \Theta^*_{j,:} X \wh{B}}_F, \notag 
\end{align}
for $1 \leq i,j \leq n$, where $\mc{P}$ denotes the set of all permutation matrices~\eqref{eq:permutations}. Note that perfect recovery corresponds to $\Pi = (\Theta^*)^{-1}$. The additional constraints are imposed as a means to achieve
sparsity of $\Pi$ in the sense of small Hamming distance to the identity: the first constraint sets diagonal elements to one for which the discrepancy between observed and fitted values is within a factor of $\sqrt{2}$ of the noise level, and the second constraint excludes pairings that do not lead to improvements in terms of fit.

Given the minimizer $\wh{\Pi}$ of~\eqref{eq:rematch_casestudy}, it is worth attempting a re-fit of the regression model based on data
$[\wh{\Pi} (\Theta^* X) \;\; Y]$. As shown in the top panel of Figure~\ref{fig:casestudy}, the solution $\wh{\Pi}$ is able to reduce mismatch error to an extent
that is comparable to the error of the original regression model. Moreover, the bottom panel of Figure~\ref{fig:casestudy} shows that the fitted values
of the re-fit agree considerably better with the fitted values based on $[X \; Y]$ relative to the fitted values of naive least squares (plot of the first principal component is meaningful here since here $\text{srank}(Y) \approx 1$). Accordingly,
the $R^2$ of the refit increases to $0.715$ close to the original $0.725$.

In addition, we consider the competitors {\bfseries \textsf{CRR}} and {\bfseries \textsf{EM}} as alternatives. {\bfseries \textsf{CRR}} achieves slightly better performance than~\eqref{eq:grouplasso} with
an oracular choice of its tuning parameter (sparsity level $k$); choosing the latter so as to minimize the $R^2$ at $0.717$ yields the choice
$k/n =0.19$ while an $R^2$ of $0.71$ or higher is achieved within the entire range $k/n \in [0.09, 0.32]$. The "effective" fraction of mismatches
is expected to be contained in that interval. By contrast, the performance of {\bfseries \textsf{EM}} is rather poor, with an additional drop of the $R^2$ compared
to naive least squares. At the same time, the $R^2$ achieved by {\bfseries \textsf{EM}} on the mismatched data is close to $0.8$ (i.e., much larger than $0.725$), which indicates substantial overfitting. A numerical summary of the performance of the approaches compared here can be found in Table~\ref{tab:casestudy}.      
 
\begin{table}
\caption{{\bfseries \textsf{oracle}}: least squares fit based on the original data $[X \; Y]$; {\bfseries \textsf{prop}}: short for {\bfseries \textsf{proposed}}; {\bfseries \textsf{prop}}-$\wh{\Pi}$, {\bfseries \textsf{CRR}}-$\wh{\Pi}$: least squares refit after solving~\eqref{eq:rematch_casestudy} with $\wh{B}$ obtained according to~\eqref{eq:grouplasso} and~\eqref{eq:Bhatia}, respectively. The second
  and third row contain the RMSE in estimating $B^*$ including intercepts ($a$) and not including intercepts
  ($b$). Note that the combination of both tends to provide a more accurate picture: {\bfseries \textsf{EM}} achieves a decent value for $(a)$ despite poor performance based on $R^2$ and confirmed by $(b)$.\vspace{-0.2in}}\label{tab:casestudy}
\begin{center}
 \begin{tabular}{|c|c|c|c|c|c|c|c|c|}
            \hline & & & & & & & &\\[-2ex]
                 & {\bfseries \textsf{oracle}} & {\bfseries \textsf{naive}} & {\bfseries \textsf{prop}} & {\bfseries \textsf{prop+}} & {\bfseries \textsf{CRR}} & {\bfseries \textsf{EM}} & {\bfseries \textsf{prop}}-$\wh{\Pi}$ & {\bfseries \textsf{CRR}}-$\wh{\Pi}$ \\ \hline
  $R^2$          &0.725         & 0.66   &0.70        & 0.712        &.717     & 0.625 &0.715    &0.715     \\ \hline
  $B^*$-RMSE$^a$ &  0        &  431.4     & 318.1         & 295.81    & 259.1   &  280.6  &  298.9    & 304.8 \\ \hline
            $B^*$-RMSE$^b$ & 0        &  4.11     &  3.94      & 3.98  & 3.42    &  5.97  &  3.67    & 3.58  \\
            \hline
  \end{tabular}
\end{center}
\end{table}

\begin{figure}
\begin{center}
  \begin{tabular}{|c|c|c|c|c|c|c|}
    \hline & & & & & \\[-2ex]
    RMSE  & $(Y, XB^*)$ & $(Y, \Theta^{*-1} Y)$ &  $(\wh{\Pi} Y, \Theta^{*-1} Y)$ & $(\wh{\Pi} Y, \Theta^{*-1} Y)$ & $(\wh{\Pi} Y, \Theta^{*-1} Y)$ \\
           &             &                      &  {\bfseries \textsf{proposed}}             &  {\bfseries \textsf{CRR}}            &    {\bfseries \textsf{EM}}
    \\ \hline 
           & 1.8         & 2.53                 &     1.89                           & 1.86                          &  2.13\\
    \hline
  \end{tabular}
  \vskip2.2ex
  before correction~\eqref{eq:rematch_casestudy} $\qquad \qquad$ after correction~\eqref{eq:rematch_casestudy}  \\
  \includegraphics[width = 0.36\textwidth]{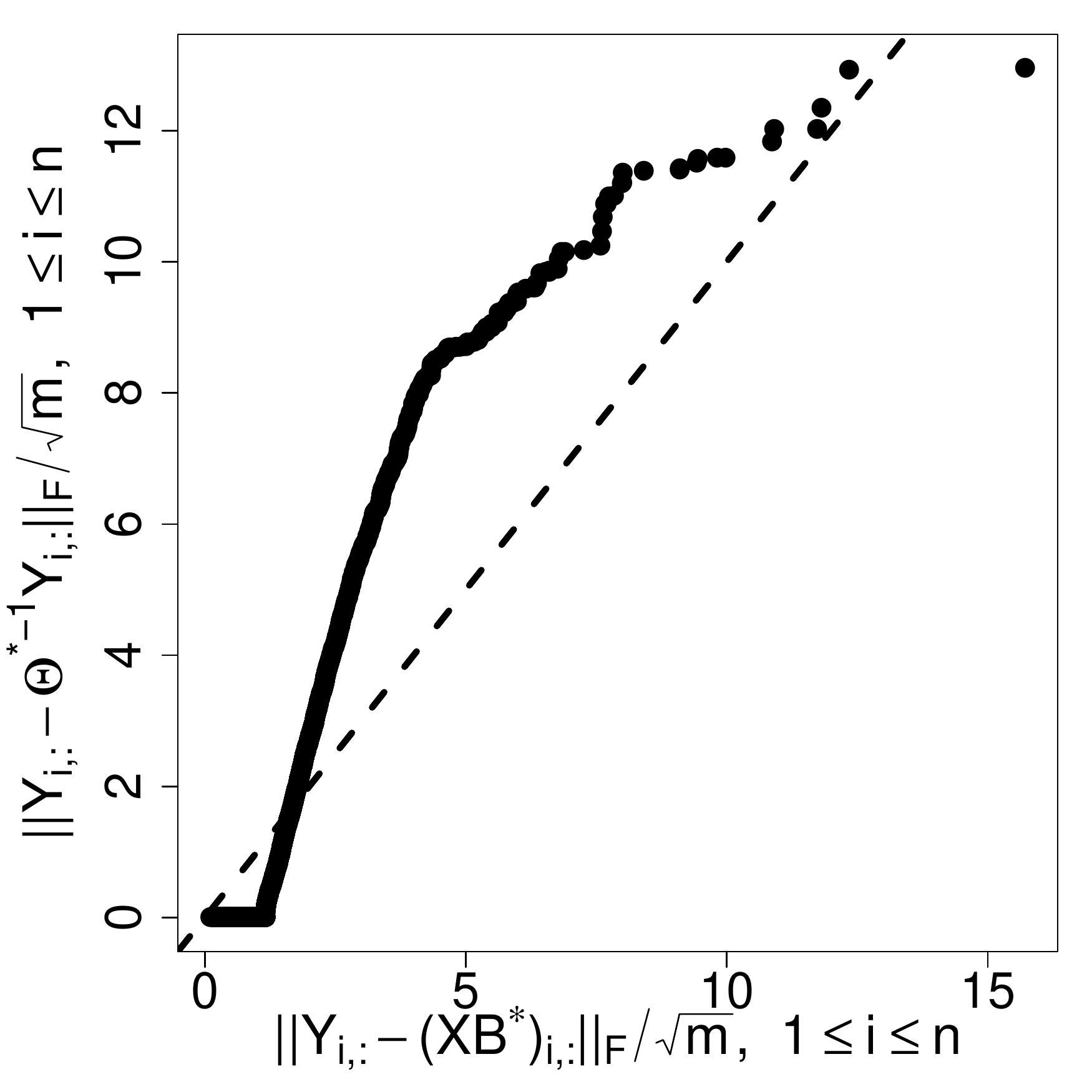}
  \includegraphics[width = 0.36\textwidth]{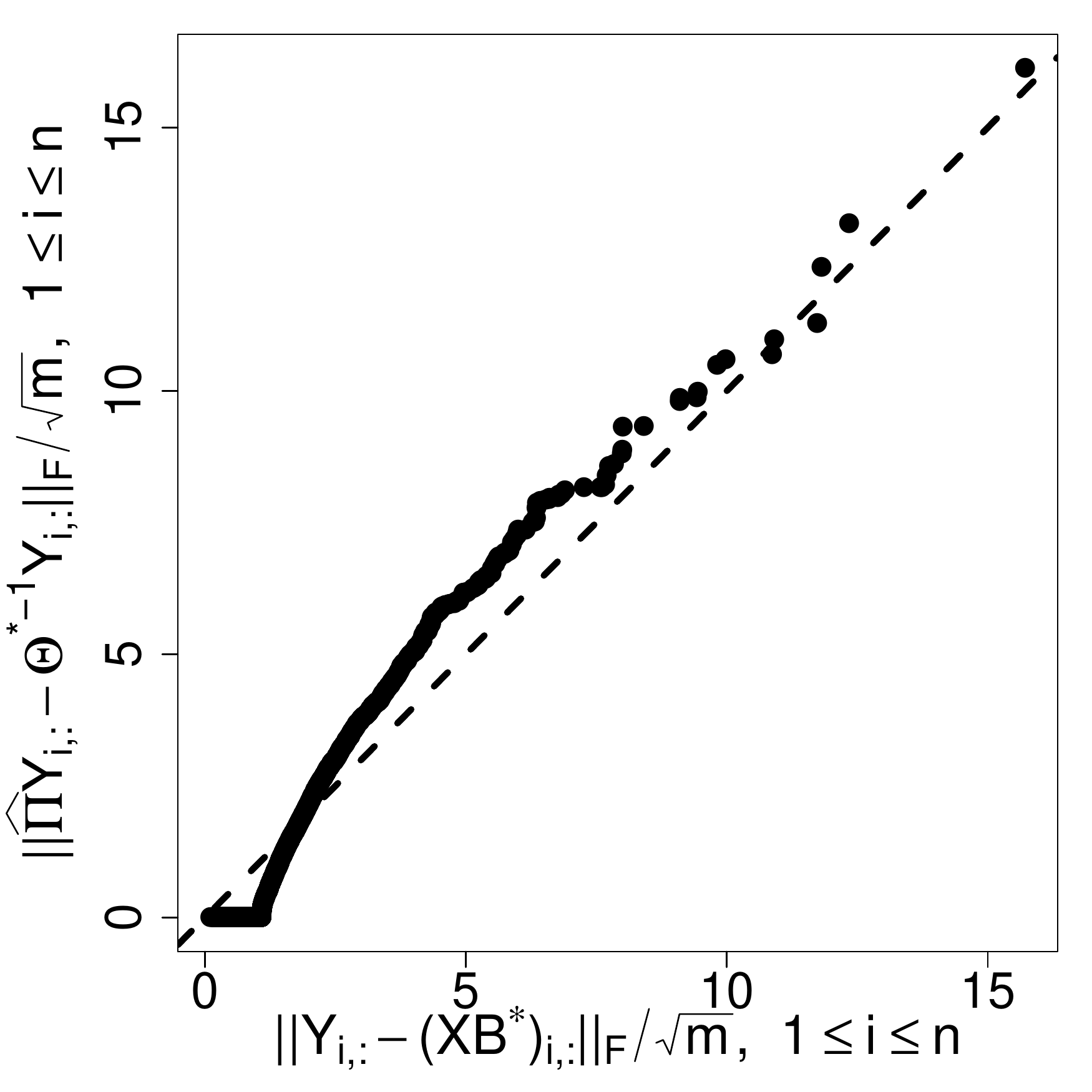} \\[1ex]
  before correction~\eqref{eq:rematch_casestudy} $\qquad \qquad$ after correction~\eqref{eq:rematch_casestudy}  \\
  \includegraphics[width = 0.36\textwidth]{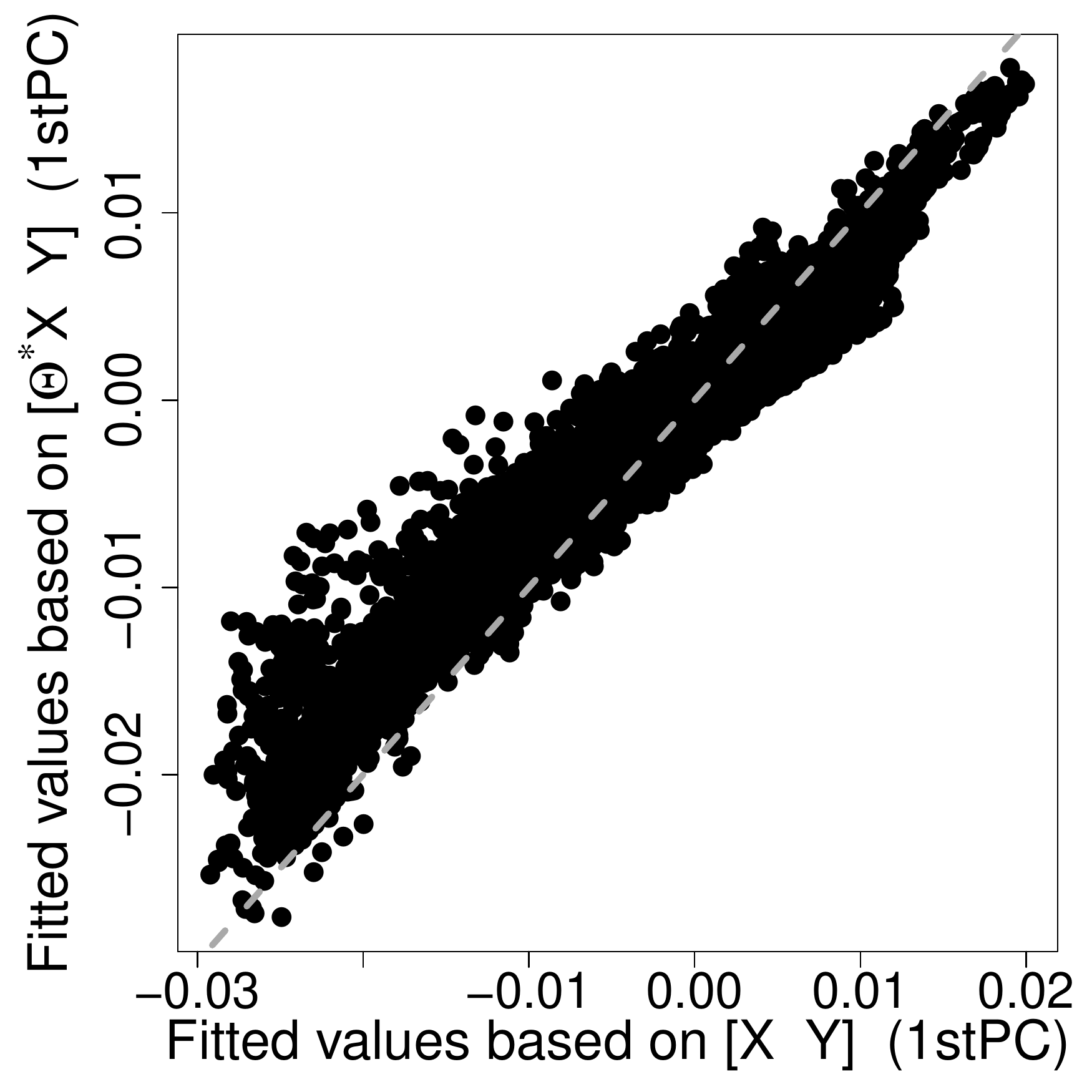}
  \includegraphics[width = 0.36\textwidth]{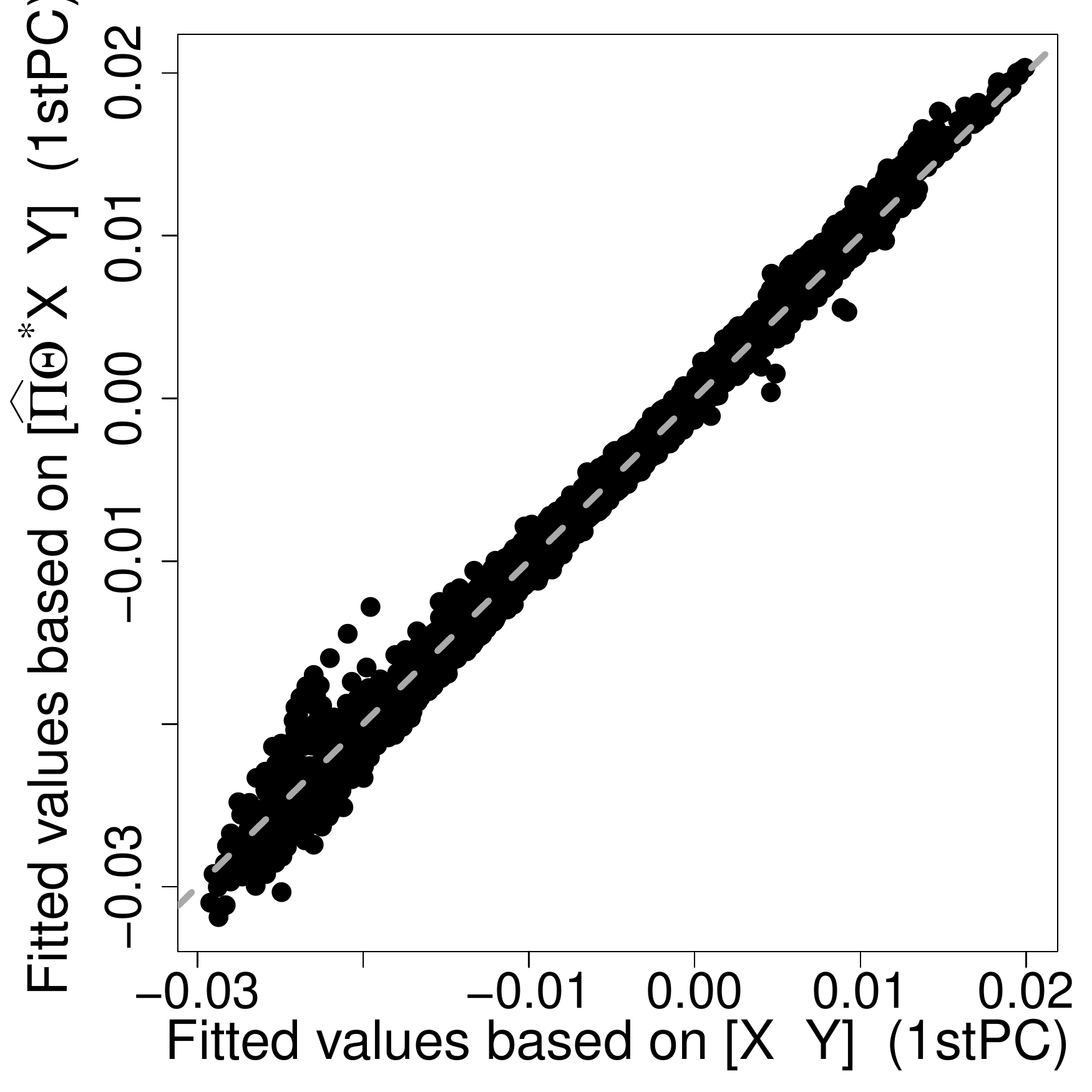}
\end{center}
\vspace*{-4ex}
\caption{Table at the top: RMSEs of various quantities $(A,B)$, i.e., $\nnorm{A - B}_F/\sqrt{n \cdot m}$. The first entry equals the RMSE of the original least squares fit, the second entry equals the mismatch error introduced by
  $\Theta^*$, and the remaining entries show the reduction based on~\eqref{eq:rematch_casestudy} in combination
  with three methods for obtaining $\wh{B}$. Top plots: Mismach error vs.~residual error, before (left) and after
  correction based on~\eqref{eq:rematch_casestudy} with $\wh{B}$ from~\eqref{eq:grouplasso} (right). Bottom plots:
  the fitted values based on $[\Theta^* X \;\; Y]$ vs.~fitted values based on $[X \; Y]$ (left), and the fitted values based on $[\Theta^* X \;\; Y]$ vs.~fitted values based on $[\wh{\Pi}(\Theta^* X) \;\; Y]$ (right). ``Fitted values" here refer to the projection on the leading eigenvector
(first principal component)~of~$X B^*$.}\label{fig:casestudy}
\end{figure}
\renewcommand{\thefootnote}{\arabic{footnote}}
\section{Conclusion}\label{sec:conclusion}
In this paper, we have presented a computationally appealing two-stage approach to multivariate linear regression
in the presence of a small to moderate number of mismatches. The proposed approach can be used to safeguard
against a potentially dramatic increase in the estimation error that can be incurred when ignoring the possibility
of mismatches, as demonstrated in terms of statistical analysis and supported by a series of empirical results. Moreover,
under certain conditions involving ``separability'' of pairs of data points and the signal-to-noise ratio, it is shown
that the true correspondence between those pairs can be perfectly recovered. A key result in this paper asserts that the availability of multiple, linearly independent response variables (as measured by the stable rank of the regression coefficients) considerably simplifies the problem as it increases separability.

A limitation of the proposed approach is that it imposes a stringent limit on the allowed fraction of mismatches. In fact,
as long as a sufficiently large superset of correctly matched data (of size $\Omega(n)$) can be identified, the regression
parameter can still be estimated at the usual rate. Accordingly, the given problem does not appear hopeless even for significantly larger fraction of mismatches, say, up to $1 - \delta$ for $\delta$ bounded away from zero. Closing this gap is a worthwhile endeavor for future research. A second direction of future work concerns extension of the setup beyond classical linear models, specifically more flexibility regarding the range of the response variables (binary, mixed discrete/continuous etc.).

\bibliography{references_M}
\bibliographystyle{jmlr2e_org}

\appendix

\section{Proof of Theorem~\ref{theo:parameter_estimation}}

\noindent (I) \emph{Bound on} $\nnorm{\Xi^* - \wh{\Xi}}_F$.
\vskip1ex
\noindent A crucial observation is that the joint optimization problem~\eqref{eq:grouplasso} in
$B$ and $\Xi$ can be decomposed into two optimization problems involving only $B$ and $\Xi$, respectively, as
stated in the following Lemma.
\begin{lemmaApp}\label{lem:decomposition}
Consider optimization problem~\eqref{eq:grouplasso} with solution $(\wh{B}, \wh{\Xi})$ and denote by $\pre_X^{\perp}$ the projection on the orthogonal complement of $\text{\emph{range}}(X)$. Then, if $n \geq d$, with probability one
\begin{align}
&\wh{\Xi} \in \mathfrak{X}, \quad \mathfrak{X} \coloneq \argmin_{\Xi} \frac{1}{2n \cdot m} \nnorm{\pre_X^{\perp} (Y - \sqrt{n}  \Xi)}_2^2 + \lambda \su \nnorm{\Xi_{i,:}}_2, \label{eq:opt_Xi} \\
&\wh{B} \in \bigg\{ \left(\frac{X^{\T} X}{n} \right)^{-1} \frac{X^{\T} (Y - \sqrt{n} \wh{\Xi})}{n}, \; \, \wh{\Xi} \in \mathfrak{X} \bigg \} \label{eq:opt_B}. 
\end{align}
\end{lemmaApp}    

\noindent The proof is along the lines of the proof of Lemma 1 in~\cite{SlawskiBenDavid2017}, and is hence omitted. Note that
$\texttt{P}_{X}^{\perp} Y =  \texttt{P}_{X}^{\perp} (\sqrt{n} \Xi^* +  \sigma \wt{E})$ with $\wt{E} = \mc{S} E$. The optimization problem in~\eqref{eq:opt_Xi} thus becomes
\begin{equation}\label{eq:opt_Xi_2}
\min_{\Xi} \frac{1}{2n \cdot m} \nnorm{\texttt{P}_{X}^{\perp} (\sqrt{n} \Xi^* + \sigma \wt{E} - \sqrt{n}  \Xi)}_2^2 + \lambda \su \nnorm{\Xi_{i,:}}_2
\end{equation}
In the sequel, we study an equivalent vectorized problem. Accordingly, we define
\begin{align}\label{eq:vectorization}
\begin{split}  
&  \xi^* = [(\Xi_{:,1}^*)^{\T}; \ldots ; (\Xi_{:,m}^*)^{\T}] \in \R^{n \cdot m}, \quad \wt{e} = [\wt{E}_{:,1}^{\T}; \ldots ; \wt{E}_{:,m}^{\T}] \\[.75ex]
&\texttt{P}_{X}^{\perp\otimes} = I_m \otimes \texttt{P}_X^{\perp} = {\footnotesize \begin{pmatrix}
                                   \texttt{P}_X^{\perp}  & 0 & \ldots & 0  \\
                                     0  & \texttt{P}_X^{\perp} & \ddots & \vdots  \\
                                     \vdots    & \ddots  & \ddots & 0 \\
                                      0   & \ldots & 0 & \texttt{P}_X^{\perp} 
                                    \end{pmatrix}},
\end{split}                                 
\end{align}
with $\otimes$ denoting the Kronecker product, the subscripts $_{:,j}$ refer to the $j$-th column, $j = 1,\ldots,m$, and ``$;$'' here means row-wise concatenation. Moreover, for any $v \in \R^{n \cdot m}$, we let
\begin{equation*}
v^{[i]} = (v_j)_{j \in G_i}, \; i=1,\ldots,n,  \quad G_i = \{i, i + n, \ldots, i + (m-1) \cdot n \}. 
\end{equation*}
With this in place, the $(2,q)$-norm with respect to $G_1, \ldots, G_n$ is defined by 
\begin{align}\label{eq:mixednorms}
  &\nnorm{v}_{2,q} \coloneq \left( \sum_{i = 1}^n \nnorm{v^{[i]}}_2^q \right)^{1/q}, \;\, 1 \leq q < \infty, \quad \text{and} \;\, \nnorm{v}_{2,\infty} \coloneq \max_{1 \leq i \leq n} \nnorm{v^{[i]}}_2,  \\
  & \nnorm{v}_{2,0} \coloneq \sum_{i = 1}^n \mathbb{I}(\nnorm{v^{[i]}}_2 > 0),  
\end{align}
where the latter is not a norm; it counts the number of non-zero groups of
components, with each of the $\{ G_i \}_{i = 1}^n$ forming a group. Note that $\nnorm{\xi^*}_{2,0} \leq k$ with support
\begin{equation*}
S_* = \{1 \leq i \leq n: \, \Theta_{ii}^* \neq 1 \} = \{1 \leq i \leq n: \, \nnorm{\xi^{*[i]}}_2 > 0 \}. 
\end{equation*}
We also observe that for all $v, w \in \R^{n \cdot m}$
\begin{equation}\label{eq:mixednorms_prop}
\nnorm{v}_{2,2} = \nnorm{v}_2, \;\,\quad  |\scp{v}{w}| = \left| \sum_{i = 1}^n v^{[i]\T} w^{[i]} \right| \leq  \sum_{i = 1}^n \nnorm{v^{[i]}} \nnorm{w^{[i]}}_2 \leq  \nnorm{v}_{2,1} \nnorm{w}_{2,\infty} 
\end{equation}
by the inequalities of Cauchy-Schwarz and H\"older.

After these preparations, we are in position to state another Lemma. First note that optimization problem~\eqref{eq:opt_Xi_2}
can be expressed in vectorized form as
\begin{equation}\label{eq:optxi_vec}
\min_{\xi} \frac{1}{2 n \cdot m} \nnorm{\texttt{P}_{X}^{\perp\otimes}  (\sqrt{n} \xi^* +  \sigma \wt{e})   -  \texttt{P}_{X}^{\perp\otimes} \xi \sqrt{n}}_2^2 + \lambda \su \nnorm{\xi^{[i]}}_2,   
\end{equation}
Letting $\wh{\delta} = \xi^* - \wh{\xi}$, where $\wh{\xi}$ is a minimizer of~\eqref{eq:optxi_vec}, we have the following basic
inequality
\begin{equation}\label{eq:basicinequality}
  \frac{1}{2 n \cdot m} \nnorm{\texttt{P}_{X}^{\perp\otimes}  \sqrt{n} \wh{\delta}}_2^2 + \lambda \su \nnorm{\wh{\xi}^{[i]}}_2 \leq \frac{1}{\sqrt{n} \cdot m} |\nscp{\texttt{P}_{X}^{\perp\otimes} \wh{\delta} }{\sigma \wt{e}}| +  \lambda \sum_{i \in S_*} \nnorm{\xi^{*[i]}}_2,
\end{equation}
which is obtained by evaluating~\eqref{eq:optxi_vec} at $\xi = 0$, expanding squares and re-arranging. 
\begin{lemmaApp}\label{lem:conecondition} Consider $\wh{\delta}$ in~\eqref{eq:basicinequality} and Let $\lambda_0$ be a number such that
  \begin{equation}\label{eq:lambda0}
   \frac{1}{\sqrt{n} \cdot m} \nnorm{\texttt{\emph{P}}_{X}^{\perp\otimes}  \sigma \wt{e}}_{2, \infty} \leq \lambda_0. 
    \end{equation}
    Then for any $\lambda \geq 2 \lambda_0$, it holds that either $\wh{\delta} = 0$ or $\wh{\delta}/\nnorm{\wh{\delta}}_2 \in 2 \,\text{\emph{conv}}(B_0(k')) \cap \mathbb{S}^{n \cdot m - 1}$, where for $r \geq 0$, $B_0(r) = \{v \in \R^{n \cdot m}: \; \nnorm{v}_{2,0} \leq r, \; \nnorm{v}_2 \leq 1 \}$ according to~\eqref{eq:mixednorms} and
    $k' = \left(1 + \frac{\lambda + \lambda_0}{\lambda - \lambda_0} \right)^2 k \leq 16k$. 
  \end{lemmaApp}
  \begin{proof}
As an immediate consequence of~\eqref{eq:basicinequality} and the triangle inequality, we obtain that 
\begin{equation*}
\lambda \sum_{i \in S_*^c} \nnorm{\wh{\delta}^{[i]}}_2 \leq \frac{1}{\sqrt{n} \cdot m} |\nscp{\texttt{P}_{X}^{\perp\otimes}  \wh{\delta} }{\sigma \wt{e}}| +  \lambda \sum_{i \in S_*} \nnorm{\wh{\delta}^{[i]}}_2 \leq \lambda_0 \nnorm{\wh{\delta}}_{2,1}    +  \lambda \sum_{i \in S_*} \nnorm{\wh{\delta}^{[i]}}_2,  
\end{equation*}
where the second inequality is a result of~\eqref{eq:mixednorms_prop} and~\eqref{eq:lambda0}. If $k = 0$,
$S_* = \emptyset$, we must have $\wh{\delta} = \wh{\xi} = \xi^* = 0$ as the above inequality would be violated otherwise, and the claim of the lemma follows. On the other hand, if $k \geq 1$, combination of the left and right hand side
of the above chain of inequalities yields
\begin{align}
&\lambda \sum_{i \in S_*^c} \nnorm{\wh{\delta}^{[i]}}_2 \leq  \lambda_0 \nnorm{\wh{\delta}}_{2,1} +  \lambda \sum_{i \in S_*} \nnorm{\wh{\delta}^{[i]}}_2 = \lambda_0 \left(\sum_{i \in S_*} \nnorm{\wh{\delta}^{[i]}}_2    +  \sum_{i \in S_*^c} \nnorm{\wh{\delta}^{[i]}}_2 \right) + \lambda \sum_{i \in S_*} \nnorm{\wh{\delta}^{[i]}}_2 \notag\\ 
  \Rightarrow & \sum_{i \in S_*^c} \nnorm{\wh{\delta}^{[i]}}_2 \leq \frac{\lambda + \lambda_0}{\lambda - \lambda_0} \sum_{i \in S_*} \nnorm{\wh{\delta}^{[i]}}_2 \notag\\
\Rightarrow & \nnorm{\wh{\delta}}_{2,1} \leq  \left(1 + \frac{\lambda + \lambda_0}{\lambda - \lambda_0} \right) \sum_{i \in S_*} \nnorm{\wh{\delta}^{[i]}}_2 \leq \left(1 + \frac{\lambda + \lambda_0}{\lambda - \lambda_0} \right) \sqrt{k} \nnorm{\wh{\delta}}_2  \label{eq:conebound}   
\end{align}
The assertion then follows from Lemma~\ref{lem:sparsity_conv} provided in a separate section below.  
\end{proof}
\noindent As in the above Lemma, under event~\eqref{eq:lambda0}, inequality~\eqref{eq:basicinequality} implies
\begin{equation}\label{eq:basicinequality_cons}
  \frac{1}{2 n \cdot m} \nnorm{\texttt{P}_{X}^{\perp\otimes}  \sqrt{n} \wh{\delta}}_2^2 +  \leq  \left(\lambda_0 \left(1 + \frac{\lambda + \lambda_0}{\lambda - \lambda_0} \right) + \lambda \right) \sqrt{k} \nnorm{\wh{\delta}}_2 = \lambda \left(\frac{\lambda + \lambda_0}{\lambda - \lambda_0} \right) \sqrt{k} \nnorm{\wh{\delta}}_2
\end{equation}
by following the steps leading to~\eqref{eq:conebound}. We now lower bound the l.h.s.~of~\eqref{eq:basicinequality_cons}. Let $\Lambda = \{(\lambda_s)_{s = 1}^N \subset \R_+: \; N \in \{1,2,\ldots,\}, \; \sum_{s = 0}^{N} \lambda_s \leq 2\}$. In light of Lemma~\ref{lem:conecondition}, we have 
\begin{align*}
\frac{1}{n} \nnorm{\texttt{P}_{X}^{\perp\otimes} \sqrt{n} \wh{\delta}}_2^2 &
 \geq \nnorm{\wh{\delta}}_2^2 \;\,\min_{\substack{\{\lambda_{s} \} \in \Lambda, \, \{ v_{s} \} \subset B_0(k'), \\ \sum_{s} \lambda_{s} v_{s} \in \mathbb{S}^{n \cdot m - 1}}} \; \nnorm{\texttt{P}_{X}^{\perp\otimes} \textstyle \sum_{s} \lambda_{s} v_{s}}_2^2
\end{align*}
Structuring each $v_{s}$ into sub-vectors $v_{s}^{(l)} \in \R^n$, $l=1,\ldots,m$, we obtain 
\begin{align*}
\min_{\substack{\{\lambda_{s} \} \in \Lambda, \, \{ v_{s} \} \subset B_0(k'), \\ \sum_{s} \lambda_{s} v_{s} \in \mathbb{S}^{n \cdot m - 1}}} \; \nnorm{\texttt{P}_{X}^{\perp\otimes} \textstyle \sum_{s} \lambda_{s} v_{s}}_2^2 &=
\min_{\substack{\{\lambda_{s} \} \in \Lambda, \, \{ v_{s} \} \subset B_0(k'), \\ \sum_{s} \lambda_{s} v_{s} \in \mathbb{S}^{n \cdot m - 1}}} \; \sum_{l = 1}^m \nnorm{\texttt{P}_{X}^{\perp} \textstyle \sum_{s} \lambda_{s} v_{s}^{(l)}}_2^2
\end{align*}
Since each  $v_{s}$ is $k'$-group sparse according to the partitioning defined by $\{ G_i \}_{i = 1}^n$, each 
$v_{s}^{(l)}$ is at most $k'$-sparse in the ordinary sense, i.e., having at most $k'$ non-zero entries. Letting 
$\mc{B}_{0}(k') = \{v \in \R^n: \nnorm{v}_0 \leq k' \}$ denote the usual $k'$-sparsity ball in $\R^n$, we have 
\begin{align}\notag
  &\min_{\substack{\{\lambda_{s} \} \in \Lambda, \, \{ v_{s} \} \subset B_0(k'), \\ \sum_{s} \lambda_{s} v_{s} \in \mathbb{S}^{n \cdot m - 1}}} \; \sum_{l = 1}^m \nnorm{\texttt{P}_{X}^{\perp} \textstyle \sum_{s} \lambda_{s} v_{s}^{(l)}}_2^2\\ \notag
=& \min_{\substack{\{\lambda_{s} \} \in \Lambda, \, \{ v_{s}^{(l)} \} \subset \mc{B}_0(k'), \\ \{ \sum_{s} \lambda_{s} v_{s}^{(l)} \} \subset \mathbb{S}^{n - 1} \\ \{\gamma^{(l)} \} \subset \R_+, \, \sum_{l = 1}^m \{\gamma^{(l)}\}^2 = 1}} \; \sum_{l = 1}^m \nnorm{\texttt{P}_{X}^{\perp} \, \gamma^{(l)} \textstyle \sum_{s} \lambda_{s} v_{s}^{(l)}}_2^2 \notag\\
=& \min_{\{\gamma^{(l)} \} \subset \R_+, \, \sum_{l = 1}^m \{\gamma^{(l)}\}^2 = 1} \sum_{l = 1}^m \{ \gamma^{(l)} \}^2 \, \times  \min_{u \in 2 \text{conv}(\mc{B}_{0}(k')) \cap \mathbb{S}^{n - 1}} \nnorm{\texttt{P}_{X}^{\perp} u}_2^2 \notag\\
=& \min_{u \in 2 \text{conv}(\mc{B}_{0}(k')) \cap \mathbb{S}^{n - 1}} \nnorm{\texttt{P}_{X}^{\perp} u}_2^2 \notag\\
=&\text{dist}^2(2 \text{conv}(\mc{B}_{0}(k')) \cap \mathbb{S}^{n - 1}, \text{range}(X)) \label{eq:lowerbound_lhs} 
\end{align}
In order to lower bound this squared distance, we apply Gordon's Theorem (cf.~Lemma~\ref{lem:escape} below) with $K = 2 \text{conv}(\mc{B}_{0}(k')) \cap \mathbb{S}^{n - 1}$ and $V = \text{range}(X)$ noting
that the latter random subspace follows a uniform distribution on the Grassmannian $\textsf{G}(n, d)$, thus
we identify $ p = n$, $p - q = d \Leftrightarrow q = n-d$. It is well-known that $\nu_r = \sqrt{r^2 /(r+1)} = (1 - O(1/\sqrt{r})) \sqrt{r} \sim \sqrt{r}$ as $r \rightarrow \infty$; to simplify our argument, we henceforth replace $\nu_r$ by $\sqrt{r}$. Translated
to the setting under consideration, the condition $w(K) < (1 - \epss) \nu_q - \epss \nu_p$ in Lemma~\ref{lem:escape} reads
\begin{equation}\label{eq:cond_gordon_trans}
\frac{1}{1 - \epss} \, w(2 \text{conv}(\mc{B}_{0}(k')) \cap \mathbb{S}^{n - 1}) <  \sqrt{n-d} - \frac{\epss}{1 - \epss} \sqrt{n}. 
\end{equation}
Invoking the assumption $d/n \leq 1/4$, the r.h.s.~of~\eqref{eq:cond_gordon_trans} evaluates as $(\sqrt{3}/2 - \frac{\epss}{1 - \epss}) \sqrt{n}$. Regarding the l.h.s.~of~\eqref{eq:cond_gordon_trans}, it follows
from standard results (cf.~\cite{PlanVershynin2013a}, Lemma 2.3) that the Gaussian width $w(2 \, \text{conv}(\mc{B}_{0}(k')) \cap \mathbb{S}^{n - 1}) \leq 7 \sqrt{k' \log(e n / k')}$. It thus follows that for any $\epss \in (0,1/3)$, there exists $c_{\epss}, c_{\epss}' > 0$ so that
if
\begin{equation*}
k \leq c_{\epss} \cdot n / \log(n/k)
\end{equation*}
inequality~\eqref{eq:cond_gordon_trans} is satisfied, so that with probability at least $1 - 3.5 \cdot \exp(-c_{\epss}' n)$,
\eqref{eq:lowerbound_lhs} is lower bounded by $\epss^2$. Combining~\eqref{eq:basicinequality_cons} and this lower bound on
\eqref{eq:lowerbound_lhs}, we conclude that
\begin{equation*}
m^{-1/2} \nnorm{\wh{\Xi} - \Xi^*}_F  = m^{-1/2} \nnorm{\wh{\delta}}_2 \leq \epss^{-2} \cdot 2 \lambda \sqrt{m} \cdot \frac{\lambda + \lambda_0}{\lambda - \lambda_0} \sqrt{k}.  
\end{equation*}
The lemma below elaborates on the choice of $\lambda_0$, which completes the proof of the bound on $m^{-1/2} \nnorm{\wh{\Xi} - \Xi^*}_F$. 
\newpage

\begin{lemmaApp} With probability at least $1 - 2/n$, it holds that
\begin{equation*}
  \frac{1}{\sqrt{n} \cdot m} \nnorm{\texttt{\emph{P}}_{X}^{\perp\otimes}  \sigma \wt{e}}_{2, \infty} \leq \lambda_0 \; \; \; \text{with} \; \lambda_0 = \frac{\mu_{n,d} \, \sigma}{\sqrt{n \cdot m}} \left( 1 + \sqrt{\frac{4 \log n}{m}}\right), \; \mu_{n,d} \coloneq \Big(\textstyle\frac{n-d}{n} + \textstyle\sqrt{24 \frac{\log n}{n}} \Big) \wedge 1. 
\end{equation*}
\end{lemmaApp}
\begin{proof}
\begin{equation*}  
  \frac{1}{\sqrt{n} \cdot m} \nnorm{\texttt{P}_{X}^{\perp\otimes} \sigma \wt{e}}_{2,\infty} = \frac{\sigma}{\sqrt{n} \cdot m} \max_{1 \leq i \leq n} \nnorm{E^{\T} \mc{S}^{\T} \texttt{P}_{X}^{\perp} \mathfrak{e}_i}_2,
\end{equation*}
where  $\{ \mathfrak{e}_i \}_{i = 1}^n$ is the canonical basis of $\R^n$.   
Observe that conditional on $\texttt{P}_{X}^{\perp}$, $E^{\T}  \mc{S}^{\T} \texttt{P}_{X}^{\perp} \mathfrak{e}_i$ is
a zero mean-Gaussian random vector with covariance matrix $\nnorm{\mc{S} \,\texttt{P}_{X}^{\perp} \mathfrak{e}_i}_2^2 \cdot I_m $, $1 \leq i \leq n$. Since $\nnorm{\mc{S}}_2 \leq 1$ and since $\texttt{P}_{X}^{\perp}$ is a random projection in the sense of~\cite{DasGupta2003}, it follows from results therein that for all $\mu > 0$
\begin{equation*}
\p\left(\max_{1 \leq i \leq n} \nnorm{\mc{S} \, \texttt{P}_{X}^{\perp} \mathfrak{e}_i}_2^2 \geq \frac{n - d}{n} (1 + \mu) \wedge 1 \right) \leq n \exp\left(-(n-d) \frac{\eta^2}{12} \right).
\end{equation*}
In particular, with the choice $\mu = \textstyle\sqrt{24 \frac{\log n}{n  - d}} \invcoloneq c_1$, 
\begin{equation*}
\p\left(\max_{1 \leq i \leq n} \nnorm{\mc{S} \, \texttt{P}_{X}^{\perp} \mathfrak{e}_i}_2^2 \geq  \mu_{n,d} \right) 
\leq 1/n, \quad \mu_{n,d} \coloneq \Big(\textstyle\frac{n-d}{n} + \textstyle\sqrt{24 \frac{\log n}{n}} \Big) \wedge 1 
\end{equation*} 
Combining this result with Lemma~\ref{lem:gaussiannorm} with $r = m$, $L = n$, $\max_{1 \leq \ell \leq L} \sigma_{\ell} = \mu_{n,d}$, we have
\begin{equation*}
\nnorm{\texttt{P}_{X}^{\perp\otimes} \sigma \wt{e}}_{2,\infty} \leq \mu_{n,d} \, \sigma \{\sqrt{m} + 2 \sqrt{\log n} \}     
\end{equation*}
with probability at least $1 - 2/n$. This finally yields the choice
\begin{equation*}
\lambda_0 =  \frac{\mu_{n,d} \, \sigma}{\sqrt{n \cdot m}} \left( 1 + \sqrt{\frac{4 \log n}{m}}\right).     
\end{equation*}
%
\end{proof}

\noindent (II) \emph{Bound on} $\nnorm{B^* - B}_F$.
\vskip1ex
\noindent Let $\sigma_{\min}(\cdot)$ and $\sigma_{\max}(\cdot)$ denote the minimum and maximum singular value functional, respectively. Invoking Lemma~\ref{lem:decomposition}, we bound
\begin{align}
\nnorm{\wh{B} - B^*}_F &\leq \frac{\norm{ \left( \textstyle\frac{X^{\T} X}{n} \right)^{-1} \frac{X^{\T}}{\sqrt{n}} (\sigma \mc{S} E + \sqrt{n}(\Xi^* - \wh{\Xi})}_F}{\sqrt{n}}  \notag \\
&\leq \sigma \frac{\norm{ \left( \textstyle\frac{X^{\T} X}{n} \right)^{-1} \frac{X^{\T}}{\sqrt{n}} \mc{S} E}_F}{\sqrt{n}} +
\frac{\nnorm{\wh{\Xi}  - \Xi^*}_F}{\sigma_{\min}(X/\sqrt{n})} \label{eq:leastsquaresnoise}, 
\end{align}
where we have used that $\left( \textstyle\frac{X^{\T} X}{n} \right)^{-1} \frac{X^{\T}}{\sqrt{n}} = \left( \frac{X}{\sqrt{n}} \right)^{\dagger}$, with $^{\dagger}$ denoting the Moore-Penrose pseudo-inverse, and $\sigma_{\max}\left(\left( \frac{X}{\sqrt{n}} \right)^{\dagger} \right) = \sigma_{\min}^{-1}((X/\sqrt{n})^{\dagger}) $. Consider $\Gamma = \mc{S} \frac{X}{\sqrt{n}} \left( \frac{X^{\T} X}{n} \right)^{-2} \frac{X^{\T}}{\sqrt{n}} \mc{S}$, and let
$\Gamma^{\otimes} = I_m \otimes \Gamma$. We then can write
\begin{equation*}
\norm{ \left( \textstyle\frac{X^{\T} X}{n} \right)^{-1} \frac{X^{\T}}{\sqrt{n}} \mc{S} E}_F^2 = \nnorm{\Gamma^{\otimes} e}_2^2,  
\end{equation*}
where $e$ is a standard Gaussian random vector of dimension $n \cdot m$. By straightforward adaptations of Lemma 3 in~\cite{SlawskiBenDavid2017} that is based on a concentration result for quadratic forms in~\cite{Hsu2012}, we obtain that 
\begin{equation*}
\p \left( \norm{ \left( \textstyle\frac{X^{\T} X}{n} \right)^{-1} \frac{X^{\T}}{\sqrt{n}} E}_F > \frac{\sqrt{5 (d \cdot m \vee \log(n \cdot m))}}{\sigma_{\min}(X/\sqrt{n})} \; \bigg| \, X\right) \leq \exp(-(d \cdot m) \vee \log(n \cdot m))  
\end{equation*}
The proof is completed by appealing to concentration results (e.g., Corollary 5.35 in~\cite{Vershynin2010}) to lower bound $\sigma_{\min}(X/\sqrt{n})$ with $X$ having
i.i.d.~standard Gaussian entries.

\section{Proofs of Lemmas~\ref{lem:gamma_min} and~\ref{lem:gamma_min_const}}\label{app:separation}
Lemma~\ref{lem:gamma_min} is an immediate consequence of the following result.

\begin{lemmaApp}\label{theo:latala}(Proposition 2.6 in~\cite{Latala2007}) \\
Let $g \sim N(0, I_d)$. There exist universal constants $\alpha_0 \in (0,1)$ and $\kappa > 0$ such that for any $\alpha \in (0, \alpha_0)$
\begin{equation*}
  \sup_{\mu \in \R^m} \p\left(\nnorm{\mu - B^{*\T} g}_2 \leq \alpha \nnorm{B^*}_F \right) \leq \exp \left(\kappa \log(\alpha)
  \, \text{\emph{srank}$(B^*)$} \right). 
\end{equation*}
\end{lemmaApp}
\noindent Lemma~\ref{lem:gamma_min} is obtained by applying Lemma~\ref{theo:latala} with $\mu = 0$, $g = \frac{\M{x}_i - \M{x}_j}{\sqrt{2}}$, and then using a union bound over pairs, i.e., $\{ \min_{i < j} \nnorm{B^{*\T}(\M{x}_i - \M{x}_j)}_2  \leq \delta \} \subseteq
\bigcup_{i < j} \{  \nnorm{B^{*\T}(\M{x}_i - \M{x}_j)}_2 \leq \delta \}$ for any $\delta > 0$. We then choose
$\alpha$ as the term inside the curly brackets in~\eqref{eq:gamma_latala} to conclude the result. 
\vskip2ex
\noindent \tcb{\emph{Remark 1.} Lemma~\ref{theo:latala} immediately implies that the quantity $\gamma_0^2$~\eqref{eq:gamma0} exhibits qualitatively the same lower bound as $\gamma^2$ according to Lemma~\ref{lem:gamma_min}: since it is assumed that $\{\M{y}_i: \, i \in \mc{N} \}$ and $\{\M{x}_j: \, 1 \leq j \leq n \}$ are independent, we have
\begin{align*}
\p \left( \min_{\substack{i \in \mc{N} \\ 1 \leq j \leq n}} \nnorm{\M{y}_i - B^{*\T} \M{x}_j}_2 \leq \delta \right) &\leq
\sum_{i \in \mc{N}} \sum_{j = 1}^n \E_{\M{y}_i}\left[ \p( \nnorm{\M{y}_i - B^{*\T} \M{x}_j}_2 \leq \delta \,|\M{y}_i) \right] \\
&\leq |\mc{N}| n \sup_{\mu \in \R^m} \p( \nnorm{\mu  - B^{*\T} \M{x}_j}_2 \leq \delta ), 
\end{align*}
and thus Lemma~\ref{theo:latala} can be applied as in the proof of Lemma~\ref{lem:gamma_min}. Since $|\mc{N}| n \lesssim \binom{n}{2}$, the lower bound~\eqref{eq:gamma_latala} also holds true for $\gamma_0^2$ up to a constant factor, i.e., $\gamma_0^2 \gtrsim \gamma^2$.}
\vskip1ex
\noindent \tcb{\emph{Remark 2.} A similar albeit slightly weaker result than Lemma~\ref{theo:latala} holds true if the entries 
of $g$ are independent, unit variance \emph{sub}-Gaussian random variables (see, e.g., $\S$2.5 in~\cite{Vershynin2018}). Specifically, Theorem 2.5 in~\cite{Latala2007} implies that
\begin{equation*}
\sup_{\mu \in \R^m} \p\left(\nnorm{\mu - B^{*\T} g}_2 \leq \frac{1}{2} \nnorm{B^*}_F \right) \leq 
2 \exp \left(-c \cdot \text{srank}(B^*) \right),
\end{equation*}
for some constant $c > 0$. The main difference of the above result and that of Lemma~\ref{theo:latala} is that the tail bound in the latter can still be driven to zero even if  $\text{srank}(B^*) = O(1)$ by choosing the parameter $\alpha$ appropriately. On the other hand, if $\alpha$ is chosen as a constant bounded away from zero, the two results yield qualitatively the same conclusions.  
}  

\vskip3ex
\noindent Regarding Lemma~\ref{lem:gamma_min_const}, we first prove the lower bound. We observe that
under the assumption of $B^*$ having constant non-zero singular values, $\nnorm{B^{*\T}(\M{x}_i - \M{x}_j)}_2^2 \sim
2 b_*^2 \chi^2(r)$, where $\chi^2(\nu)$ denotes the Chi-Square distribution with $\nu \in \{1,2,\ldots\}$ degrees of freedom. It is easy to verify that for $r = 2(q + 1)$, $q \in \{0,1,\ldots\}$, 
\begin{equation}\label{eq:chisquarecdf}
\p(\chi^2(r) \leq z) = 1 - \exp(-z/2) \sum_{s = 0}^q \frac{(z/2)^s}{s!}, \quad z \geq 0. 
\end{equation}
Combining~\eqref{eq:chisquarecdf} with a union bound over pairs $i < j$, we obtain
\begin{equation}\label{eq:prob_upper_bound}
  \p \left(\min_{i < j} \nnorm{B^{*\T}(\M{x}_i - \M{x}_j)}_2^2 \leq 2 b_*^2 z \right) \leq \binom{n}{2}
  \left( 1 - \exp(-z/2) \sum_{s = 0}^q \frac{(z/2)^s}{s!} \right)
\end{equation}
Below, $z$ is chosen s.t.~the r.h.s.~of the above inequality is upper bounded by $\delta$. We have
\begin{align}\label{eq:taylorbound}
  \binom{n}{2} \left(1 - \exp(-z/2) \sum_{s= 0}^q \frac{(z/2)^s}{s!} \right) 
  = \binom{n}{2} \left(\exp(-z/2) \sum_{s=q+1}^{\infty} \frac{(z/2)^s}{s!} \right) 
  \leq \binom{n}{2} \frac{(z/2)^{q+1}}{(q+1)!},
\end{align}
where the inequality follows from a Taylor expansion with Lagrange form of the remainder:
\begin{align*}
  &\exp(z/2) = \sum_{s = 0}^{q} \frac{(z/2)^{s}}{s!} + \frac{\exp(\xi)}{(q+1)!} (z/2)^{q+1} \; \; \text{for some} \; \xi \in [0,z/2]\\
  \Rightarrow&
    \;\;\exp(z/2) - \sum_{s = 0}^{q} \frac{(z/2)^{s}}{s!} = \sum_{s=q+1}^{\infty} \frac{(z/2)^s}{s!}
    = \frac{\exp(\xi)}{(q+1)!} (z/2)^{q+1} \leq \exp(z/2) \frac{(z/2)^{q+1}}{(q+1)!}. 
\end{align*}
Using that $\frac{1}{(q+1)!} \leq ((q+1)/e)^{-(q+1)}$,~\eqref{eq:taylorbound} can be upper bounded as 
\begin{align*}
  \binom{n}{2}  \exp\left(-(q+1) \log \left(\frac{2 (q + 1)}{z \cdot e} \right) \right) \leq \frac{n^2}{2}  \left(\frac{2 (q + 1)}{z \cdot e} \right)^{-(q+1)}   
\end{align*}
Choosing $z = \frac{2}{e} (q + 1)   \cdot (n^{-2} \delta  )^{1/(q+1)}$ ensures that the probability in~\eqref{eq:prob_upper_bound}  is bounded by $\frac{\delta}{2}$.
\vskip1ex
\noindent We turn to the upper bound in Lemma~\ref{lem:gamma_min_const}. Let $n_2 = \lfloor \frac{n}{2} \rfloor$. We first use that for any $z \geq 0$
\begin{align}
\p \left(\min_{i < j}\nnorm{\M{B}^{*\T} (\M{x}_i - \M{x}_j)}_2^2  < z \right)
  &\geq \p \left(\min_{1 \leq i \leq n/2} \nnorm{(\M{B}^*)^{\T} (\M{x}_{2i} - \M{x}_{2i-1})}_2^2  < z \right) \notag \\
  &= 1 - \p(\chi ^2(r) > z/2b_*^2)^{n_2} \label{eq:end_first_display},
\end{align}
where we have used that $\{ \nnorm{\M{B}^{*\T} (\M{x}_{2i} - \M{x}_{2i-1})}_2^2 \}_{ i = 1}^{n_2} \overset{\text{i.i.d.}}{\sim} 2b_*^2 \chi^2(r)$. Using~\eqref{eq:chisquarecdf} and setting $z = c \cdot 4 b_*^2$ in~\eqref{eq:end_first_display} for $c > 0$ to be determined below, we obtain that 
\begin{align}
\p \left(\min_{i < j}\nnorm{\M{B}^{*\T} (\M{x}_i - \M{x}_j)}_2^2  < z \right) 
&\geq 1 - \left( \sum_{s = 0}^q \frac{c^s}{s!} \exp(-c) \right)^{n_2} \notag \\
 &= 1 - \left( 1 -  \sum_{s = q+1}^{\infty} \frac{c^s}{s!}  \exp(-c)  \right)^{n_2} \notag \\
 &\geq   1 - \left( 1 - \textstyle\frac{c^{q+1}}{(q+1)!}  \exp(-c)  \right)^{n_2} \label{eq:end_second_display}
\end{align} 
Choosing $c = \theta^{1/(q+1)} n^{-1/(q+1)} (q + 1)$ and using that $(q+1)! < (q+1)^{q+1}$, we obtain the following lower bound on~\eqref{eq:end_second_display}
\begin{align*}
  1 -  \left( \left( 1 - \textstyle\frac{\theta}{n}  \exp(-c)  \right)^{n} \right)^{1/2} 
  \geq 1 - \exp\left(-(\theta/2)  \exp(-c)\right)
\end{align*}
as long as $n \geq \theta$. Setting $\theta = 8$, the above probability is lower bounded by $0.75$ if $n > 8 (q+1)^{q+1}$. Combining this with the choice of $z = c \cdot 4 b_*^2$ in~\eqref{eq:end_first_display} yields the assertion.

\section{Proof of Theorem~\ref{theo:correspondence}}
We first show that $\wh{\Theta}(\wh{B})_{i,:} = \Theta_{i,:}^* = 0$ for $i \in \mc{N} = \{1 \leq i \leq n:\; \theta^*(i) = 0 \}$. For this purpose, it needs to be established that $\min_{i \in \mc{N}} \min_{1 \leq j \leq n} \, \nnorm{\M{y}_i - \wh{B}^{\T} \M{x}_j}_2 > \tau$. We have 
\begin{align*}
  \min_{i \in \mc{N}} \min_{1 \leq j \leq n} \, \nnorm{\M{y}_i - \wh{B}^{\T} \M{x}_j}_2
  &\geq \min_{i \in \mc{N}} \min_{1 \leq j \leq n} \; \nnorm{\M{y}_i -  B^{*\T} \M{x}_j}_2 - \max_{1 \leq j \leq n} \nnorm{\M{x}_j}_2 \nnorm{B^* - \wh{B}}_2 \\
    &\geq  \gamma_0 \sigma \sqrt{m} \textsf{SNR}^{1/2} - \max_{1 \leq j \leq n} \nnorm{\M{x}_j}_2 \nnorm{B^* - \wh{B}}_2 \\
    &> 2 \max \left\{\max_{1 \leq j \leq n} \nnorm{\M{x}_j}_2 \nnorm{B^* - \wh{B}}_2, \, \tau \right \} - \max_{1 \leq j \leq n} \nnorm{\M{x}_j}_2 \nnorm{B^* - \wh{B}}_2 > \tau, 
\end{align*}
in view of the event $\mc{B}$ defined in the theorem. 

\noindent Next, we show that $\wh{\Theta}(\wh{B})_{i,:} \neq 0$ if $i \in \mc{N}^c$. This is implied by 
demonstrating that $\max_{i \in \mc{N}^c} \nnorm{\M{y}_i - \wh{B}^{\T} \M{x}_{\theta^*(i)}}_2 \leq \tau$. We have  
\begin{align*}
\max_{i \in \mc{N}^c} \nnorm{\M{y}_i - \wh{B}^{\T} \M{x}_{\theta^*(i)}}_2 &\leq \max_{i \in \mc{N}^c} \nnorm{\M{y}_i - B^{*\T} \M{x}_{\theta^*(i)}}_2 + \max_{1 \leq j \leq n} \nnorm{\M{x}_j}_2 \nnorm{B^* - \wh{B}}_2 \\
&\leq \sigma \max_{1 \leq i \leq n} \nnorm{\eps_i}_2 + \max_{1 \leq j \leq n} \nnorm{\M{x}_j}_2 \nnorm{B^* - \wh{B}}_2. 
\end{align*}
Consider the event
\begin{equation}\label{eq:corr_noisebound}
\left\{ \sigma \max_{1 \leq i \leq n} \nnorm{\eps_i}_2 \leq \sigma \sqrt{m} + 2 \sqrt{\log n} \right \}.
\end{equation}
By Lemma~\ref{lem:gaussiannorm}, event~\eqref{eq:corr_noisebound} holds with probability at least $1 - 1/n$. Observe
that conditional on the event~\eqref{eq:corr_noisebound}, $\max_{i \in \mc{N}^c} \nnorm{\M{y}_i - \wh{B}^{\T} \M{x}_{\theta^*(i)}}_2 \leq \tau_0 < \tau$ with $\tau_0$ as defined in Theorem~\ref{theo:correspondence}.

\noindent Finally, we show that for $i \in \mc{N}^c$, it holds that $\wh{\Theta}(\wh{B})_{i\theta^*(i)} = 1$ which then in conjunction
with the two previous results implies that $\wh{\Theta}(\wh{B}) = \Theta^*$. For this purpose, we consider 
\begin{align}
  &\bigcap_{i \in \mc{N}^c} \bigcap_{\substack{1 \leq j \leq n \\ j \neq \theta^*(i)}} \left\{ \nnorm{\M{y}_i - \wh{B}^{\T} \M{x}_{\theta^*(i)}}_2^2 \leq \nnorm{\M{y}_i - \wh{B}^{\T} \M{x}_j}_2^2 \right \} \notag\\
  =& \bigcap_{i \in \mc{N}^c} \bigcap_{\substack{1 \leq j \leq n \\ j \neq \theta^*(i)}} \left\{ \nnorm{(B^{*} - \wh{B})^{\T} \M{x}_{\theta^*(i)} + \sigma \eps_{i}}_2^2 \leq \nnorm{B^{*\T}  \M{x}_{\theta^*(i)} - \wh{B}^{\T} \M{x}_j + \sigma \eps_{i}}_2^2 \right \} \notag\\
  =& \bigcap_{i \in \mc{N}^c} \bigcap_{\substack{1 \leq j \leq n \\ j \neq \theta^*(i)}}
  \Big\{ \nnorm{(B^{*}  - \wh{B})^{\T} \M{x}_{\theta^*(i)}}_2^2 + 2\nscp{(B^{*}  - \wh{B})^{\T} \M{x}_{\theta^*(i)}}{\sigma \eps_{i}} \notag\\[-3ex]
  &\qquad \qquad \qquad \leq \nnorm{B^{*\T}  \M{x}_{\theta^*(i)} - \wh{B}^{\T} \M{x}_j}_2^2 + 2 \nscp{B^{*\T}  \M{x}_{\theta^*(i)} - \wh{B}^{\T} \M{x}_j}{\sigma  \eps_{i}}   \Big \}  \notag\\
 =& \bigcap_{i \in \mc{N}^c} \bigcap_{\substack{1 \leq j \leq n \\ j \neq \theta^*(i)}}
  \Big\{ \nnorm{(B^{*}  - \wh{B})^{\T} \M{x}_{\theta^*(i)}}_2^2 + 2\nscp{(B^{*}  - \wh{B})^{\T} \M{x}_{\theta^*(i)}}{\sigma \eps_{i}} \notag\\[-3ex]
  &\qquad \qquad \qquad \leq \nnorm{B^{*\T}  (\M{x}_{\theta^*(i)} - \M{x}_j)}_2^2 + \nnorm{(\wh{B} - B^*)^{\T} \M{x}_j}_2^2  +
    \notag\\
  &\qquad \qquad \qquad  \qquad + 2 \nscp{B^{*\T}  (\M{x}_{\theta^*(i)} - \M{x}_j)}{(B^{*\T}  - \wh{B}^{\T}) \M{x}_j} + 2 \nscp{B^{*\T}  \M{x}_{\theta^*(i)} - \wh{B}^{\T} \M{x}_j}{\sigma  \eps_{i}}   \Big \} \notag \\
 =& \bigcap_{i \in \mc{N}^c} \bigcap_{\substack{1 \leq j \leq n \\ j \neq \theta^*(i)}} \Big\{
  \nnorm{(B^{*}  - \wh{B})^{\T} \M{x}_{\theta^*(i)}}_2^2 - \nnorm{(B^{*}  - \wh{B})^{\T} \M{x}_{j}}_2^2 + \notag\\[-2ex]
  &\qquad \qquad \qquad  + 2 \nscp{(\wh{B} - B^*)^{\T}(\M{x}_j - \M{x}_{\theta^*(i)})}{\sigma  \eps_{i}}
    + 2 \nscp{B^{*\T}(\M{x}_j - \M{x}_{\theta^*(i)})}{\sigma  \eps_{i}} + \notag\\
  &\qquad \qquad \qquad  + 2 \nscp{B^{*\T}(\M{x}_{\theta^*(i)} - \M{x}_{j})}{(B^{*\T}  - \wh{B}^{\T}) \M{x}_j}
    \leq \nnorm{B^{*\T}  (\M{x}_{\theta^*(i)} - \M{x}_j)}_2^2 \Big \} \notag \\
  \supseteq& \bigcap_{i \in \mc{N}^c} \bigcap_{\substack{1 \leq j \leq n \\ j \neq \theta^*(i)}} \Bigg\{
  \frac{\nnorm{(B^{*}  - \wh{B})^{\T} \M{x}_{\theta^*(i)}}_2^2}{\nnorm{B^{*\T}  (\M{x}_{\theta^*(i)} - \M{x}_j)}_2^2}  +
  \frac{2 \nnorm{\sigma  \eps_{i}}_2}{\nnorm{B^{*\T}  (\M{x}_{\theta^*(i)} - \M{x}_j)}_2} + \notag\\
  &\qquad \qquad \qquad \qquad  + \frac{2\nnorm{(B^{*\T}  - \wh{B}^{\T}) \M{x}_j}_2}{\nnorm{B^{*\T}  (\M{x}_{\theta^*(i)} - \M{x}_j)}_2}
    + \frac{2\nnorm{(\wh{B} - B^*)^{\T}(\M{x}_j - \M{x}_{\theta^*(i)})}_2 \nnorm{\sigma  \eps_{i}}_2}{\nnorm{B^{*\T}  (\M{x}_{\theta^*(i)} - \M{x}_j)}_2^2} \leq 1 \Bigg \} \notag \\
  \supseteq& \Bigg \{
    \left( \frac{\nnorm{B^{*}  - \wh{B}}_2 \max\limits_{1 \leq i \leq n} \nnorm{\M{x}_i}_2}{\min_{i < j} \nnorm{B^{*\T} (\M{x}_i - \M{x}_j)}_2} \right)^2 + \frac{2\sigma \max\limits_{1 \leq i \leq n}  \nnorm{\eps_i}_2}{\min_{i < j} \nnorm{B^{*\T} (\M{x}_i - \M{x}_j)}_2} + \frac{2 \nnorm{B^{*}  - \wh{B}}_2 \max\limits_{1 \leq i \leq n} \nnorm{\M{x}_i}_2}{\min_{i < j} \nnorm{B^{*\T} (\M{x}_i - \M{x}_j)}_2} \notag \\
  &\qquad 
    \quad + \frac{2\sigma \max_{1 \leq i \leq n}  \nnorm{\eps_i}_2 }{\min_{i < j} \nnorm{B^{*\T} (\M{x}_i - \M{x}_j)}_2}
    \cdot  \frac{2 \nnorm{B^{*}  - \wh{B}}_2 \max_{1 \leq i \leq n} \nnorm{\M{x}_i}_2}{\min_{i < j} \nnorm{B^{*\T} (\M{x}_i - \M{x}_j)}_2} \leq 1 \Bigg \} \label{eq:suff_nonnull}
\end{align}
Given the event $\mc{B}$, we have that 
\begin{equation}\label{eq:suff_nonnull_imp}
  \min_{i < j} \nnorm{B^{*\T} (\M{x}_j - \M{x}_i)}_2 = \gamma \nnorm{B^*}_F = \gamma \sigma \sqrt{m} \textsf{SNR}^{1/2}. 
\end{equation}
Plugging~\eqref{eq:suff_nonnull_imp} into~\eqref{eq:suff_nonnull} and~\eqref{eq:corr_noisebound}, it is easy to verify that under the conditions
of the theorem the left hand side of the event in~\eqref{eq:suff_nonnull_imp} is upper bounded
by $1/36 + 1/3 + 1/3 + 1/9 < 1$ with the stated probability.

We now turn to the converse statement in the regime $m = O(1)$ (second bullet); the converse
statement without restriction on $m$ is given subsequently. 
Let $(i_0, j_0)$ denote the pair of indices such that
\begin{equation*}
\nnorm{B^{*\T} (\M{x}_{i_0} - \M{x}_{j_0})}_2^2 = \min_{i < j} \nnorm{B^{*\T} (\M{x}_i - \M{x}_j)}_2^2
= \gamma^2 \nnorm{B^*}_F^2,
\end{equation*}
and suppose that $i_0' = \theta^{*-1}(i_0) \neq \emptyset$. For the event $\{  \wh{\Theta}(B^*) = \Theta^* \}$ to hold it is required that
\begin{align*}
  &\nnorm{\M{y}_{i_0'} - B^{*\T} \M{x}_{i_0}}_2^2 \leq \nnorm{\M{y}_{i_0'} - B^{*\T} \M{x}_{j_0}}_2^2 \notag \\
  \; \Leftrightarrow& 2 \nscp{\sigma \eps_{i_0'}}{B^{*\T} (\M{x}_{j_0} - \M{x}_{i_0})} \leq
       \nnorm{B^{*\T} (\M{x}_{i_0} - \M{x}_{j_0})}_2^2 \\
  \Leftrightarrow&  2 \scp{\sigma \eps_{i_0'}}{\frac{B^{*\T} (\M{x}_{j_0} - \M{x}_{i_0})}{\nnorm{B^{*\T} (\M{x}_{i_0} - \M{x}_{j_0})}_2}} \leq
    \nnorm{B^{*\T} (\M{x}_{i_0} - \M{x}_{j_0})}_2 \notag \\
  \Leftrightarrow&  2 \scp{\sigma \eps_{i_0'}}{\frac{B^{*\T} (\M{x}_{j_0} - \M{x}_{i_0})}{\nnorm{B^{*\T} (\M{x}_{i_0} - \M{x}_{j_0})}_2}} \leq \gamma \nnorm{B^*}_F \notag \\
  \Leftrightarrow&  2 \scp{\sigma \eps_{i_0'}}{\frac{B^{*\T} (\M{x}_{j_0} - \M{x}_{i_0})}{\nnorm{B^{*\T} (\M{x}_{i_0} - \M{x}_{j_0})}_2}} \leq \gamma \sigma \sqrt{m} \textsf{SNR}^{1/2} \notag
\end{align*}
Note that conditional on $\M{x}_{i_0}, \M{x}_{j_0}$ the left hand side follows a $N(0, 4 \sigma^2)$-distribution. It is easy to show that if $g \sim N(0,1)$, $\p(|g| \leq \delta) \leq \delta$ and thus
$\p(g > \delta) \geq \frac{1}{2} (1 - \delta)$for
all $\delta > 0$. Hence if
\begin{equation}\label{eq:restore_lowerbound2}
\gamma \textsf{SNR}^{1/2} < \frac{2}{3} \frac{1}{\sqrt{m}} \; \Leftrightarrow \gamma^2 \textsf{SNR} < \frac{4}{9m}  \invcoloneq c,   
\end{equation}
$\wh{\Theta}(B^*) \neq \Theta^*$ with probability at least $1/3$.

We now turn to the converse statement without restriction on $m$ (first bullet). Note that the event $\{  \wh{\Theta}(B^*) = \Theta^* \}$ implies the event
\begin{align}
&\bigcap_{i = 1}^n \left\{ \nnorm{\M{y}_{i} - B^{*\T} \M{x}_{\theta^*(i)}}_2^2 \leq \min_{j \neq \theta^*(i)} \nnorm{\M{y}_{i} - B^{*\T} \M{x}_{j}}_2^2 \right\} \notag \\
  =&\bigcap_{i = 1}^n \bigcap_{j \neq \theta^*(i)} \left\{ 2\sigma \scp{\eps_i}{B^{*\T} (\M{x}_j - \M{x}_{\theta^*(i)}) / \nnorm{B^{*\T} (\M{x}_j - \M{x}_{\theta^*(i)})}_2 }\leq \nnorm{B^{*\T} \M{x}_{\theta^*(i)} - B^{*\T} \M{x}_{j}}_2 \right\} \notag \\
  \subseteq& \bigcap_{i = 1}^n \left\{ 2\sigma \scp{\eps_i}{B^{*\T} (\M{x}_{\eta(i)} - \M{x}_{\theta^*(i)}) / \nnorm{B^{*\T} (\M{x}_{\eta(i)} - \M{x}_{\theta^*(i)})}_2 }\leq \nnorm{B^{*\T} \M{x}_{\theta^*(i)} - B^{*\T} \M{x}_{\eta(i)}}_2 \right\}, \label{eq:restore_lowerbound3}
\end{align}
where $\eta(i) = \theta^*(i) - 1$ if $\theta^*(i) \geq 2$ and $\eta(i) = \theta^*(i) + 1$ otherwise. Now note that conditional on the $\{ \M{x}_i \}_{i = 1}^n$, the collection
\begin{equation*}
  \{ \nscp{\eps_i}{B^{*\T} (\M{x}_{\eta(i)} - \M{x}_{\theta^*(i)}) / \nnorm{B^{*\T} (\M{x}_{\eta(i)} - \M{x}_{\theta^*(i)})}_2}, \; 1 \leq i \leq n \}
\end{equation*}
are i.i.d.~$N(0,1)$ random variables. By standard concentration arguments for the maximum
of a collection of Gaussian random variables (cf.~\cite{Ledoux1991}, p.~79), we thus have 
\begin{equation}\label{eq:maximagaussian_lower}
\p\left(\max_{1 \leq i \leq n} 2\sigma\scp{\eps_{i}}{\frac{B^{*\T} (\M{x}_{\eta(i)} - \M{x}_{\theta^*(i)})}{\nnorm{B^{*\T} (\M{x}_{\eta(i)} - \M{x}_{\theta^*(i)})}_2}} < 2\sigma c_0 \sqrt{\log n} \;\;\,\Big| \{ \M{x}_i \}_{i = 1}^n \right) \leq 2/5. 
\end{equation}
for a constant $c_0 > 0$. 
At the same time, concentration of Lipschitz functions of Gaussian random variables yields
\begin{align}\label{eq:hansonwright}
  \p(\nnorm{B^{*\T} (\M{x}_{\eta(i)} - \M{x}_{\theta^*(i)})}_2^2 \geq (1 + t)^2 2 \nnorm{B^*}_F^2) &\leq \exp \left(-\frac{t^2 \nnorm{B^*}_F^2}{2 \nnorm{B^*}_2^2} \right) \\
  &\leq \exp\left(-\frac{t^2}{2} \right), \; t \geq 0, \;  1 \leq i \leq n. \notag
\end{align}
Let $i_{\max}$ be the index such that
\begin{equation*}
\scp{\eps_{i_{\max}}}{\frac{B^{*\T} (\M{x}_{\eta(i_{\max})} - \M{x}_{\theta^*(i_{\max})})} {\nnorm{B^{*\T} (\M{x}_{\eta(i_{\max})} - \M{x}_{\theta^*(i_{\max})})}_2}}  = \max_{1 \leq i \leq n} \scp{\eps_i}{\frac{B^{*\T} (\M{x}_{\eta(i)} - \M{x}_{\theta^*(i)})}{\nnorm{B^{*\T} (\M{x}_{\eta(i)} - \M{x}_{\theta^*(i)})}_2}} 
\end{equation*}
Since $\{ (B^{*\T} (\M{x}_{\eta(i)} - \M{x}_{\theta^*(i)}) / \nnorm{B^{*\T} (\M{x}_{\eta(i)} - \M{x}_{\theta^*(i)})}_2,  \nnorm{B^{*\T} (\M{x}_{\eta(i)} - \M{x}_{\theta^*(i)})}_2 \}_{i = 1}^n$ are pairs of independent random variables, we combine~\eqref{eq:maximagaussian_lower} and~\eqref{eq:hansonwright} to conclude that the event $\mc{A}_1 \cap \mc{A}_2$ occurs with probability at least $1/3$, where
\begin{align*}
  \mc{A}_1 &= \left\{ 2 \sigma \scp{\eps_{i_{\max}}}{\frac{B^{*\T} (\M{x}_{\eta(i_{\max})} - \M{x}_{\theta^*(i_{\max})})} {\nnorm{B^{*\T} (\M{x}_{\eta(i_{\max})} - \M{x}_{\theta^*(i_{\max})})}_2}}
     > 2 c_0 \sigma \sqrt{\log n}
  \right\} \\
  \mc{A}_2 &= \left\{ \nnorm{B^{*\T} (\M{x}_{\eta(i_{\max})} - \M{x}_{\theta^*(i_{\max})})}_2 \leq \sqrt{18}  \nnorm{B^*}_F
                   = \sigma \sqrt{18 m} \textsf{SNR}^{1/2} \right\}
  \end{align*}
Combining~\eqref{eq:restore_lowerbound3} and the
  previous display then yields that $\wh{\Theta}(B^*) \neq \Theta^*$ with the stated probability if
  \begin{equation*}
  \textsf{SNR} < \frac{4}{18} c_0^2 \frac{\log n}{m} \invcoloneq c' \frac{\log n}{m}. 
  \end{equation*}

\section{Proof of Proposition~\ref{prop:discovery_mismatches}}
By the triangle inequality and the fact that $\Xi^*_{i,:} = 0$ for all $i \in S_*^c$, we have
\begin{align}
  \min_{i \in S_*} \nnorm{\wh{\Xi}_{i,:}}_2 - \max_{i \in S_*^c} \nnorm{\wh{\Xi}_{i,:}}_2 &\geq \min_{i \in S_*} \nnorm{\Xi_{i,:}^*}_2  - 2\max_{1 \leq i \leq n} \nnorm{\wh{\Xi}_{i,:} - \Xi_{i,:}^*}_2 \notag \\
  &\geq \min_{i \in S_*} \nnorm{\Xi_{i,:}^*}_2 - 2 \nnorm{\wh{\Xi} - \Xi^*}_F.  \label{eq:discovery_mismatches_proof_1}
\end{align} 
In the sequel, we derive a lower bound on $\min_{i \in S_*} \nnorm{\Xi_{i,:}^*}_2$ in a fashion similar to the previous proof. For any $i$ with $\theta^*(i) = 0$, we have
\begin{align}\label{eq:discovery_mismatches_proof_2}
  \nnorm{\sqrt{n} \Xi^*_{i,:}}_2 = \nnorm{\M{y}_i - B^{*\T} \M{x}_i}_2 \geq \gamma_0 \nnorm{B^*}_F
  = \gamma_0 \cdot \sigma \sqrt{\textsf{SNR}} \sqrt{m}.
\end{align}
On the other hand, for any $i$ with $\theta^*(i) \notin \{0,i \}$, we have
\begin{equation}\label{eq:discovery_mismatches_proof_3}
  \nnorm{\sqrt{n} \Xi^*_{i,:}}_2 = \nnorm{B^{*\T} \M{x}_{\theta^*(i)} - B^{*\T} \M{x}_i}_2 \geq \gamma \nnorm{B^*}_F =
  \gamma \cdot \sigma \sqrt{\textsf{SNR}} \sqrt{m}
\end{equation}
Combining~\eqref{eq:discovery_mismatches_proof_1},~\eqref{eq:discovery_mismatches_proof_2} and~\eqref{eq:discovery_mismatches_proof_3} yields the assertion. 
\section{Auxiliary Results}

\begin{lemmaApp}\label{lem:sparsity_conv} For any $r \geq 1$, we have the inclusion
\begin{equation}\label{eq:sparsity_conv}
\{v \in \R^{n \cdot m}:\; \nnorm{v}_2 \leq 1, \; \nnorm{v}_{2,1} \leq \sqrt{r} \} \subset 2 \, \text{\emph{conv}}\,B_0(r), 
\end{equation}
with $\nnorm{\cdot}_{2,1}$ and $B_0(r)$ are defined in~\eqref{eq:mixednorms} and Lemma~\ref{lem:conecondition}, respectively. 
\end{lemmaApp}
\begin{proof}
The proof is an adaptation of a standard argument in the sparsity literature, cf.~Lemma 3.1 in~\cite{PlanVershynin2013b}. Pick an arbitrary element $v$ contained in the left hand side in~\eqref{eq:sparsity_conv}, and consider subsets $T_{\ell} \subset \{ 1, \ldots, n \}$, 
$|T_{\ell}| \leq r$, and corresponding vectors $v(T_{\ell}) \in B_0(r)$ such that 
\begin{equation*}
(v(T_{\ell}))_j  \coloneq \begin{cases}
v_j  &\text{if} \; j \in \bigcup_{i \in T_{\ell}} G_i, \\
0   & \text{else}. 
\end{cases} 
\end{equation*}
and such that $T_1$ contains the $r$ indices of $\{1,\ldots,n\}$ corresponding to 
the $r$ largest norms among $\{ \nnorm{v^{[i]}}_2 \}_{i = 1}^n$, $T_2$ contains the
$r$ indices corresponding to the next $r$ largest norms among $\{ \nnorm{v^{[i]}}_2 \}_{i = 1}^n$,
and so forth. Observe that $v = \sum_{\ell} v(T_{\ell})$ and that for any $\ell$
\begin{equation*}
\nnorm{v(T_{\ell+1})}_{2,\infty}  = \max_{i \in T_{\ell + 1}} \nnorm{v^{[i]}}_2 \leq \frac{1}{r} \sum_{i \in T_{\ell}} \nnorm{v^{[i]}}_2 = \frac{1}{r} \nnorm{v(T_{\ell})}_{2,1} 
\end{equation*} 
As a result,
\begin{equation*}
\nnorm{v(T_{\ell+1})}_{2} \leq \sqrt{r} \nnorm{v(T_{\ell+1})}_{2,\infty}   = \frac{1}{\sqrt{r}} \nnorm{v(T_{\ell})}_{2,1}.  
\end{equation*} 
Consequently, 
\begin{align*}
\sum_{\ell} \nnorm{v(T_{\ell})}_2 &= \nnorm{v(T_{1})}_2 + \sum_{\ell \geq 2} \nnorm{v(T_{\ell})}_2 \\[-.5ex]
                               &\leq 1 + \frac{1}{\sqrt{r}}  \sum_{\ell \geq 1} \nnorm{v(T_{\ell})}_{2,1} \\[-.5ex]
                               &\leq 1 + \frac{1}{\sqrt{r}}  \sum_{\ell \geq 1}  \sum_{i \in T_{\ell}} \nnorm{v_{G_i}}_2 \\[-.5ex]
                               &\leq 1 + \frac{1}{\sqrt{r}} \nnorm{v}_{2,1} \leq 2.   
\end{align*}
In conclusion, we have demonstrated that 
\begin{equation*}
  v = \sum_{\ell} \underbrace{\frac{v({T_{\ell})}}{\nnorm{v(T_{\ell})}_2}}_{\in B_0(r)}  \underbrace{\nnorm{v(T_{\ell})}_2}_{\lambda_{\ell}},
  \qquad  \sum_{\ell} \lambda_{\ell} \leq 2,
\end{equation*}
and thus $v \in 2 \, \text{conv} \, B_0(r)$. Since $v$ was an arbitrary element of the left hand side in
\eqref{eq:sparsity_conv}, the proof is complete. 
\end{proof}

\begin{lemmaApp}\label{lem:gaussiannorm}
Let $g_{\ell} \sim N(0, \sigma_{\ell}^2 I_r)$, $1 \leq \ell \leq L$, be isotropic Gaussian random vectors. Then:   
\begin{equation*}
\p \left(\max_{1 \leq \ell \leq L} \nnorm{g_{\ell}}_2 > \max_{1 \leq \ell \leq L} \sigma_{\ell} \{ \sqrt{r} + 2 \sqrt{\log L} \} \right) \leq 1/L. 
\end{equation*}
\end{lemmaApp}
\begin{proof} We note that $\E[\nnorm{g}_{\ell}] \leq \sigma_{\ell} \sqrt{r}$, $\ell = 1,\ldots,L$, and that the map
$x \mapsto \nnorm{x}_2$ is 1-Lipschitz. By concentration of measure of Lipschitz functions of Gaussian
random vectors, we hence have 
\begin{equation*}
\p(\nnorm{g_{\ell}}_2 \geq \sigma_{\ell} (\sqrt{r} + 2 \sqrt{ \log L})) \leq \exp(-2 \log L), \; \ell=1,\ldots,L. 
\end{equation*} 
The result then follows from a union bound over $\{1,\ldots,L\}$.
\end{proof}

\begin{lemmaApp}\label{lem:escape}\emph{(\textbf{Gordon's Escape theorem~\citep{Gordon1988}})}
 Let $K$  be a closed subset of the unit sphere in $\R^p$, let $\nu_r = \E_{g \sim N(0, I_r)}[\nnorm{g}_2]$, and let 
  $\epss \in (0,1)$. If the Gaussian width (cf.~$\S$7.5 in~\citep{Vershynin2018}) of $K$ obeys $w(K) < (1 - \epss) \nu_q - \epss \nu_p$, then a $(p - q)$-dimensional subspace
  $V$ drawn uniformly from the Grassmannian $\emph{\textsf{G}}(p, p-q)$ satisfies
  \begin{equation*}
\p(\text{\emph{dist}}(K, V) > \epss) \geq 1 - \frac{7}{2} \exp \left(-\frac{1}{2} \left( \frac{(1 - \epss)\nu_q - \epss \nu_p  - w(K)}{3 + \epss + \epss \nu_p / \nu_q} \right)^2 \right).
\end{equation*}
\end{lemmaApp}
\vspace*{-1.5ex}
\tcb{\section{From Gaussian to sub-Gaussian}\label{app:tosubgaussian}
In this section, we state and prove a result analogous to Lemma~\ref{lem:escape} above for random subspaces $V$ generated
by a $p$-by-$(p - q)$ matrix $A$ with i.i.d.~isotropic \emph{sub}-Gaussian rows, i.e., $\E[\scp{A_{i,:}}{v}^2] = 1$ and $\nnorm{\scp{A_{i,:}}{v}}_{
\psi_2} \leq L < \infty$ for all $v \in \R^{p-q}$, $1 \leq i \leq n$, where $\nnorm{\cdot}_{\psi_2}$ denotes the sub-Gaussian    
norm of a random variable (see, e.g., $\S$2.5 in~\cite{Vershynin2018}).
\begin{lemmaApp} Let $V = \text{\emph{range}}(A)$ with $A$ as above, and let $K$ be a closed subset of the unit sphere in $\R^p$. For any $\epss, \alpha \in (0,1)$, if 
\begin{equation}\label{eq:escape_samplecomplexity} 
p > \frac{1}{1 - \epss^2} \frac{2 (p-q) + C_1 L^4 \cdot w^2(K)}{(1-\alpha)^2}  \vee \frac{C}{\alpha^2}\{ (p-q) \vee \log p \}
\end{equation} 
then $\p(\text{\emph{dist}}(K, V) > \epss) \geq 1 - 2 \big( \exp(-w^2(K)) + \exp(-c \{ (p-q) \vee \log p \}) \big)$,  
where $C_1, C_2, c > 0$ are universal constants depending only on $L$. 
\end{lemmaApp}
\noindent It is worth noting that the condition~\eqref{eq:escape_samplecomplexity} is comparable to the condition 
in Lemma~\ref{lem:escape} which after term simplifications becomes $p \gtrsim \frac{1}{1-\epss^2} ((p-q) + w^2(K))$, which 
corresponds to the first (and leading) term on the right hand side of~\eqref{eq:escape_samplecomplexity}. 
\begin{proof}
Let $V^{\perp}$ denote the orthogonal complement of $V$ in $\R^p$, respectively. Accordingly, denote by $\textsf{P}_{V}$ and $\textsf{P}_{V^{\perp}}$ the orthoprojectors on $V$ and $V^{\perp}$, respectively. Note that  
\begin{equation}
    \text{dist}^2(K, V) = \inf_{\xi \in K} \nnorm{\textsf{P}_{V^{\perp}} \xi}_2^2 = 1 - \sup_{\xi \in K} 
    \nnorm{\textsf{P}_{V} \xi}_2^2. \label{eq:infPA} 
\end{equation}
Hence in order to lower bound $\text{dist}^2(K, V)$, it suffices to upper bound $\sup_{\xi \in K} \nnorm{\textsf{P}_{V} \xi}_2^2$. Assuming for now that $A$ is non-singular, we have 
\begin{align}
\sup_{\xi \in K} \nnorm{\textsf{P}_{V} \xi}_2^2 &= \sup_{\xi \in K} \xi^{\T} A (A^{\T} A)^{-1} A^{\T} \xi \notag \\
                                                       &\leq \sup_{\xi \in K} \nnorm{(A^{\T} A)^{-1/2} A^{\T} \xi}_2^2 \notag \\
                                                      &\leq \nnorm{(A^{\T} A)^{-1/2}}_2^2 \; \sup_{\xi \in K} \nnorm{A^{\T} \xi}_2^2 
          \leq \frac{1}{\sigma_{\min}(A)^2} \sup_{\xi \in K} \; \nnorm{A^{\T} \xi}_2^2 \label{eq:supPA_final}.
      \end{align}
In order to bound the second factor on the right hand side, we invoke the following result: 
\begin{lemmaApp}(cf.~Exercise 9.1.8 in~\cite{Vershynin2018}). Let $A$, $L$, and $K$ be as above. Then for any $u \geq 0$, the following event occurs
      with probability at least $1 - 2 \exp(-u^2)$:
      \begin{equation*}
      \sup_{\xi \in K} \left|\nnorm{A^{\T} \xi}_2 - \sqrt{p-q} \right| \leq C L^2 (w(K) + u).
    \end{equation*}
\end{lemmaApp}    
\noindent Invoking the above lemma with the choice $u = w(K)$, we obtain that 
\begin{equation}\label{eq:app_matrixdeviation}
\p \left(\sup_{\xi \in K} \nnorm{A^{\T} \xi}_2 \leq \sqrt{p - q} + C' L^2 w(K) \right) \geq 1 - 2 \exp(-w^2(K)). 
\end{equation}
At the same time, concentration results~\citep[Theorem 5.35]{Vershynin2010} on the minimum singular value of random matrices with sub-Gaussian rows yield that for any $\alpha \in (0,1)$
\begin{equation}\label{eq:app_extremesingular}
\p(\sigma_{\min}(A)^2 \geq (1 - \alpha)^2 p) \geq 1 - 2\exp(-c \{ (p-q) \vee \log p \}).    
\end{equation}
provided that $p \geq \frac{C}{\alpha^2}\{ (p-q) \vee \log p \}$ for positive constants $c = c_L$ and $C = C_L$ depending
only on the sub-Gaussian norm $L$ of the rows of $A$. Combining~\eqref{eq:infPA},~\eqref{eq:supPA_final}, ~\eqref{eq:app_matrixdeviation} and~\eqref{eq:app_extremesingular}, we obtain that with the probability stated in the theorem, it holds that
\begin{equation*}
 \inf_{\xi \in K}  \nnorm{\textsf{P}_{V^{\perp}}\xi}_2^2 \geq 1 - \frac{2 (p-q) + C'' L^4 w^2(K)}{(1-\alpha)^2 p} \geq \varepsilon^2 
\end{equation*}
as long as $p > \frac{1}{1 - \varepsilon^2} \frac{2 (p-q) + C'' L^4 \cdot w^2(K)}{(1-\alpha)^2}$ for any $\varepsilon \in (0,1)$, which concludes the proof. 
\end{proof}}
\vspace*{-1.5ex}
\tcb{\section{Conditional gradient method for optimization of~\eqref{eq:grouplasso_refined} \& ~\eqref{eq:grouplasso_refined_cons}\label{app:congradient}}
We start with optimization problem~\eqref{eq:grouplasso_refined_cons}. Let 
\begin{equation*}
f(\Theta) \coloneq \frac{1}{2n \cdot m}\nnorm{\texttt{P}_{X}^{\perp} \Theta Y}_F^2, \qquad \nabla f(\Theta) =\frac{1}{n \cdot m} \texttt{P}_{X}^{\perp} \Theta Y Y^{\T}. 
\end{equation*}
be the objective and gradient, respectively, of~\eqref{eq:grouplasso_refined_cons}. Following Algorithm 1 in~\cite{Jaggi2013}, the conditional gradient (Frank-Wolfe) updates for minimizing $f$ over $\mc{C}_k \coloneq \{\Theta \in \mc{C}:\; \su \Theta_{ii} \geq n -k \}$ with $\mc{C}$ defined in~\eqref{eq:matchings_relaxed} are given as follows.
\begin{algorithm}[!h]
  \caption{Frank-Wolfe method for minimizing~\eqref{eq:grouplasso_refined_cons}}\label{alg:fw}
Initialize $\Theta^{(0)} = I_n$.\\
\textbf{Repeat for $t = 0,1,\ldots$}
\begin{align*}
  D^{(t)}        &\leftarrow \argmin_{\Theta \in \mc{C}_k} \tr(\Theta^{\T} \nabla f(\Theta^{(t)})), \qquad
  \Theta^{(t+1)} \leftarrow (1 - \alpha^{(t)}) \Theta^{(t)} + \alpha^{(t)} D^{(t)}, 
\end{align*}
where $\alpha^{(t)} = \argmin_{\alpha > 0} f((1 - \alpha^{(t)}) \Theta^{(t)} + \alpha D^{(t)}) = -\frac{\tr(\texttt{P}_{X}^{\perp} D^{(t)} YY^{\T} \Theta^{(t){\T}})}{\tr(\texttt{P}_{X}^{\perp} D^{(t)} YY^{\T} D^{(t){\T}})}$. 
\end{algorithm}
}\vskip.01ex
\noindent \tcb{The dominant computational cost in the above algorithm is incurred for the $\argmin$ over $\mc{C}_k$, which requires the solution of a linear program with $n^2$ variables and $O(n)$ linear constraints.}

A similar algorithm can be applied for optimization problem~\eqref{eq:grouplasso_refined}. An additional complication 
arises from the penalty in~\eqref{eq:grouplasso_refined} which renders the objective non-smooth. As a workaround, we apply the above Frank-Wolfe scheme to a successively smoothed objective~\citep{Nesterov2005}.



\end{document}

%% file: temp.tex
\usepackage{amsfonts}
\usepackage{amssymb}
\usepackage{amsmath}
\usepackage{amsthm}
\usepackage{latexsym}
\usepackage{verbatim}
\usepackage{epsfig}
\usepackage{amsbsy}
\usepackage{amscd}
\usepackage{bm}
\usepackage{graphicx}

\usepackage{ifthen} 
\usepackage{xcolor}

\providecommand{\M}[1]{\mathbf#1}
\providecommand{\mc}[1]{\mathcal#1}
\providecommand{\mc}[1]{\mathcal#1}

\newcommand{\R}{{\mathbb R}}

\DeclareMathOperator{\E}{\mathbf{E}}
\DeclareMathOperator{\p}{\mathbf{P}}

\DeclareMathOperator{\tr}{tr}

\providecommand{\T}{\top} 

\DeclareMathOperator*{\argmin}{argmin}
\DeclareMathOperator*{\argmax}{argmax}

\providecommand{\wt}[1]{\widetilde{#1}}
\providecommand{\wh}[1]{\widehat{#1}}

\providecommand{\norm}[1]{\left \lVert#1 \right \rVert}
\providecommand{\nnorm}[1]{ \lVert#1 \rVert}
\newcommand{\scp}[2]{\left\langle#1, #2\right\rangle}
\newcommand{\nscp}[2]{\langle#1, #2\rangle}

\newcommand{\blanco}[1]{  }

\newcommand{\deriv}[3]{%
\ifthenelse{#1 = 1}{\frac{d\,#2}{d\,#3}}{\frac{d^{{#1}} #2}{d{#3}^{{#1}}}}
}

\newcommand{\partials}[3]{%
\ifthenelse{#1 = 1}{\frac{\partial\,#2}{\partial\,#3}}{\frac{\partial^{#1}
    #2}{\partial#3^{#1}}}
} 

\def\su{\sum_{i=1}^n}

\def \coloneq{\mathrel{\mathop:}=}
\def \invcoloneq{=\mathrel{\mathop:}}
\def \eps{\varepsilon}

\newcommand{\tcb}[1]{\textcolor{black}{#1}}

\newtheorem{theo}{Theorem}
\newtheorem{propo}{Theorem}

\newtheorem{lemmaA}{Theorem}[section]

\newtheorem{lemmachen}{Theorem}
\newtheorem{theorm}[theo]{Theorem}

\newtheorem{prop}[propo]{Proposition}
\newtheorem{lemma}[lemmachen]{Lemma}

\newtheorem{lemmaApp}[lemmaA]{Lemma}

